\documentclass[10pt,journal,compsoc]{IEEEtran}

\ifCLASSOPTIONcompsoc
  \usepackage[nocompress]{cite}
\else
  \usepackage{cite}
\fi

\ifCLASSINFOpdf
\else
\fi

\usepackage{hyperref}       % hyperlinks
\usepackage{url}            % simple URL typesetting
\usepackage{booktabs}       % professional-quality tables
\usepackage{amsfonts}       % blackboard math symbols
\usepackage{nicefrac}       % compact symbols for 1/2, etc.
\usepackage{microtype}      % microtypography
\usepackage{xcolor}         % colors
\usepackage{graphicx}
\usepackage{amsmath,amssymb}
\usepackage{verbatim}
\usepackage{multirow}
\usepackage{marvosym}
\usepackage{makecell}
\usepackage{subcaption}
\usepackage{bm}
\usepackage{wrapfig}
\usepackage{caption}
\usepackage[marginal]{footmisc}
\usepackage{algorithm}
\usepackage{algorithmic}
\usepackage{utfsym}
\usepackage{makecell}
\usepackage{makebox}
 
\newtheorem{theorem}{Theorem}
\newtheorem{proof}{Proof}

\newtheorem{mylem}{Lemma}

\begin{document}
\title{Decentralized and Lifelong-Adaptive \\ Multi-Agent Collaborative Learning}

\author{Shuo Tang, Rui Ye, Chenxin Xu, Xiaowen Dong, Siheng Chen, Yanfeng Wang}

% abstract 
\IEEEtitleabstractindextext{
\begin{abstract}

Decentralized and lifelong-adaptive multi-agent collaborative learning aims to enhance collaboration among multiple agents without a central server, with each agent solving varied tasks over time. To achieve efficient collaboration, agents should: i) autonomously identify beneficial collaborative relationships in a decentralized manner; and ii) adapt to dynamically changing task observations. In this paper, we propose \texttt{DeLAMA}, a decentralized multi-agent lifelong collaborative learning algorithm with dynamic collaboration graphs. To promote autonomous collaboration relationship learning, we propose a decentralized graph structure learning algorithm, eliminating the need for external priors. To facilitate adaptation to dynamic tasks, we design a memory unit to capture the agents' accumulated learning history and knowledge, while preserving finite storage consumption. To further augment the system's expressive capabilities and computational efficiency, we apply algorithm unrolling, leveraging the advantages of both mathematical optimization and neural networks. This allows the agents to `learn to collaborate' through the supervision of training tasks. Our theoretical analysis verifies that inter-agent collaboration is communication efficient under a small number of communication rounds. The experimental results verify its ability to facilitate the discovery of collaboration strategies and adaptation to dynamic learning scenarios, achieving a 98.80\% reduction in MSE and a 188.87\% improvement in classification accuracy. We expect our work can serve as a foundational technique to facilitate future works towards an intelligent, decentralized, and dynamic multi-agent system. Code is available at \href{https://github.com/ShuoTang123/DeLAMA}{https://github.com/ShuoTang123/DeLAMA}.
\end{abstract}

}

\maketitle

\IEEEdisplaynontitleabstractindextext

\IEEEpeerreviewmaketitle

% introduction
\IEEEraisesectionheading{\section{Introduction}\label{sec:introduction}}
\IEEEPARstart{C}ollaboration is a stealthy, yet ubiquitous phenomenon in nature, evident in the cooperative behaviors of animals and humans, from pack-hunting to constructing complex social relationships. Such cooperative behaviors enable individuals to accomplish complex tasks~\cite{woolley2010evidence} and form social relationships~\cite{mennis2006wisdom}. Inspired by this natural tendency for collaboration, the field of multi-agent collaborative learning has been extensively explored~\cite{mennis2006wisdom}. It aims to enable multiple agents to collaboratively strategize and achieve shared objectives, exhibiting capabilities surpassing that of any single agent. Recently, there has been an urgent demand for relevant methods and systems, such as vehicle-to-everything(V2X) communication-aided autonomous driving~\cite{v2x-communication, v2x-sim}, multi-robot environment exploration~\cite{multi-robot, col-multi-robot} and collaborative training of machine learning models across multiple clients~\cite{fedavg}.

Depending on the type of content exchanged among agents, multi-agent collaboration can be categorized into four distinct levels, namely perception, cognition, decision-making, and control. (1) At the perception level, the primary focus is to compensate for the sensory information gaps of individual agents by transmitting perceptual data among multiple agents. The shared information often encompasses perceptual signals, such as raw observations and perceptual features. Such mechanisms are especially valuable in research fields like collaborative perception~\cite{where2com}, and multi-view learning~\cite{multi-view}. (2) At the cognitive level, the aim is to improve multiple agents' learning task performance by sharing and processing cognitive information. Specifically, assuming local decision-making is done via training an individual model, the shared cognitive information often relates to individuals' model parameters~\cite{fedavg}, knowledge graph~\cite{knowledge-graph}, or dictionary~\cite{COLLA}. Such methodologies are prominently employed in research fields such as collaborative machine learning~\cite{fedavg, fedprox, fedrep} and multi-task learning~\cite{liu2017distributed, wang2016distributed}. (3) At the decision-making level, agents cooperate and communicate to make joint decisions, aiming to maximize rewards from tasks performed in their environments. Agents share high-level information about decision-making~\cite{tarmac, marl-communication}. The related approaches are widely used in research fields like multi-agent reinforcement learning (MARL)~\cite{marl_survey_1}. (4) Moving to the control level, the focus shifts towards formulating efficient control strategies and coordination mechanisms. These strategies mainly share individual state information such as position or velocity. Common application scenarios for these control strategies encompass areas like traffic control~\cite{multi-vehicle-control, multi-agent-system}, or unmanned aerial vehicle (UAV) system control~\cite{multi-agent-UAV}. In this paper, we focus on cognitive-level collaboration, where agents collectively train models by sharing model parameters beyond raw observations, facilitating more streamlined knowledge sharing among agents to equip the system with a more powerful cognitive capability.  

\begin{figure*}
    \centering
    \includegraphics[width=0.97\textwidth]{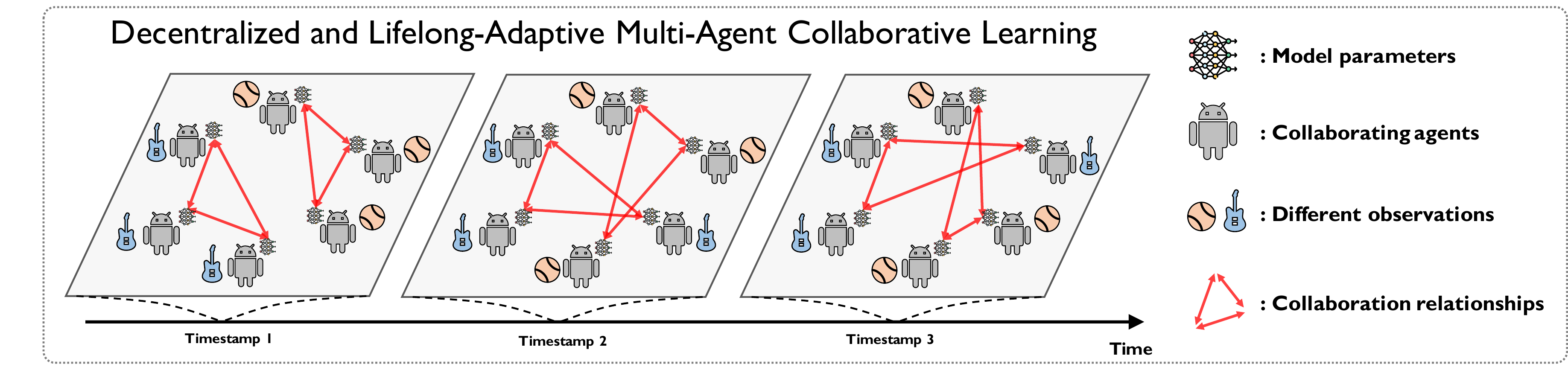}
        \vspace{-3mm}
    \caption{The illustration of our decentralized and lifelong-adaptive multi-agent collaborative learning system. We demonstrate a collaboration system with six agents, each faced with a learning task. The learning tasks' configurations between the agents are dynamic, making the collaboration relationships dynamically adapt to the time-evolving tasks.}
    \label{fig:basic_setting}
    \vspace{-3mm}
\end{figure*}

To promote efficient collaboration among the agents, it is crucial to identify whom to collaborate with. This could avoid redundant and low-quality information from an arbitrary collaborator, improving collaboration effectiveness. Current collaboration methods often involve using predefined fixed collaboration relationships, such as fully connected graphs (e.g., CoLLA~\cite{COLLA}, DiNNO~\cite{DiNNO}), star graphs with a central server (e.g., federated learning~\cite{fedavg}), or task correlations calculated by a central server in multi-task learning~\cite{liu2017distributed}. However, these predefined structures are limited by two major issues. First, collaboration~\cite{fedavg,liu2017distributed} guided by a central server faces vulnerability issues, where the central server failure would disable all collaborations. Second, the static collaboration~\cite{fedavg,liu2017distributed} relationships limit the flexibility and efficiency of the system performance when faced with dynamic scenarios. In dynamic scenarios~\cite{continuallearningsurvey2}, agents come up with continuous and time-varying observations and update their local decision-making model, leading to the change of agent potential relationships. The static collaboration modeling becomes inadequate and inflexible to capture the relationships in these dynamic scenarios, limiting the collaboration efficiency. In this work, we introduce a novel decentralized and life-long adaptive learning framework for multi-agent cognitive-level collaboration. The core feature of our framework is that each agent can autonomously choose its collaborators and updates its model to address tasks that evolve over time.

To address the limitations of centralized collaboration mechanisms, instead of using a central server, we propose a decentralized collaboration mechanism that decouples the globally oriented collaboration problem into the agent level. Given the model parameters and observations, the key design of our approach is to discard the central server and learn sparse collaboration relationships between agents. Agents are empowered to autonomously select collaborators and communicate independently. Specifically, we model agents' collaboration relationships by a collaboration graph, where nodes denote the agents and edges stand for the collaboration strength. Compared to the centralized collaboration mechanism~\cite{fedavg, wang2016distributed}, our decentralized collaboration mechanism offers two advantages: 1) computations are decentralized at the level of each agent, enhancing system robustness and 2) alleviates the communication overhead typically imposed on a central server, dispersing the communication load more evenly among the agents.

To address the issue of the static collaboration relationship, we design a lifelong-learning-based approach for multi-agent collaboration. The core problem is enabling agents to dynamically adapt to ever-changing observations during individual model training, to provide a basis for obtaining more accurate collaboration relationships. In agent individual model training, we design a memory unit to capture an agent's accumulated learning history and knowledge, facilitating adaptation to evolving observation flows. Compared with the static collaboration mechanism~\cite{fedavg, liu2017distributed}, this lifelong-learning-based approach could lead to evolving collaboration structures and model parameters, making it more suitable for dynamic scenarios. 

Integrating the above designs, we formulate a decentralized optimization for lifelong adaptive multi-agent collaborative learning. The corresponding iterative algorithm consists of two steps: collaborative relational inference and lifelong model update. For the collaborative relational inference step, we propose a convex quadratic program as a graph structure learning problem with the graph smoothness prior. To solve this problem without a central server, we propose a decentralized optimization approach based on Newton iterations by sharing dual variables through the communication links. For the lifelong model update step, we propose a graph message-passing mechanism by sharing model parameters with collaborators inferred in the previous step. To enable agents to retain memories of past tasks, we introduce an online update mechanism by continuously approximating previously encountered loss functions through Taylor expansion. These two iterative steps, while distinct, are closely intertwined: efficient collaborative relational inference simplifies lifelong model update, and vice versa.

To further promote more learning ability and flexibility,  we apply algorithm unrolling techniques~\cite{unrolling-survey, unrolling-plugplay} to upgrade the aforementioned iterative algorithm of the decentralized optimization to a neural network. Through algorithm unrolling, each iteration of our optimization solution is mapped to a customized neural network layer with learnable parameters. The network parameters are learned via supervised training tasks, enabling agents to effectively learn to collaborate.

We term the overall method as~\texttt{DeLAMA}, which achieves Decentralized and Lifelong-Adaptive learning for Multiple Agents. The proposed~\texttt{DeLAMA} combines the advantages of both mathematical optimization and neural networks; see the illustration in \textbf{Figure}~\ref{fig:basic_setting}. Unlike multi-task learning where the agents' relationships are decided by the central server, \texttt{DeLAMA} provides a decentralized mechanism to infer collaboration relationships.  According to our design, our collaboration method~\texttt{DeLAMA} can 1) efficiently operate in dynamic learning scenarios, 2) promote decentralized collaboration without a central server, and 3) autonomously learn and seek collaborative partners. The communication and collaboration between agents of \texttt{DeLAMA} possess both theoretical convergence guarantee and robust learning capabilities.

Our theoretical analysis of \texttt{DeLAMA} reveals that: 1) the overall solution exhibits favorable convex optimization properties and expressive capability; 2) the collaboration relation inference solution possesses a quadratic convergence rate, enabling efficient collaborator identification by agents; and 3) our model-updating solution achieves a linear convergence rate, effectively minimizing the communication demand.

Our experimental evaluation of \texttt{DeLAMA} focuses on two typical multi-agent learning tasks: regression and classification, covering both synthetic and real-world datasets. In the classification scenario, the experiment incorporates both camera image data and lidar scan data, supporting a broad spectrum of downstream applications. To further validate the rationale of our introduced collaboration relationship priors in \texttt{DeLAMA}, we compare the collaboration mechanisms to humans by conducting a human-involved experiment. The key findings of experimental results on \texttt{DeLAMA} lie in two aspects. From the decentralized collaboration viewpoint, \texttt{DeLAMA} outperforms other graph structure learning approaches in inferring collaboration relationships, despite many of them being centralized methods~\cite{L2G, graphlasso}. From the lifelong adaptation perspective, \texttt{DeLAMA} could enable agents to efficiently adapt to dynamic learning environments compared to previous lifelong learning methods~\cite{lwf,gem}.

The outline of this paper is structured as follows.  In \textbf{Section}~\ref{prob-formulate} we introduce the setting of decentralized lifelong-adaptive collaborative learning.  \textbf{Section}~\ref{sec:co_lifelong} formulates the optimization problem and introduced a iterative learning framework. \textbf{Section}~\ref{sec:unroll} further applies the algorithm unrolling method to the iterations. In \textbf{Section}~\ref{sec:theory}, we analyzed the numerical properties of the iterations proposed in \textbf{Section}~\ref{sec:co_lifelong} and \textbf{Section}~\ref{sec:unroll}. In \textbf{Section}~\ref{sec:experiments} we conduct four different experiments on both synthetic and real-world datasets. \textbf{Section}~\ref{sec:related_work} describes related works in collaborative learning and lifelong learning. Finally, we conclude the paper in \textbf{Section}~\ref{sec:conclusion}. 

% collaborative lifelong learning
\vspace{-3mm}
\section{Problem Formulation}
\label{prob-formulate}

\subsection{Task Setting}
\label{task-setting}
Considering a multi-agent system containing $N$ agents, each agent can train a dynamic machine-learning model for evolving tasks. At any given timestamp $t$, the $i$-th agent receives a training dataset $\mathcal{D}_i^{(t)}=\left\{\mathcal{X}_i^{(t)}, \mathcal{Y}_i^{(t)}\right\}$, where  $\mathcal{X}_i^{(t)}$ represents the data inputs and $\mathcal{Y}_i^{(t)}$ denotes the associated ground-truth labels. Based on this, the $i$-th agent updates its model parameters, denoted as $\boldsymbol{\theta}_i^{(t)}$, to effectively handle the test dataset $\widetilde{\mathcal{D}}_i^{(t)}=\left\{\widetilde{\mathcal{X}}_i^{(t)}, \widetilde{\mathcal{Y}}_i^{(t)} \right\}$.  It is worth noting that the test dataset $\widetilde{\mathcal{D}}_i^{(t)}$ shares the same distribution as the training dataset $\mathcal{D}_i^{(t)}$. The overall goal is to enhance inter-agents' collective learning performance. To achieve this, each agent can further 1) collaborate with other agents by sharing valuable information under a collaboration protocol, and 2) use its historical information in the learning process.

For decentralized multi-agent collaboration, information sharing is achieved by direct peer-to-peer communication with a time-varying communication graph $\mathbf{C}^{(t)} \in \{0, 1\}^{N \times N}$. Here, each node represents one agent and each $(i,j)$-th edge with $\mathbf{C}^{(t)}_{ij}=1$ is a pairwise communication link indicating the message transmission between agent $i$, $j$ is enabled.

Given that the current learning task could relate to previous ones, each agent should be able to incorporate past task experiences into its current learning process. Let $\mathcal{M}_i^{(t)}$ be the learning memory of the $i$-th agent up to time $t$. To maintain and utilize the lifelong and adaptive learning ability, the storage space for the memory $\mathcal{M}_i^{(t)}$ of the entire learning task $\mathcal{T}_i$ is required to be finite, hence directly storing past training data is prohibited. Based on the constraint, agents will update the learning memory according to newly observed data $\mathcal{D}_i^{(t)}$ at each time, hence promoting lifelong adaptation.

Under the above conditions on information sharing and memory, decentralized and lifelong-adaptive collaborative learning aims to find an effective decentralized model training strategy. Denote the communication message sent from agent $i$ to agent $j$ at time $t$ as $m_{i \rightarrow j}^{(t)}$, where the information routing process could be multi-hop via the communication graph $\mathbf{C}^{(t)}$. The model parameter and memory updating rule are described as
\begin{equation}
\setlength\abovedisplayskip{1pt}
\setlength\belowdisplayskip{1pt}
\label{setting_eq}
    \begin{aligned}
 \boldsymbol{\theta}_i^{(t)}, \mathcal{M}_i^{(t)} & = \mathbf{\Phi}_i\left(\mathcal{D}_i^{(t)}, \{\boldsymbol{m}_{j\rightarrow i}^{(t)}\}_{j = 1}^N, \mathcal{M}_i^{(t-1)} \big | \mathbf{C}^{(t)}\right),
    \end{aligned}
\end{equation}
where $\mathbf{\Phi}_i(\cdot)$ is the $i$-th agent's approach to update learning memory and model parameters according to previous individual learning memory $\mathcal{M}_i^{(t-1)}$, training data $\mathcal{D}_i^{(t)}$ and the messages $\{\boldsymbol{m}_{j\rightarrow i}^{(t)}\}_{j=1}^N$ sent from other agents.

Note that: i) the collaboration mechanism among the agents is decentralized, which means agents learn their model parameters based on individual observations and mutual collaboration without the management of a central server; ii) the collaboration system, including the agents' models and the collaboration strategies, is lifelong-adaptive to the dynamic scenarios, enabling agents to efficiently learn new knowledge while preserving knowledge from previously encountered ones; and iii) the sparse communication prior among agents based on the constraint $\mathbf{C}^{(t)}$ should be guaranteed, which means the inter-agent messages $\boldsymbol{m}_{j \rightarrow i}^{(t)}$ should adheres strictly to the communication graph constraint $\mathbf{C}^{(t)}$.

To evaluate the learning performance of the $i$-th agent at timestamp $t$, we use the performance accumulation as the evaluation metric:
\begin{equation*}
\setlength\abovedisplayskip{1pt}
\setlength\belowdisplayskip{1pt}
L_i^{(t)} = \frac{1}{t}\sum_{k = 1}^t \mathcal{L}\left(f_{\boldsymbol{\theta}_i^{(k)}}(\widetilde{\mathcal{X}}_i^{(k)}), \widetilde{\mathcal{Y}}_i^{(k)}\right),
\end{equation*}
where $\widetilde{\mathcal{X}}_i^{(t)}$ is the test data and  $\widetilde{\mathcal{Y}}_i^{(t)}$ is the associated test labels. $\mathcal{L}(\cdot)$ is the metric function corresponding to the task. The overall performance of $N$ agents is $L^{(t)} = L_i^{(t)}/N$. The task objective is to maximize $L^{(t)}$ by designing, optimizing the communication protocol of sharing messages $\boldsymbol{m}_{j \rightarrow i}^{(t)}$ and the model updating approach $\mathbf{\Phi}_i(\cdot)$ based on the observed data $\mathcal{D}_{1:N}^{(t)}=\left\{\mathcal{D}_i^{(t)}\right\}_{i=1}^N$ at each timestamp.

\begin{table}
    \centering
    \caption{\small Relations to previous task settings, including lifelong learning, multi-task learning, and federated learning.}
    \vspace{-3mm}
    \begin{tabular}{c|ccc}
         \toprule
         \textbf{Task} & \makecell[c]{\textbf{Dynamic}\\ \textbf{scenario}} &\makecell[c]{\textbf{Decentralized} \\ \textbf{mechanism}}  &\makecell[c]{\textbf{Number of}\\ \textbf{agents}}\\
         \midrule
         Lifelong learning & \usym{2713} & \usym{2717} &1\\
         Multi-task learning & \usym{2717} &\usym{2717}  &$N$\\
         Federated learning & \usym{2717} &\usym{2717}  &$N$\\
         \midrule
         \makecell[c]{\textbf{Ours (Decentralized} \\ \textbf{and lifelong-adaptive} \\ \textbf{collaborative learning)}} &\usym{2713} &\usym{2713} &$N$ \\
         \bottomrule
    \end{tabular}
    \vspace{-3mm}
    \label{tab:comparison}
\end{table}

\vspace{-3mm}
\subsection{Relations to Previous Task Settings}
\label{relation-setting}

We examine related learning settings: lifelong learning, multi-task learning, and federated learning, with relationships shown in \textbf{Table}~\ref{tab:comparison} and \textbf{Figure}~\ref{fig:related_work}.

\begin{figure}
    \centering
    \includegraphics[width=0.97\linewidth]{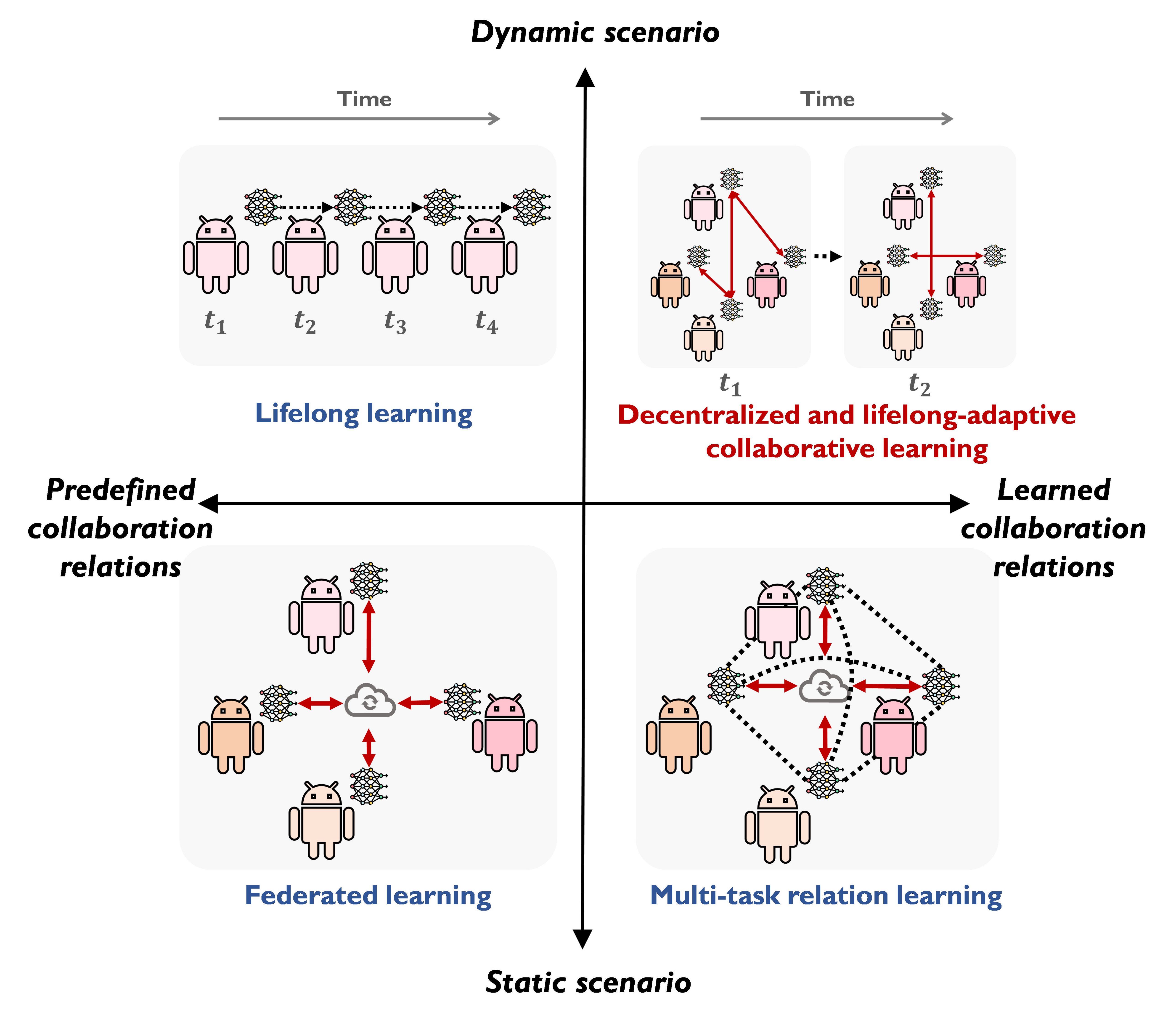}
    \vspace{-3mm}
    \caption{\small Relationships among our lifelong collaborative learning and its related works: lifelong learning, federated learning, and multi-task relation learning. }
    \vspace{-6mm}
    \label{fig:related_work}
\end{figure}

Their primary distinctions lie in two aspects. First, multi-task learning and federated learning focuses on static tasks, while our setting target tasks that can evolve over time. This dynamic learning scenario requires agents to accumulate and transfer previously encountered learning experiences, fostering adaptive learning and collaboration in evolving environments. Second, multi-task learning and federated learning are coordinated by the centralized server, while our setting utilizes decentralized collaboration mechanisms. This decentralized mechanism requires agents to actively seek a limited number of valuable collaborators on their own, rather than passively managed by the central server.

\section{Optimization for Decentralized and Lifelong-Adaptive Collaborative Learning}
\label{sec:co_lifelong}
In this section, we introduce our solution to the problem shown in \textbf{Section}~\ref{prob-formulate}. In \textbf{Section}~\ref{opt-prob} we introduce the overall mathematical optimization. In \textbf{Section}~\ref{sol-overview}, we provide an overview of our iterative solution, encompassing three main steps: local learning, collaborative relational inference, and lifelong model update, {whose detail is further provided in \textbf{Sections}~\ref{subsec-local-learning}, ~\ref{col-relation-infer}, and~\ref{lifelong-model-update}, respectively.}

\vspace{-3mm}
\subsection{Optimization Problem}
\label{opt-prob}
We consider solving this decentralized and lifelong-adaptive collaboration task shown in \textbf{Section}~\ref{task-setting} from a probabilistic perspective. According to the task setting described in \textbf{Section}~\ref{task-setting}, suppose the previous encountered datasets of agent $i$ until time $t$ is written as $\mathcal{D}_{\mathcal{T}_i}^t = \left\{\mathcal{D}_i^{(k)} \big | 1 \leq k \leq t\right\}$, where $\mathcal{T}_i$ represents the unique learning task's experience index of the $i$-th agent. Each time $t$ the collaboration system will estimate a group of model parameters $\mathbf{\Theta}^{(t)}= \left\{ \boldsymbol{\theta}_i^{(t)} \big | 1 \leq i \leq N\right\}$ given the group of the datasets $\mathcal{D}_{\mathcal{T}_{1:N}}^t$, where $\mathcal{D}_{\mathcal{T}_{1:N}}^t = \left\{ \mathcal{D}_{\mathcal{T}_i}^t \big | 1\leq i \leq N \right\}$ is the full dataset for agent $1\leq i \leq N$. Since this is a decentralized lifelong adaptive collaboration scheme where direct sharing training data is forbidden, each agent at time $t$ can only access to $p\left(\boldsymbol{\theta}_i^{(t)}\big |\mathcal{D}_i^{(t)}\right)$ or $p\left(\mathcal{D}_i^{(t)} \big | \boldsymbol{\theta}_i^{(t)}\right)$. Hence simply using the likelihood function $p\left(\mathcal{D}_{\mathcal{T}_{1:N}}^t \big| \mathbf{\Theta}^{(t)}\right)$ is impractical due to complex inter-agent correlations between agents' model parameters. Instead, we consider the Bayesian learning paradigm by maximizing the posterior distribution of $\mathbf{\Theta}^{(t)}$ as $p\left(\mathbf{\Theta}^{(t)} \big| \mathcal{D}_{\mathcal{T}_{1:N}}^t \right)$ after knowing the datasets $\mathcal{D}_{\mathcal{T}_{1:N}}^t$. The inter-agents' correlations are described by the prior distribution of $\mathbf{\Theta}^{(t)}$. Specifically, according to the Bayes rule, the posterior distribution can be decomposed into the following two parts:
\begin{equation}
\setlength\abovedisplayskip{2pt}
\setlength\belowdisplayskip{2pt}
    \label{MAP}
    \begin{aligned}
        p\left(\mathbf{\Theta}^{(t)} \big| \mathcal{D}_{\mathcal{T}_{1:N}}^t \right) \propto  p\left(\mathcal{D}_{\mathcal{T}_{1:N}}^t \big| \mathbf{\Theta}^{(t)} \right)p\left(\mathbf{\Theta}^{(t)}\right),
    \end{aligned}
\end{equation}
where the first probability is the likelihood of the agents model parameters on their datasets $\mathcal{D}_{\mathcal{T}_{1:N}}$, and the second probability $p\left(\mathbf{\Theta}^{(t)}\right)$ is the prior distribution of $\mathbf{\Theta}^{(t)}$.

\noindent \textbf{Likelihood probability $p\left(\mathcal{D}_{\mathcal{T}_{1:N}}^t \big| \mathbf{\Theta}^{(t)} \right)$ :} Since the training data arrives independently among the agents at each time $t$, we assume that the data batches $\mathcal{D}_{\mathcal{T}_i}^{(j)}$ satisfy conditional independency on both agent-level for $1\leq i \leq N$ and time-level for $1 \leq k \leq t$ shown in the following equation:
\begin{equation}
    \label{independence}
\setlength\abovedisplayskip{1pt}
\setlength\belowdisplayskip{1pt}
    \begin{aligned}
        & \textbf{Agent:~} p\left(\mathcal{D}_{\mathcal{T}_{1}}^t, \mathcal{D}_{\mathcal{T}_{2}}^t, \ldots , \mathcal{D}_{\mathcal{T}_{N}}^t \big | \mathbf{\Theta}^{(t)}\right) = \prod_{i=1}^N p\left(\mathcal{D}_{\mathcal{T}_{i}}^t \big | \mathbf{\Theta}^{(t)}\right), \\
        & \textbf{Time:~} p\left(\mathcal{D}_i^{(1)}, \mathcal{D}_i^{(2)}, \ldots, \mathcal{D}_i^{(t)} \big | \boldsymbol{\theta}_i^{(t)}\right) = \prod_{k=1}^{t} p\left(\mathcal{D}_i^{(k)} \big | \boldsymbol{\theta}_i^{(t)}\right), 
    \end{aligned}
\end{equation}
where the first equation represents agent-level conditional independence and the second denotes time-level conditional independence, standing for the task data $\mathcal{D}_{\mathcal{T}_{1:N}}^t$ arrives independently among the agents as time progresses. Thus we decompose the likelihood $p\left(\mathcal{D}_{\mathcal{T}_{1:N}}^t \big| \mathbf{\Theta}^{(t)} \right)$ as:
\begin{equation}
    \label{likelihood}
\setlength\abovedisplayskip{0pt}
\setlength\belowdisplayskip{0pt}
    \begin{aligned}
         p\left(\mathcal{D}_{\mathcal{T}_{1:N}}^t \big| \boldsymbol{\theta}_i^{(t)} \right) 
        &= \prod_{i=1}^{N} \prod_{k=1}^t p\left(\mathcal{D}_{i}^{(k)} \big| \boldsymbol{\theta}_i^{(t)} \right),
    \end{aligned}
\end{equation}
where $p\left(\mathcal{D}_{i}^{(j)} \big| \boldsymbol{\theta}_i^{(t)} \right)$ is the likelihood corresponding to the discriminative model for each agent. In ~\eqref{likelihood}, the decomposition follows from both the agent-level and time-level conditional independence as shown in~\eqref{independence}.

% the first step follows from the agent-level independence shown in ~\eqref{independence}; the second step stems from the assumption that other agents' model parameters have no direct influence on the conditional probability, which means $p\left(\mathcal{D}_{\mathcal{T}_i}^t \big| \mathbf{\Theta}^{(t)} \right) = p\left(\mathcal{D}_{\mathcal{T}_i}^t \big| \boldsymbol{\theta}_i^{(t)} \right)$ for $1 \leq i \leq N$; the third step follows from the time-level conditional independence shown in ~\eqref{independence}.

\begin{figure*}
    \centering
    \includegraphics[width=0.96\linewidth]{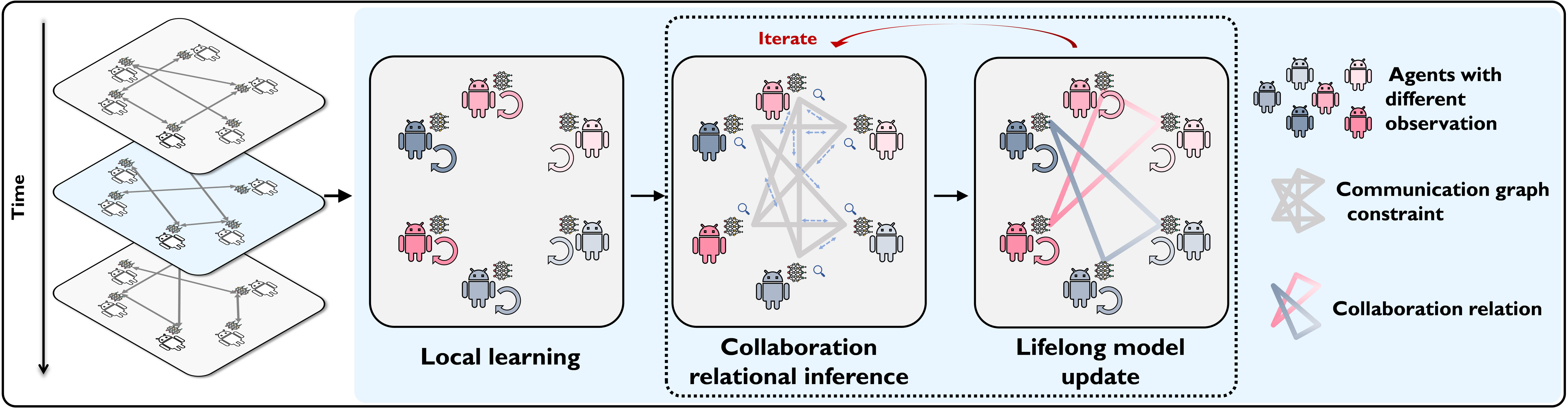}
    \vspace{-3mm}
    \caption{\small The decentralized and lifelong-adaptive multi-agent collaborative learning system. The system consists of three steps: local learning, collaborative relational inference, and lifelong model update. In the first step, each agent learns the model parameters based on their own observations to prepare the initialization of model parameters. In the second step, the agents transmit their own model parameters along the communication graph and learn the collaboration relation. In the last step, agents share their model parameters along the learned collaboration relation with their collaborators and refine their model parameters. Note the second and third steps are iterated several times until convergent.}
    \label{fig:framework}
        \vspace{-3mm}
\end{figure*}

\noindent \textbf{Prior probability $p\left(\mathbf{\Theta}^{(t)}\right)$ :} To model the prior information of $\mathbf{\Theta}^{(t)}$, we consider the probabilistic approach~\cite{graph-learning-smooth} by parameterizing this prior distribution with pairwise correlations among multiple agents' model parameters. Specifically, unlike the agent-level independence of the training data, these correlations between agents describe the relationships between model parameters, which correspond to a graph $\mathcal{G}^{(t)}=\{\mathcal{V}, \mathcal{E}^{(t)}\}$. Here $\mathcal{V}$ is the node set of the agents and $\mathcal{E}^{(t)}$ is the edge set at timestamp $t$. {The adjacency matrix of graph $\mathcal{G}^{(t)}$ is defined as $\mathbf{W}^{(t)}\in \mathbb{R}^{N\times N}$, indicating the collaboration relationships among the agents, used to guide the collaboration behaviors. Note that the collaboration relationship $\mathbf{W}^{(t)}$ is restricted by the communication graph $\mathbf{C}^{(t)}$ as its substructure. The weight of $\mathbf{W}^{(t)}$ represents the model similarity between agents. Assuming that similar model parameters promote agents to collaborate, where more similar model parameters correspond to stronger collaboration relations,} the prior distribution of model parameters $\mathbf{\Theta}^{(t)}$ given collaboration relationship $\mathbf{W}^{(t)}$ and the communication graph $\mathbf{C}^{(t)}$ are described by:
\begin{equation}
    \label{prior}
\setlength\abovedisplayskip{1pt}
\setlength\belowdisplayskip{1pt}
    \begin{aligned}
        -\log p\left(\mathbf{\Theta}^{(t)} \right)& \propto 
        \sum_{1 \leq i, j \leq N} \mathbf{W}_{ij}^{(t)} \left\| \boldsymbol{\theta}_i^{(t)} - \boldsymbol{\theta}_j^{(t)} \right\|_2^2 \\
        &=\operatorname{\textbf{tr}} \left(\mathbf{\Theta}^{(t) \top} \mathbf{L}^{(t)} \mathbf{\Theta}^{(t)}\right),
    \end{aligned}
\vspace{-1mm}
\end{equation}
where $\mathbf{W}^{(t)}_{ij}= 0$ if $\mathbf{C}^{(t)}_{ij}=0$. Here $\operatorname{\textbf{tr}}\left(\mathbf{\Theta}^{(t) \top} \mathbf{L}^{(t)} \mathbf{\Theta}^{(t)} \right)$ is the smoothness of $\mathbf{\Theta}^{(t)}$ viewed as functions on the graph and $\mathbf{L}^{(t)}$ is the combinatorial graph laplacian defined as $\mathbf{L}^{(t)} = \mathbf{D}^{(t)} - \mathbf{W}^{(t)}$ at time $t$. $\mathbf{D}^{(t)}$ is a diagonal matrix with its $i$-th entry representing the degree of the $i$-th node. 

Taking the likelihood probability shown in ~\eqref{likelihood} and the prior probability shown in ~\eqref{prior} to the original posterior probability $p\left(\mathbf{\Theta}^{(t)} \big| \mathcal{D}_{\mathcal{T}_{1:N}}^t \right)$ in ~\eqref{MAP}, we obtain a decomposed version of probability: 
\begin{equation}
\setlength\abovedisplayskip{1pt}
\setlength\belowdisplayskip{1pt}
\begin{aligned}
    p\left(\mathbf{\Theta}^{(t)} \big| \mathcal{D}_{\mathcal{T}_{1:N}}^t \right) 
    &\propto \prod_{i=1}^{N} \prod_{k=1}^t p\left(\mathcal{D}_{i}^{(k)} \big| \boldsymbol{\theta}_i^{(t)} \right) p\left(\mathbf{\Theta}^{(t)}\right).
\end{aligned}
\end{equation}
Thus we can reformulate~\eqref{MAP} as the following non-constraint optimization problem:
\begin{equation}
    \label{original}
\setlength\abovedisplayskip{1pt}
\setlength\belowdisplayskip{1pt}
    \begin{aligned}
    \min_{\boldsymbol{\theta}_i^{(t)}} &\sum_{i=1}^{N}\sum_{k=1}^{t} -\log p\left(\mathcal{D}_i^{(k)} \big | \boldsymbol{\theta}_i^{(t)} \right) + \lambda \operatorname{\textbf{tr}}\left(\mathbf{\Theta}^{(t) \top} \mathbf{L}^{(t)} \mathbf{\Theta}^{(t)} \right),
    \end{aligned}
\end{equation}
where the prior probability distribution of $\mathbf{\Theta}^{(t)}$ known as $p\left(\mathbf{\Theta}^{(t)}\right)$ corresponds to the graph smoothness regularization $\operatorname{\textbf{tr}}\left(\mathbf{\Theta}^{(t) \top} \mathbf{L}^{(t)} \mathbf{\Theta}^{(t)} \right)$. 

In practice, accessing the collaboration relationships $\mathbf{W}^{(t)}$ among agents is challenging due to i) the difficulties of quantifying task similarity; ii) the need to adapt to dynamically evolving tasks among agents; and iii) the necessity for a decentralized and autonomous mechanism to determine this structure. Hence, we incorporate the learning of collaboration graph into the optimization problem~(\ref{original}) where agents actively choose their collaborators, which takes both $\mathbf{W}^{(t)}$ and $\mathbf{\Theta}^{(t)}$ into consideration:
\begin{equation}
    \label{modified}
\setlength\abovedisplayskip{1pt}
\setlength\belowdisplayskip{1pt}
    \begin{aligned}
    \min_{\mathbf{\Theta}^{(t)}, \mathbf{W}^{(t)}} &\sum_{i=1}^N \mathcal{L}_i^{(t)}\left(\boldsymbol{\theta}_i^{(t)}\right) + \lambda_1 \left\| \boldsymbol{\theta}_i^{(t)} \right\|_2^2 + \mathcal{I}_{\geq 0}\left(\mathbf{W}^{(t)}\right)\\
    &+ \lambda_2 \operatorname{\textbf{tr}}\left(\mathbf{\Theta}^{(t) \top} \mathbf{L}^{(t)} \mathbf{\Theta}^{(t)} \right) + \lambda_3 \left\| \mathbf{W}^{(t)} \right\|_{\text{F}}^2   \\
    \text{s.t.} \  &\left\|\mathbf{W}^{(t)} \boldsymbol{1} \right\|_1 = \boldsymbol{m}, \ \operatorname{\textbf{diag}}\left(\mathbf{W}^{(t)} \right) = \boldsymbol{0}, \\
     & \mathbf{W}^{(t)}_{ij}= 0 \ \text{if} \ \mathbf{C}^{(t)}_{ij}=0, \ 1\leq i, j \leq N.
    \end{aligned}
\end{equation}
Here the accumulated loss of the $i$-th agent with the model parameter $\boldsymbol{\theta}_i^{(t)}$ is $\mathcal{L}_i^{(t)}\left(\boldsymbol{\theta}_i^{(t)}\right)=\frac{1}{t}\sum_{k=1}^t \ell^k \left(\boldsymbol{\theta}_i^{(t)}\right)$, where $\ell^k \left(\boldsymbol{\theta}_i^{(t)}\right) = -\log p\left(\mathcal{D}_i^{(k)} \big | \boldsymbol{\theta}_i^{(t)} \right)$ is the loss of the model $\boldsymbol{\theta}_i^{(t)}$ evaluated on the supervised dataset  $\mathcal{D}_i^{(k)}$. The edge weights of the collaboration graph $\mathbf{W}^{(t)}$ are set to be positive without self-loops by adding the indicator function $\mathcal{I}_{\geq 0}\left(\mathbf{W}^{(t)}\right)$ and the constraint $\operatorname{\textbf{diag}}\left(\mathbf{W}^{(t)} \right) = \boldsymbol{0}$ to the optimization.

Note that: i) to increase the generalization ability for models' parameter learning, we incorporate the $L_2$ regularization term associated with parameter $\lambda_1$; ii) to ensure non-trivial solutions and smooth edge weights in graph structure $\mathbf{W}^{(t)}$, we introduce the $L_2$ regularization term of $\mathbf{W}^{(t)}$ to the optimization problem with hyper-parameter $\lambda_3$; and iii) to maintain a constant norm and sparse structure in the collaboration graph, we normalize the graph $\mathbf{W}^{(t)}$ to hyper-parameter $\boldsymbol{m}$ using $L_1$ metric. 
\vspace{-3mm}
\subsection{Solution Overview for The Optimization Problem}
\label{sol-overview}
As a general solver of \textbf{Problem}~\eqref{modified}, $\mathcal{F}(\cdot)$ is a mapping that takes both the training data $\mathcal{D}_{1:N}^{(t)}=\left\{\mathcal{D}_i^{(t)}\right\}_{i=1}^N$ and the learning memories $\mathcal{M}_{1:N}^{(t-1)} = \left\{\mathcal{M}_{i}^{(t-1)}\right\}_{i=1}^{N}$ as input and then outputs the model parameters $\mathbf{\Theta}^{(t)}$ and collaboration relationships $\mathbf{W}^{(t)}$. Note that this solver $\mathcal{F}(\cdot)$ needs to follow the task setting as stated in~\textbf{Section}~\ref{task-setting}; that is, the process of learning model parameters should be decentralized under the condition of the communication graph $\mathbf{C}^{(t)}$. Mathematically, the solver $\mathcal{F}(\cdot)$ works as follows,
\begin{equation}
\setlength\abovedisplayskip{1pt}
\setlength\belowdisplayskip{1pt}
    \label{compact_solution}
    \begin{aligned}
        \mathbf{\Theta}^{(t)}, \mathbf{W}^{(t)}, \mathcal{M}_{1:N}^{(t)} = \mathcal{F}\left(\mathcal{D}_{1:N}^{(t)}, \mathcal{M}_{1:N}^{(t-1)} \big | \mathbf{C}^{(t)} ;\{\lambda_{i}\}_{i=1}^3 \right),
    \end{aligned}
\end{equation}
where $\left\{\lambda_{i}\right\}_{i=1}^3$ are hyper-parameters in \textbf{Problem}~\eqref{modified}. 

% This mapping $\mathcal{F}(\cdot)$ follows the same setting shown in \textbf{Section}~\ref{task-setting} 

% {\HC where the processes of generating shared messages and model parameter learning shown in \eqref{setting_eq} are not directly elaborated.}\Note{sc: confusing}

The main challenge of solving \textbf{Problem}~\eqref{modified} is the existence of the graph smoothness term $\operatorname{\textbf{tr}}\left(\mathbf{\Theta}^{(t) \top} \mathbf{L}^{(t)} \mathbf{\Theta}^{(t)} \right)$. Since $\mathbf{L}^{(t)}$ depends on $\mathbf{W}^{(t)}$, this term involves the interaction between $\mathbf{\Theta}^{(t)}$ and $\mathbf{W}^{(t)}$, causing the optimization problem to be non-convex. Hence we leverage alternate convex search to $\mathbf{\Theta}^{(t)}$ and $\mathbf{W}^{(t)}$, which could efficiently find stationary solutions for certain non-convex (biconvex) optimization problems~\cite{biconvex}. Based on this approach, our solver $\mathcal{F}(\cdot)$ can be  decomposed into three parts: i) \textbf{local learning} $\mathbf{\Phi}_{\rm local}(\cdot)$, which online updates the learning history $\mathcal{M}_{1:N}^{(t)}$ and initialize the model parameters $\mathbf{\Theta}^{(t)}$ given the newly observed task data $\mathcal{D}_{1:N}^{(t)}$; ii) \textbf{collaborative relational inference} $\mathbf{\Phi}_{\rm graph}(\cdot)$, which minimizes the objective function by finding the solution $\mathbf{W}^{(t)}$ with fixed $\mathbf{\Theta}^{(t)}$; and iii) \textbf{lifelong model update} $\mathbf{\Phi}_{\rm param}(\cdot)$, which optimizes the model parameters $\mathbf{\Theta}^{(t)}$ with the collaboration relations $\mathbf{W}^{(t)}$ held constant. Mathematically,  the solution is:
\begin{equation}
    \label{iterative_solution}
\setlength\abovedisplayskip{1pt}
\setlength\belowdisplayskip{1pt}
    \begin{aligned}
        &\mathbf{\Theta}_{\rm init}^{(t)}, \mathcal{M}_{1:N}^{(t)} = \mathbf{\Phi}_{\rm local}\left(\mathcal{D}_{1:N}^{(t)}, \mathcal{M}_{1:N}^{(t-1)}; \lambda_1 \right), \\
        &\mathbf{W}^{(t), k} = \mathbf{\Phi}_{\rm graph}\left(\mathbf{\Theta}^{(t), k-1} \big | \mathbf{C}^{(t)}; \lambda_2, \lambda_3 \right), \\
        &\mathbf{\Theta}^{(t), k} = \mathbf{\Phi}_{\rm param}\left(\mathbf{\Theta}^{(t), k-1}, \mathbf{W}^{(t), k}, \mathcal{M}_{1:N}^{(t)} \big | \mathbf{C}^{(t)} ; \lambda_1, \lambda_2 \right), \\
        & \text{for } 1\leq k \leq M.
    \end{aligned}
\end{equation}
Here the solver begins with the local learning initialization $\mathbf{\Phi}_{\rm local}(\cdot)$, which prepares the initialization of model parameters $\mathbf{\Theta}_{\rm init}^{(t)}$ used for collaborative relational inference, and then follows by the alternate convex search between $\mathbf{\Phi}_{\rm graph}(\cdot)$ and $\mathbf{\Phi}_{\rm param}(\cdot)$ for $M$ iterations.

\noindent \textbf{Local learning $\mathbf{\Phi}_{\rm local}(\cdot)$.} In this part, the algorithm tries to give an initialization of model parameters to launch the iterations between $\mathbf{\Theta}^{(t)}$ and $\mathbf{W}^{(t)}$. This initialization considers previous learning experiences by solving the following optimization problem:
\begin{equation}
    \label{local_learning}
\setlength\abovedisplayskip{1pt}
\setlength\belowdisplayskip{1pt}
    \begin{aligned}
        \min_{\mathbf{\Theta}^{(t)}} \sum_{i=1}^N\mathcal{L}_i^{(t)}\left(\boldsymbol{\theta}_i^{(t)}\right) + \lambda_1 \left\| \boldsymbol{\theta}_i^{(t)} \right\|_2^2.
    \end{aligned}
\end{equation}
Note that the learning memory $\mathcal{M}_{1:N}^{(t)}$ acts as intermediate variables essential for solving this optimization problem. Rather than using all historical datasets, the size of  $\mathcal{M}_{1:N}^{(t)}$ does not increase over time, prompting storage efficiency.

\noindent \textbf{Collaborative relational inference $\mathbf{\Phi}_{\rm graph}(\cdot)$.} In this part, the algorithm tries to solve the sub-optimization problem of $\mathbf{W}^{(t)}$ given fixed $\mathbf{\Theta}^{(t)}$. It aims to find the target collaboration structure by solving the following optimization problem:
\begin{equation}
    \label{colla_relation_learning}
\setlength\abovedisplayskip{1pt}
\setlength\belowdisplayskip{1pt}
    \begin{aligned}
        \min_{\mathbf{W}^{(t)}}
        &\lambda_2 \operatorname{\textbf{tr}}\left(\mathbf{\Theta}^{(t) \top} \mathbf{L}^{(t)} \mathbf{\Theta}^{(t)} \right)+ \mathcal{I}_{\geq 0}\left(\mathbf{W}^{(t)}\right) + \lambda_3 \left\| \mathbf{W}^{(t)}\right\|_{\rm F}^2 \\
        \text{s.t.} \  &\|\mathbf{W}^{(t)} \|_1 = \boldsymbol{m}, \ \operatorname{\textbf{diag}}\left(\mathbf{W}^{(t)} \right) = 0, \\
        & \mathbf{W}^{(t)}_{ij}= 0 \ \text{if} \ \mathbf{C}^{(t)}_{ij}=0, \ 1\leq i, j \leq N,
    \end{aligned}
\end{equation}
where $\mathcal{I}_{\geq 0}\left(\mathbf{W}^{(t)}\right)$ is the indicator function requiring edge weights of $\mathbf{W}^{(t)}$ to be positive.

\noindent \textbf{Lifelong model update $\mathbf{\Phi}_{\rm param}(\cdot)$.} Given the optimized collaboration relation $\mathbf{W}^{(t)}$, this part tries to find the model parameters $\mathbf{\Theta^{(t)}}$ by solving the corresponding non-constraint optimization problem:
\begin{equation}
    \label{colla_param_learning}
\setlength\abovedisplayskip{1pt}
\setlength\belowdisplayskip{1pt}
    \begin{aligned}
        \min_{\mathbf{\Theta}^{(t)}} &\sum_{i=1}^N \mathcal{L}_i^{(t)}\left(\boldsymbol{\theta}_i^{(t)}\right) + \lambda_1 \left\| \boldsymbol{\theta}_i^{(t)} \right\|_2^2 + \lambda_2 \operatorname{\textbf{tr}}\left(\mathbf{\Theta}^{(t) \top} \mathbf{L}^{(t)} \mathbf{\Theta}^{(t)} \right).
    \end{aligned}
\end{equation}

Note that to ensure the decentralized collaboration mechanism, both iterative steps $\mathbf{\Phi}_{\rm graph}(\cdot)$ and $\mathbf{\Phi}_{\rm param}(\cdot)$ should be operated under the communication graph $\mathbf{C}^{(t)}$ constraint. A visual representation of the overall solver can be found in \textbf{Figure}~\ref{fig:framework}. 
We now introduce each step in detail.

\vspace{-3mm}
\subsection{Local Learning}
\vspace{-1.5mm}
\label{subsec-local-learning}
Here we aim to solve the optimization problem shown in ~\eqref{local_learning}. Direct optimization of this objective function is impractical due to agents' inability to access data from previous tasks, preventing the computation of gradients from prior tasks, hence difficult for model parameter training. Starting from this point, once we can approximate previous functions' gradients, we obtain the information for the correct optimization direction of model parameter training. This approximation does not require the data samples from prior tasks, but merely the gradient information derived from the loss function. Drawing inspiration from the function approximation theory~\cite{rudin_analysis}, we consider using a second-order Taylor expansion to approximate the loss functions of agents' encountered tasks, which could provide the gradient information for future learning tasks. 

Specifically, recall that the definition of $\ell^{k}\left(\boldsymbol{\theta}_i^{(t)}\right)=-\log p\left(\mathcal{D}_i^{(k)} \big | \boldsymbol{\theta}_i^{(t)}\right)$ shown in ~\eqref{modified} is the loss function correspond to $\mathcal{D}_i^{(k)}$ for model parameter $\boldsymbol{\theta}_i^{(t)}$. Suppose we perform direct Taylor expansion at one point $\boldsymbol{\alpha}_{i}^{(k)}$ of $\ell^k \left(\boldsymbol{\theta}_i^{(t)}\right)$. Then, the accumulated loss $\mathcal{L}_i^{(t)}\left(\boldsymbol{\theta}_i^{(t)}\right) = \frac{1}{t} \sum_{k=1}^t \ell^k \left(\boldsymbol{\theta}_i^{(t)}\right)$ is approximated by
\begin{equation}
\setlength\abovedisplayskip{1pt}
\setlength\belowdisplayskip{1pt}
\begin{aligned}
    &\mathcal{L}_i^{(t)}\left(\boldsymbol{\theta}_i^{(t)}\right) \approx \frac{1}{t} \sum_{k=1}^{t} \frac{1}{2}\left(\boldsymbol{\theta}_i^{(t)} - \boldsymbol{\alpha}_i^{(k)}\right)^\top \mathbf{H}_i^{(k)} \left(\boldsymbol{\theta}_i^{(t)} - \boldsymbol{\alpha}_i^{(k)} \right) \\
    &+ \left(\boldsymbol{\theta}_i^{(t)} - \boldsymbol{\alpha}_i^{(k)} \right)^\top \boldsymbol{g}_i^ {(k)}+ \ell^{k}\left(\boldsymbol{\alpha}_i^{(k)}\right),
\end{aligned}
\end{equation}
where $\mathbf{H}_i^{(k)}$ is the Hessian of $\ell^{k}\left(\boldsymbol{\theta}_i^{(t)}\right)$ at $\boldsymbol{\alpha}_i^{(k)}$, and $\boldsymbol{g}_i^{(k)}$ is the corresponding gradient at the same point. Thus the gradient of $\boldsymbol{\theta}_i^{(t)}$ at time $t$ can be approximated as 
\begin{equation}
    \label{gradient}
\setlength\abovedisplayskip{1pt}
\setlength\belowdisplayskip{1pt}
    \begin{aligned}
        \frac{\partial \mathcal{L}_i^{(t)}\left(\boldsymbol{\theta}_i^{(t)}\right)}{\partial \boldsymbol{\theta}_i^{(t)}} &= \frac{1}{t} \sum_{k=1}^{t} \mathbf{H}_i^{(k)}\left(\boldsymbol{\theta}_i^{(t)} - \boldsymbol{\alpha}_i^{(k)}\right) + \boldsymbol{g}_i^{(k)}\\
        &= \mathbf{A}^{(t)} \boldsymbol{\theta}_i^{(t)} - \boldsymbol{b}_i^{(t)},
    \end{aligned}
\end{equation}
where $\mathbf{A}_i^{(t)}$, $\boldsymbol{b}_i^{(t)}$ are two intermediate variables that can be understood as the linear summation of historical variables. These two variables $\left(\mathbf{A}_i^{(t)}, \boldsymbol{b}_i^{(t)}\right)$, corresponding to the aforementioned learning memory $\mathcal{M}_i^{(t)}$, can be updated online in terms of $\mathbf{H}_i^{(t)}$, $\boldsymbol{\alpha}_i^{(t)}$ and $\boldsymbol{g}_i^{(t)}$ according to the following rule:
\begin{equation}
    \label{updation-intermediate}
\setlength\abovedisplayskip{1pt}
\setlength\belowdisplayskip{1pt}
    \begin{aligned}
        &\mathbf{A}_i^{(t)} = \left[(t-1)\mathbf{A}_i^{(t-1)} + \mathbf{H}_i^{(t)}\right]/t, \\
        &\boldsymbol{b}_i^{(t)} = \left[(t-1)\boldsymbol{b}_i^{(t-1)} + \mathbf{H}_i^{(t)}\boldsymbol{\alpha}_i^{(t)} - \boldsymbol{g}_i^{(t)}\right]/t.
    \end{aligned}
\end{equation}
Hence online optimization of model parameters $\mathbf{\Theta}^{(t)}$ can be achieved by continuously updating parameters’ gradient functions $\partial \mathcal{L}_i^{(t)}\left(\boldsymbol{\theta}_i^{(t)}\right) / \partial \boldsymbol{\theta}_i^{(t)}$ according to the approximation shown in ~\eqref{gradient} and updating rule ~\eqref{updation-intermediate}. Considering the first-order condition of problem~\eqref{local_learning}, $\partial \mathcal{L}_i^{(t)}\left(\boldsymbol{\theta}_i^{(t)}\right) / \partial \boldsymbol{\theta}_i^{(t)} +  2\lambda_1 \boldsymbol{\theta}_i^{(t)} = 0$, we take the approximation of the gradient function ~\eqref{gradient} into this first-order condition and obtain the following parameter initialization rule for each agent. The model parameters initialized by local learning is
\begin{equation}
\setlength\abovedisplayskip{1pt}
\setlength\belowdisplayskip{1pt}
    \begin{aligned}
    \boldsymbol{\theta}_i^{(t,0)} &= \left[\mathbf{A}_i^{(t)} + 2\lambda_1 \mathbf{I} \right]^{-1}\boldsymbol{b}_i^{(t)}.
    \end{aligned}  
\end{equation}
The full local learning framework $\mathbf{\Phi}_{\rm local}(\cdot)$ is shown in~\textbf{Algorithm}~\ref{alg1}. Note that i) all computations happen locally; and ii) the size of the learning memory $\mathcal{M}_i^{(t)} = \left(\mathbf{A}_i^{(t)}, \boldsymbol{b}_i^{(t)}\right)$ remains a constant, which is storage-friendly.

\begin{algorithm}[t]
\small
\begin{algorithmic}
\caption{Local learning $\mathbf{\Phi}_{\rm local}(\cdot)$}
\label{alg1}
\STATE \textbf{Input data: } $\mathcal{D}_{1:N}^{(t)} = \left\{\mathcal{D}_i^{(t)} \big | \mathcal{D}_i^{(t)} = \left(\mathbf{X}_i^{(t)}, \mathbf{Y}_i^{(t)}\right), 1\leq i \leq N \right\}$
\FOR{$i = 1,2, \ldots, N$ (\textbf{parallel})}
\STATE \textbf{Initialize expansion point} $\boldsymbol{\alpha}_i^{(t)}$
\STATE $\mathbf{H}_i^{(t)} = \nabla_{\boldsymbol{\alpha}_i^{(t)}}^2 \mathcal{L}\left(f_{\boldsymbol{\theta}_i^{(t)}}\left(\mathbf{X}_i^{(t)}\right), \mathbf{Y}_i^{(t)}\right)$   \COMMENT{  {Calculate Hessian}}
\STATE $\mathbf{A}_i^{(t)} =\left( (t-1) \mathbf{A}_i^{(t-1)}+\mathbf{H}_i^{(t)} \right)/t$  
 $\ \ $  \COMMENT{{Update history} $\mathbf{A}_i^{(t)}$}
\STATE $\boldsymbol{b}_i^{(t)} = \left((t-1)\boldsymbol{b}_i^{(t-1)} + \mathbf{H}_i^{(t)}\boldsymbol{\alpha}_i^{(t)}\right)/t$  $\ \ $  \COMMENT{{Update history} $\boldsymbol{b}_i^{(t)}$}
\STATE $\boldsymbol{\theta}_i^{(t,0)} = \left[\mathbf{A}_i^{(t)} + 2\lambda_1 \mathbf{I} \right]^{-1}\boldsymbol{b}_i^{(t)}$ $\ $  \COMMENT{{Calculate parameter} $\boldsymbol{\theta}_i^{(t, 0)}$}
\ENDFOR
\end{algorithmic}
\end{algorithm}
\vspace{-3mm}
\subsection{Collaborative Relational Inference}
\label{col-relation-infer}

Here we aim to solve the collaborative relational inference problem shown in ~\eqref{colla_relation_learning}. To enable the decentralized calculation of the collaboration relationships among the agents, we aim to split the graph adjacency matrix $\mathbf{W}^{(t)}$ into $N$ blocks. Specifically, similar to the reformulation trick shown in~\cite{L2G}, we define the block correspond to the $i$-th agent be $\boldsymbol{w}_i^{(t)} \in \mathbb{R}^{N-1}$ as the $i$-th row of the adjacency matrix $\mathbf{W}^{(t)}$ except the $i$-th element. Define the parameter distance vector of the $i$-th agent as $\boldsymbol{d}_i^{(t)} \in \mathbb{R}^{N-1}$. Due to the existence of communication constraint $\mathbf{C}^{(t)}$, we extend the definition of $\boldsymbol{d}_i^{(t)}$ by $\boldsymbol{d}_{ij}^{(t)} = \left\|\boldsymbol{\theta}_i^{(t)} - \boldsymbol{\theta}_j^{(t)} \right\|_2^2 / \mathbf{C}_{ij}^{(t)}$. Then, an equivalent form of \textbf{Problem}~\eqref{colla_relation_learning} is 
\begin{equation}
\label{graph_learning_optimization_problem}
   \begin{aligned}
\min_{\left(\boldsymbol{w}_1^{(t)},\ldots,\boldsymbol{w}_N^{(t)}\right)} \  & \sum_{i=1}^{N} \lambda_2\boldsymbol{w}_i^{(t) \top} \boldsymbol{d}_i^{(t)} + \lambda_3 \left\|\boldsymbol{w}_i^{(t)} \right\|_2^2 + \mathcal{I}_{x>0}\left(\boldsymbol{w}_i^{(t)}\right) \\
    & \text{s.t.} \ \sum_{i=1}^{N}\boldsymbol{1}^{\top}\boldsymbol{w}_i^{(t)} = \boldsymbol{m}.
\end{aligned} 
\end{equation}
To solve this optimization problem, we use the Lagrange multiplier method to analyze the KKT conditions of the solution $\boldsymbol{w}_i^{(t)}$s. This corresponds to solving the following single-variable equation
\begin{equation}
    \label{equation-z}
    \setlength\abovedisplayskip{1pt}
\setlength\belowdisplayskip{1pt}
    \begin{aligned}
        \sum_{i=1}^N \boldsymbol{1}^\top \operatorname{\textbf{ReLU}}\left(-\frac{\lambda_2\boldsymbol{d}_i^{(t)} + z\boldsymbol{1}}{2\lambda_3}\right) = \boldsymbol{m},
    \end{aligned}
\end{equation}
where $z$ corresponds to the Lagrange multiplier of the $L_1$ equality constraint and $\boldsymbol{1}\in \mathbb{R}^{N-1}$.

\begin{algorithm}[t]
\small
\begin{algorithmic}
\caption{Collaborative relational inference $\mathbf{\Phi}_{\rm graph}(\cdot)$}
\label{alg3}
\STATE \textbf{Input:} $\mathbf{\Theta}^{(t)} = \left(\boldsymbol{\theta}_1^{(t)}, \ldots, \boldsymbol{\theta}_n^{(t)}\right)$
% \STATE \textbf{Initialize:} $\boldsymbol{z}^0 $ 
\FOR{$k = 0, 1, 2, \ldots, M$}
\FOR{$i = 1, 2, \ldots , N$ (\textbf{parallel})}
    \STATE \textbf{Initialize:} $\boldsymbol{z}_i^0 = \boldsymbol{0} $, $1\leq i \leq N$ 

    \STATE $\boldsymbol{u}_i^k = -\left(\lambda_2\boldsymbol{d}_i^{(t)} + \boldsymbol{z}_i^{k-1} \boldsymbol{1}\right)/(2\lambda_3)$
    \STATE $\boldsymbol{w}_i^{(t), k}, \boldsymbol{w}_i^{\prime(t), k }, \boldsymbol{w}_i^{\prime \prime(t), k} = h\left(\boldsymbol{u}_i^k\right), h^\prime\left(\boldsymbol{u}_i^k\right), h^{\prime \prime} \left(\boldsymbol{u}_i^k\right)$
    
    \STATE $\boldsymbol{x}_i^{k} = \boldsymbol{u}_i^{k \top} \boldsymbol{w}_i^{\prime(t), k } + \boldsymbol{w}_i^{(t),k \top} \left(\boldsymbol{1} - \boldsymbol{w}_i^{\prime \prime (t), k}\right)$
    \STATE $\boldsymbol{s}_i^k = \left(2\boldsymbol{1} - \boldsymbol{w}_i^{\prime (t), k}\right)^\top \boldsymbol{w}_i^{\prime (t), k}$
    \STATE $\boldsymbol{r}_i^k = \left(\boldsymbol{u}_i^k - \boldsymbol{w}_i^{\prime(t), k}\right)^\top \boldsymbol{w}_i^{\prime \prime(t), k}$
    \STATE $\boldsymbol{y}_i^{k} = \boldsymbol{s}_i^k + \boldsymbol{r}_i^k$
    \STATE \textbf{Gather} $\boldsymbol{x}_j^{k}$, $\boldsymbol{y}_j^k$, $1 \leq j \leq N$ via $\mathbf{C}^{(t)}$ to  Agent $i$
    \STATE $\boldsymbol{p}_i^k = \sum_{j=1}^N \boldsymbol{x}_j^k - \boldsymbol{m}$
    \STATE $\boldsymbol{q}_i^k = -1/\lambda_3\sum_{j=1}^N \boldsymbol{y}_j^k $
    % \STATE $\boldsymbol{p}_i^k = \boldsymbol{x}_i^k - 1$
    % \STATE $\boldsymbol{q}_i^k = -1/\lambda_3 \boldsymbol{y}_i^k$
    \STATE $\boldsymbol{z}_i^k = \boldsymbol{z}_i^{k-1} -\boldsymbol{p}_i^k / \boldsymbol{q}_i^k$  \COMMENT{Dual update $z_i$}
\ENDFOR
\ENDFOR
\STATE \textbf{Output:} $\mathbf{W}^{(t)} \Leftarrow \left(\boldsymbol{w}_1^{(t),M}, \ldots, \boldsymbol{w}_N^{(t),M}\right)$
\end{algorithmic}
\end{algorithm}

To find the root $z$ of this non-smooth function, we propose an approximation method that tries to use a twice differentiable function $h(x) = (\sqrt{x^2 + b} + x)/2$ to approximate $\operatorname{\textbf{ReLU}}$. This approximation provides two advantages: i) it yields acceptable solutions that are easily derived; and ii) it enhances efficiency, requiring fewer iterations when employing second-order optimization methods. We leverage the Newton iterations to ~\eqref{equation-z}, which could converge to the target solution much faster than state-of-the-art graph learning algorithms~\cite{L2G}.

To enable the decentralized computation of $\boldsymbol{w}_i^{(t)}$ for each agent $i$, the critical part lies in the determination process of parameter $z$. This is because once the variable $z$ is determined, each agent could obtain their collaboration relation weights $\boldsymbol{w}_i^{(t)}$ independently according to the KKT condition:
\begin{equation}
\label{w_graph}
    \boldsymbol{w}_i^{(t)} = \operatorname{\textbf{ReLU}}\left(-\frac{\lambda_2\boldsymbol{d}_i^{(t)} + z\boldsymbol{1}}{2\lambda_3}\right) \ 1\leq i \leq N .
\end{equation}
However, due to the global constraint $\|\mathbf{W}^{(t)}\|_1 = \boldsymbol{m}$, the Newton iterations for $z$ require global information at every step, compromising the decentralization of this mechanism. To address this issue, we propose that each agent maintain a local copy of $z$, denoted as $z_i$ for the $i$-th agent, with the same initialization. Specifically, each agent determines $\boldsymbol{w}_i^{(t)}$ based on its own $z_i$. When updating their own $z_i$ values, agents gather information from their peers based on the communication graph $\mathbf{C}^{(t)}$ and update their local $z_i$ values accordingly. Note that since the updating rules for $z_i$ are identical and they share the same initialization, the values of $z_i$ for agents remain consistent throughout the solution process.

\textbf{Algorithm}~\ref{alg3} shows the decentralized iterations among the agents in detail. We further analyze Algorithm~\ref{alg3} from both theoretical and empirical aspects in \textbf{Section}s~\ref{sec:theory} and ~\ref{sec:experiments}, respectively.

\vspace{-3mm}
\subsection{Lifelong Model Update}
\label{lifelong-model-update}

Here we aim to solve the lifelong model update problem in ~\eqref{colla_param_learning}. To realize the decentralized optimization of $\mathbf{\Theta}^{(t)}$, we aim to optimize the problem in $N$ blocks $\mathbf{\Theta}^{(t)} = \left[\boldsymbol{\theta}_1^{(t)}, \ldots , \boldsymbol{\theta}_N^{(t)}\right]$. For each block of the model parameters $\boldsymbol{\theta}_i^{(t)}$, the first-order-condition of the problem shown in ~\eqref{colla_param_learning} can be described as a linear equation by taking the gradient approximation ~\eqref{gradient} into consideration:
\begin{equation}
    \label{foc}
    \begin{aligned}
        & \mathbf{A}_i^{(t)}\boldsymbol{\theta}_i^{(t)} - \boldsymbol{b}_i^{(t)}\\
        &= - 2\lambda_1 \boldsymbol{\theta}_i^{(t)} - 2\lambda_2 \sum_{j=1}^{N} \left(\mathbf{W}_{ij}^{(t)} + \mathbf{W}_{ji}^{(t)}\right)\left(\boldsymbol{\theta}_i^{(t)} - \boldsymbol{\theta}_j^{(t)}\right).
    \end{aligned}
\end{equation}

\begin{algorithm}[t]
\small
\begin{algorithmic}
\caption{Lifelong model update $\mathbf{\Phi}_{\rm param}(\cdot)$}
\label{alg2}
\STATE \textbf{Initialize:} $\left(\mathbf{A}_i^{(t)}, \boldsymbol{b}_i^{(t)}\right)$, $1\leq i \leq n$
\STATE \textbf{Input:} $\left(\boldsymbol{\theta}_1^{(t),0}, \ldots, \boldsymbol{\theta}_N^{(t), 0}\right)$, $\mathbf{W}^{(t)}$
\FOR{$k=1,2, \ldots, M$}
\FOR{$i = 1,2, \ldots, N$ (\textbf{parallel})}
\STATE send local model $\boldsymbol{\theta}_i^{(t), k-1}$ to agents in $\mathcal{N}(i)$
\STATE receive models $\left\{\boldsymbol{\theta}_j^{(t), k-1} \big | j \in \mathcal{N}(i) \right\}$ from agents in $\mathcal{N}(i)$
\STATE $\widetilde{\boldsymbol{\theta}_i}^{(t), k} \Leftarrow \boldsymbol{b}_i^{(t)} + 4\lambda_2 \sum_{j \in \mathcal{N}(i)} \mathbf{W}_{ij}^{(t)}\boldsymbol{\theta}_j^{(t),k-1}$
\STATE $\mathbf{B}_i^{(t)} \Leftarrow \mathbf{A}_i^{(t)} + \left(2\lambda_1 + 4\lambda_2 \mathbf{D}_i^{(t)}\right) \mathbf{I}$  
\STATE $\boldsymbol{\theta}_i^{(t),k} \Leftarrow \left( \mathbf{B}_i^{(t)} \right)^{-1}\widetilde{\boldsymbol{\theta}_i}^{(t), k}$ $\ \ $  
% \COMMENT{{$k$th iteration model updating}}
\ENDFOR
\ENDFOR
\STATE \textbf{Output:} $\left(\boldsymbol{\theta}_1^{(t)}, \ldots, \boldsymbol{\theta}_N^{(t)}\right) \Leftarrow \left(\boldsymbol{\theta}_1^{(t, M)}, \ldots, \boldsymbol{\theta}_N^{(t, M)}\right)$
\end{algorithmic}
\end{algorithm}

Hence optimizing $\mathbf{\Theta}^{(t)}$ corresponds to solving a linear system with $N$ first-order conditions. To obtain the decentralized computation and reduce the computational overhead, we consider Jacobi-iteration~\cite{matrix-computations} based approaches. Specifically, assuming the symmetry of the adjacency matrix $\mathbf{W}^{(t)}$, we rewrite ~\eqref{foc} as
\begin{align*}
    \left(\mathbf{A}_i^{(t)}+2\lambda_1\mathbf{I} + 4\lambda_2 \mathbf{D}_i^{(t)}\right)\boldsymbol{\theta}_i^{(t)} 
    = \boldsymbol{b}_i^{(t)} + 4\lambda_2 \sum_{j=1}^N \mathbf{W}_{ij}^{(t)}\boldsymbol{\theta}_j^{(t)},
\end{align*}
where $\mathbf{D}_i^{(t)}$ is the degree of the $i$-th agent of the collaboration relation $\mathbf{W}^{(t)}$. Our approach simply modifies this linear equation into parameter-updating rules. At the $k$th iteration, the model parameter $\boldsymbol{\theta}_i^{(t), k}$ of the $i$-th agent at time $t$ is updated according to the rule:
\begin{equation}
    \label{eq:message_passing}
    \begin{aligned}
        \boldsymbol{\theta}_i^{(t), k} = {\left(\mathbf{B}_i^{(t) }\right)}^{-1}\left[\boldsymbol{b}_i^{(t)} + 4\lambda_2 \sum_{j \in \mathcal{N}(i)} \mathbf{W}_{ij}^{(t)}\boldsymbol{\theta}_j^{(t), k-1}\right],
    \end{aligned}
\end{equation}
where $\mathbf{B}_i^{(t)} = \mathbf{A}_i^{(t)}+2\lambda_1\mathbf{I} + 4\lambda_2 \mathbf{D}_i^{(t)}$ (assuming $\mathbf{B}_i^{(t)}$ is invertible) and $\mathcal{N}(i)$ is the neighborhood of agent $i$ according to $\mathbf{W}^{(t)}$. Hence each agent can independently update their model parameters using both their individual task information and their collaborators' parameters. The process of decentralized lifelong model update is shown in \textbf{Algorithm}~\ref{alg2}.

Note that: 1) the parameter updating rule enables decentralized computation of model parameters $\boldsymbol{\theta}_i^{(t)}$ through message passing under almost {$\mathcal{O}\left(\log({1}/{\epsilon})\right)$ iterations}; 2) our function approximation approach shown in ~\eqref{gradient} does not require positive definiteness of Hessian $\mathbf{H}_i^{(t)}$ compared to online Laplace approximation approaches~\cite{online-laplace1, online-laplace2}, with more flexibility such as the approximation point. In \textbf{Section}~\ref{sec:theory} we will discuss the best point for Taylor expansion of choosing $\boldsymbol{\alpha}_i^{(t)}$ for certain kinds of learning tasks and the convergence property and $\mathbf{B}_i^{(t)}$'s invertibility.

\begin{figure*}
    \centering
    \includegraphics[width=0.96\textwidth]{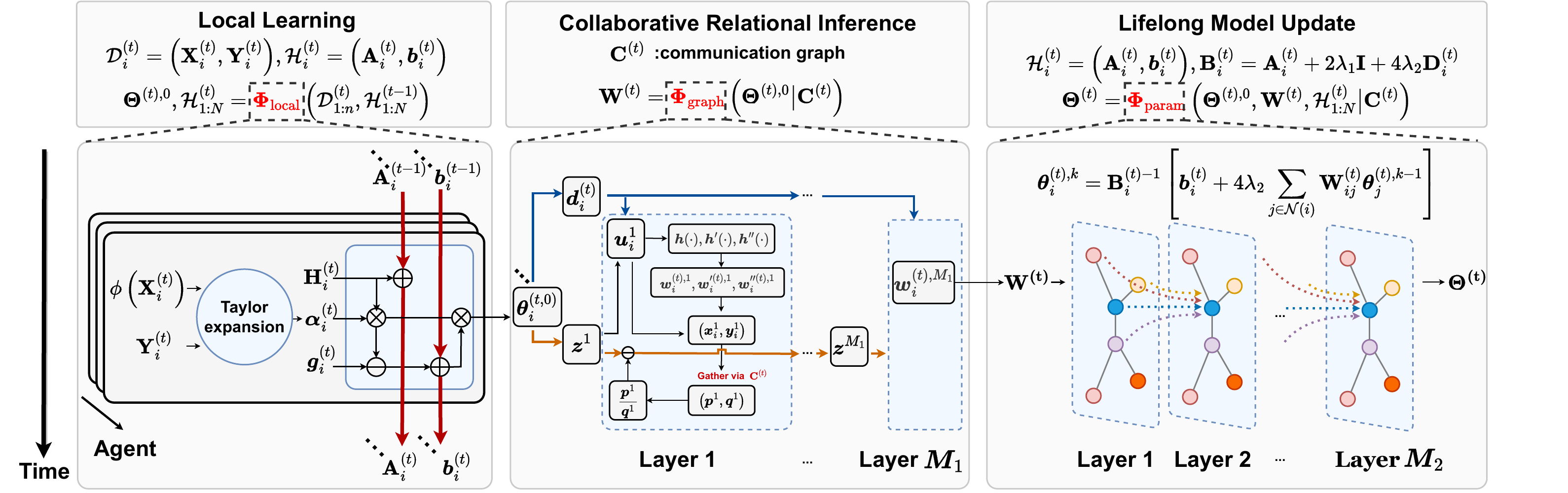}
    \vspace{-3mm}
    \caption{\small The unrolled network structure of the collaboration system with input data at one exact time. For each agent, the training data is firstly passed into a feed-forward network to be transformed into embeddings, and then this embedding is used to calculate the initialized model parameters according to $\mathbf{\Phi}_{\rm local}(\cdot)$. Then the agents start to communicate and collaborate to find proper parameters according to the iterations shown in $\mathbf{\Phi}_{\rm graph}(\cdot)$ and $\mathbf{\Phi}_{\rm param}(\cdot)$. Finally, the network $\mathcal{F}_{\gamma}(\cdot)$ outputs parameters $\mathbf{\Theta}^{(t)}$. These output model parameters $\mathbf{\Theta}^{(t)}$ are supervised by the data $(\mathcal{X}, \mathcal{Y}) \sim \mathcal{T}_i^{train}$. The training process of the network $\mathcal{F}_{\gamma}(\cdot)$ is learning to tune the parameter $\boldsymbol{\theta}$ to find the optimal parameter learning strategies of $\mathbf{\Theta}^{(t)}$.}
    \label{fig:unrolled_structure}
    \vspace{-3mm}
\end{figure*}

% unrolling
\begin{algorithm}[t]
\small
\caption{\texttt{DeLAMA}}
\label{alg4}
\begin{algorithmic}
\STATE \textbf{Network parameter: }\textcolor{blue}{$\gamma = \left(\beta, \{\lambda_i\}_{i=1}^3 \right)$}
\FOR{$t = 1,2, \ldots T$}

\STATE \textbf{Input:} $\mathcal{D}_{1:N}^{(t)} = \left\{\mathcal{D}_i^{(t)} \big | \mathcal{D}_i^{(t)} = \left(\mathcal{X}_i^{(t)}, \mathcal{Y}_i^{(t)}\right), 1\leq i \leq N \right\}$
\\$/*$ {Unfold the iterations of $\mathbf{\Phi}_{\rm local}(\cdot)$} \hfill $*/$
\FOR{$i = 1, 2, \ldots , N$ \textbf{parallel}} 
\STATE $\widehat{\mathcal{X}_i}^{(t)} = \phi_{\textcolor{blue}{\boldsymbol{\beta}}}\left(\mathcal{X}_i^{(t)}\right) $
\STATE \textbf{Initialize expansion point} $\boldsymbol{\alpha}_i^{(t)}$
\STATE $\mathbf{H}_i^{(t)} = \nabla_{\boldsymbol{\alpha}_i^{(t)}}^2 \mathcal{L}\left(f_{\boldsymbol{\theta}_i^{(t)}}\left(\widehat{\mathcal{X}_i}^{(t)}\right), \mathcal{Y}_i^{(t)}\right)$   \COMMENT{{Hessian}}
\STATE $\mathbf{A}_i^{(t)} =\left( (t-1) \mathbf{A}_i^{(t-1)}+\mathbf{H}_i^{(t)} \right)/t$  $\ \ $  \COMMENT{{Update} $\mathbf{A}_i^{(t)}$}
\STATE $\boldsymbol{b}_i^{(t)} = \left((t-1)\boldsymbol{b}_i^{(t-1)} + \mathbf{H}_i^{(t)}\boldsymbol{\alpha}_i^{(t)}\right)/t$ $\ \ $  \COMMENT{{Update} $\boldsymbol{b}_i^{(t)}$}
\STATE $\boldsymbol{\theta}_i^{(t,0)} = \left[\mathbf{A}_i^{(t)} + 2\textcolor{blue}{\lambda_1} \mathbf{I} \right]^{-1}\boldsymbol{b}_i^{(t)}$ $\ \ \ $  \COMMENT{{Calculate} $\boldsymbol{\theta}_i^{(t, 0)}$}
\ENDFOR
\\$/*$  {Unfold the iteration of $\mathbf{\Phi}_{\rm graph}(\cdot)$} \hfill$*/$ \\
\FOR{$k = 0, 1, 2, \ldots, M_1$}
\FOR{$i = 1, 2, \ldots , N$ \textbf{parallel}}
    \STATE \textbf{Initialize:} $\boldsymbol{z}_i^0 $, $1\leq i \leq N$ 

    \STATE $\boldsymbol{u}_i^k = -\left(\textcolor{blue}{\lambda_2}\boldsymbol{d}_i^{(t)} + \boldsymbol{z}_i^{k-1} \boldsymbol{1}\right)/(2\textcolor{blue}{\lambda_3})$
    \STATE $\boldsymbol{w}_i^{(t), k}, \boldsymbol{w}_i^{\prime(t), k }, \boldsymbol{w}_i^{\prime \prime(t), k} = h\left(\boldsymbol{u}_i^k\right), h^\prime\left(\boldsymbol{u}_i^k\right), h^{\prime \prime} \left(\boldsymbol{u}_i^k\right)$
    
    \STATE $\boldsymbol{x}_i^{k} = \boldsymbol{u}_i^{k \top} \boldsymbol{w}_i^{\prime(t), k } + \boldsymbol{w}_i^{(t),k \top} \left(\boldsymbol{1} - \boldsymbol{w}_i^{\prime \prime (t), k}\right)$
    \STATE $\boldsymbol{s}_i^k = \left(2\boldsymbol{1} - \boldsymbol{w}_i^{\prime (t), k}\right)^\top \boldsymbol{w}_i^{\prime (t), k}$
    \STATE $\boldsymbol{r}_i^k = \left(\boldsymbol{u}_i^k - \boldsymbol{w}_i^{\prime(t), k}\right)^\top \boldsymbol{w}_i^{\prime \prime(t), k}$
    \STATE $\boldsymbol{y}_i^{k} = \boldsymbol{s}_i^k + \boldsymbol{r}_i^k$
    \STATE \textbf{Gather} $\boldsymbol{x}_j^{k}$, $\boldsymbol{y}_j^k$, $1 \leq j \leq N$ via $\mathbf{G}^{(t)}$ to Agent $i$
    \STATE $\boldsymbol{p}_i^k = \sum_{j=1}^N \boldsymbol{x}_j^k - \boldsymbol{m}$
    \STATE $\boldsymbol{q}_i^k = -1/\textcolor{blue}{\lambda_3}\sum_{j=1}^N \boldsymbol{y}_j^k $
    % \STATE $\boldsymbol{p}_i^k = \boldsymbol{x}_i^k - 1$
    % \STATE $\boldsymbol{q}_i^k = -1/\lambda_3 \boldsymbol{y}_i^k$
    \STATE $\boldsymbol{z}_i^k = \boldsymbol{z}_i^{k-1} -\boldsymbol{p}_i^k / \boldsymbol{q}_i^k$ \COMMENT{Dual update $z_i$}
\ENDFOR
\ENDFOR
\STATE $\mathbf{W}^{(t)} = \left(\boldsymbol{w}_1^{M_1}, \ldots, \boldsymbol{w}_N^{M_1}\right)$
\\$/*$  {Unfold the iteration of $\mathbf{\Phi}_{\rm param}(\cdot)$} \hfill$*/$ \\
\FOR{$k=1,2, \ldots, M_2$}
\FOR{$i = 1,2, \ldots, N$ \textbf{parallel}}
% \\$/*$  {Unfold the iteration of $\mathbf{\Phi}_{\rm graph}(\cdot)$} $*/$ \\
\STATE {$/*$ Aggregate model parameters according to $\mathbf{W}^{(t)}$ \hfill $*/$}
\STATE $\widetilde{\boldsymbol{\theta}_i}^{(t), k} \Leftarrow \boldsymbol{b}_i^{(t)} + 4\textcolor{blue}{\lambda_2} \sum_{j \in \mathcal{N}(i)} \mathbf{W}_{ij}^{(t)}\boldsymbol{\theta}_j^{(t),k-1}$
\STATE $\mathbf{B}_i^{(t)} \Leftarrow \mathbf{A}_i^{(t)} + \left(2\textcolor{blue}{\lambda_1} + 4\textcolor{blue}{\lambda_2} \mathbf{D}_i^{(t)}\right) \mathbf{I}$ 
\STATE $\boldsymbol{\theta}_i^{(t),k} \Leftarrow \left(\mathbf{B}_i^{(t)}\right)^{-1}\widetilde{\boldsymbol{\theta}_i}^{(t), k}$ $\ \ $ 
% \COMMENT{{Model in the $k$th iteration}}
\ENDFOR
\ENDFOR
\STATE \textbf{Output:} $\left( \boldsymbol{\theta}_1^{(t, M_2)}, \ldots , \boldsymbol{\theta}_N^{(t, M_2)}\right)$
\ENDFOR
\end{algorithmic}
\end{algorithm}

\vspace{-3mm}
\section{Algorithm Unrolling for Decentralized and Lifelong-Adaptive Multi-Agent Learning}
\label{sec:unroll}
While the iterative algorithm (summarized in~\eqref{iterative_solution}) enables agents to learn and collaborate, it encounters certain challenges that limit its real-world applicability. First, the Taylor expansion in $\mathbf{\Phi}_{\rm local}(\cdot)$ has potential approximation errors for non-linear functions, limiting the method's expressive capability for non-linear learning tasks. Second, the collaborative learning system in~\eqref{modified} has many hyper-parameters, demanding significant effort for searching and tuning hyper-parameters. Third, the algorithm requires numerous iterations. Single graph learning and message passing steps involve many iterations and alternating convex search further requires alternate iterations between these two steps. These numerous iterations cause significant communication overhead, making the algorithm less practical for real applications. 

To solve these issues, we propose \texttt{DeLAMA}, a powerful collaborative learning framework suitable for a wide range of cognitive tasks, combining both mathematical optimization and the learning ability of neural networks. Through algorithm unrolling~\cite{unrolling-plugplay}, \texttt{DeLAMA} transforms the iterations of our original optimization solutions into deep neural networks, which enhances the expressive power, provides automatic learning schemes for hyperparameter-tuning, and further streamlines the learning process by reducing the number of alternating convex search iterations. Generally,  \texttt{DeLAMA} embodies the learning-to-learn paradigm by creating a universally applicable collaborative learning strategy, regardless of the training data's distribution. The learnable hyper-parameters signify a meta-level knowledge acquisition, targeting the mastery of the model's foundational mechanisms. Once established, this adaptive strategy can be applied to new tasks without reconfiguration.

In this section, we first describe the structure of the unrolled network, then introduce the learning-to-learn approaches to train the hyper-parameters of this network.
\vspace{-4mm}
\subsection{Unrolling Network Design}
\vspace{-1mm}
Recall the compact solution shown in~\eqref{compact_solution}, which is a mechanism to learn model parameters given training data $\mathcal{D}_{1:N}^{(t)}$.  We aim to use the algorithm unrolling techniques to transform~\eqref{compact_solution} into
the following neural-network-based unrolled mapping with learnable hyperparameters $\gamma$:
\begin{equation}
    \label{eq:unroll_network_forward}
    \setlength\abovedisplayskip{1pt}
\setlength\belowdisplayskip{1pt}
    \begin{aligned}
        \mathbf{\Theta}^{(t)}, \mathbf{W}^{(t)}, \mathcal{M}_{1:N}^{(t)} = \mathcal{F}_{\gamma}\left(\mathcal{D}_{1:N}^{(t)}, \mathcal{M}_{1:N}^{(t-1)} \big | \mathbf{C}^{(t)}\right).
    \end{aligned}
\end{equation}
The inputs are the training data and the output is the model parameters $\mathbf{\Theta}^{(t)}$. This can be understood as using a neural network $\mathcal{F}_{\gamma}(\cdot)$ to learn the model parameter $\mathbf{\Theta}^{(t)}$, where $\gamma$ is the hyperparameter controlling the network $\mathcal{F}_{\gamma}(\cdot)$ to obtain model parameters $\mathbf{\Theta}^{(t)}$.

Specifically, we upgrade the iterative steps in~\eqref{iterative_solution} to the unrolled network layers associated with $\mathcal{F}_{\gamma}(\cdot)$:
\begin{subequations}
    \label{eq:original_opt}
\setlength\abovedisplayskip{1pt}
\setlength\belowdisplayskip{1pt}
    \begin{align}
        &\mathbf{\Theta}^{(t), 0}, \mathcal{M}_{1:N}^{(t)} = \mathbf{\Phi}_{\rm local}\left(\mathcal{D}_{1:N}^{(t)}, \mathcal{M}_{1:N}^{(t-1)}; \lambda_1, \beta \right), \label{eq: init} \\
        &\mathbf{W}^{(t)} = \mathbf{\Phi}_{\rm graph}\left(\mathbf{\Theta}^{(t), 0} \big | \mathbf{C}^{(t) }; \lambda_2, \lambda_3 \right), \label{eq: graph} \\
        &\mathbf{\Theta}^{(t)} = \mathbf{\Phi}_{\rm param}\left(\mathbf{\Theta}^{(t), 0}, \mathbf{W}^{(t)}, \mathcal{M}_{1:N}^{(t)} \big | \mathbf{C}^{(t)}; \lambda_1, \lambda_2\right), \label{eq: param}
    \end{align}
\end{subequations}
where $\gamma = \left(\beta, \{\lambda_i\}_{i=1}^3\right)$ is the learnable parameters of network $\mathcal{F}_{\gamma}(\cdot)$. Here $\{\lambda_i\}_{i=1}^3$ are learnable hyperparameters and $\beta$ is the model parameter of a neural network.

{Note that: 1) to obtain a framework with fewer computations and less communication overhead, we reduce the number of iterations of alternative convex search between $\mathbf{\Phi}_{\rm graph}(\cdot)$ and $\mathbf{\Phi}_{\rm param}(\cdot)$ to 1 since the nonlinear operation corresponding to $\mathbf{\Phi}_{\rm param}(\cdot)$ is capable to find better initialization of model parameters $\mathbf{\Theta}^{(t),0}$, which could simplify the process of finding collaboration relationships and model update, hence requires fewer iterations to converge; and 2) to autonomously tune the hyperparameters $\{\lambda_i\}^3_{i=1}$, we leverage the supervised training paradigm of algorithm unrolling by treating the hyperparameters as learnable parameters.}

{To increase the learning flexibility and expressive power of the framework, according to the calculate Hessian step of $\mathbf{\Phi}_{\rm local}(\cdot)$ shown in \textbf{Algorithm}~\ref{alg1}, we add an non-linear backbone network $\phi_\beta$ with learnable parameters $\beta$ serving as the feature extractor:}
\begin{subequations}
    \label{eq: local_learning_unroll}
    \setlength\abovedisplayskip{1pt}
\setlength\belowdisplayskip{1pt}
     \begin{align}
        \widehat{\mathcal{X}_i}^{(t)} &= \phi_{\boldsymbol{\beta}}\left(\mathcal{X}_i^{(t)}\right), \label{eq:nonlinear_trans}\\
        \mathbf{H}_i^{(t)} &= \nabla_{\boldsymbol{\alpha}_i^{(t)}}^2 \mathcal{L}\left(f_{\boldsymbol{\theta}_i^{(t)}}\left(\widehat{\mathcal{X}_i}^{(t)}\right), \mathcal{Y}_i^{(t)}\right). \label{eq: taylor_expand_unroll}
    \end{align}   
\end{subequations}

Compared to those iterations shown in \textbf{Algorithm}~\ref{alg1}, the nonlinear transform $\phi_{\boldsymbol{\beta}}(\cdot)$ used in ~\eqref{eq:nonlinear_trans} targets to embed the input data $\mathcal{X}_i^{(t)}$ as $\widehat{\mathcal{X}_i}^{(t)}$, and uses this transformed data as the input of the model $f_{\boldsymbol{\theta}_i^{(t)}}(\cdot)$ in (\ref{eq: taylor_expand_unroll}). This approach offers two advantages. First, by incorporating the nonlinear backbone of the local learning task, the shared model $\boldsymbol{\theta}_i^{(t)}$ becomes more streamlined with a reduced number of parameters. This directly leads to a substantial reduction in communication overhead. Second, a simplified model $\boldsymbol{\theta}_i^{(t)}$ could also simplify the collaboration mechanism within $\mathbf{\Phi}_{\rm graph}(\cdot)$ and $\mathbf{\Phi}_{\rm param}(\cdot)$, making the collaboration process more efficient.   

In summary, for ~\eqref{eq: init}, we add a backbone with parameter $\beta$, for ~\eqref{eq: init}~\eqref{eq: graph}~\eqref{eq: param} we change the hyperparameters $\{\lambda_{i}\}_{i=1}^3$ to learnable parameters. Each iteration of the overall framework shown in ~\eqref{iterative_solution} is unfolded into one specific layer of the deep neural network. The full inference network framework of \texttt{DeLAMA} is shown in \textbf{Figure}~\ref{fig:unrolled_structure} and \textbf{Algorithm}~\ref{alg4}, with learnable hyper-parameters highlighted in blue.

\vspace{-5mm}
\subsection{Training Details}
\vspace{-2mm}
Different from traditional training pipeline: training model parameters, then inference; here with algorithm unrolling, the training procedure of \texttt{DeLAMA} comprises two steps: the learning-to-learn step and the model learning step. This enables \texttt{DeLAMA} to capture an efficient collaboration strategy suitable for various tasks. First, in the learning-to-learn step, we train the hyperparameters $\gamma$ in the unrolled network $\mathcal{F}_\gamma(\cdot)$, optimizing a collaboration strategy. Second, in the model learning step, by executing a forward-pass of $\mathcal{F}_\gamma(\cdot)$ with a fixed hyperparameter, we obtain the output $\mathbf{\Theta}^{(t)}$, representing the agents' model parameters. {Through this, the advantage is enabling the collaboration mechanism to learn meta knowledge on how to collaborate under different task configurations.} After training, each agent can infer using the model with parameter $\mathbf{\Theta}^{(t)}$ as usual.

\vspace{-3mm}
\subsubsection{ Learning to learn the collaboration strategy} 
We leverage the learning-to-learn paradigm to the training process of network $\mathcal{F}_{\gamma}(\cdot)$. Specifically, to train an efficient parameter learning network $\mathcal{F}_{\gamma}(\cdot)$, we use a bunch of training tasks, each represented by $\mathcal{T}_{1:N}^{train}$ where $\mathcal{T}_i^{train} \sim \mathcal{P}\left(\mathcal{T}_i\right)$ for $1 \leq i \leq N$, to provide sufficient examples and supervisions to tune the parameters $\gamma$. Here $\mathcal{P}(\mathcal{T}_i)$ represents the task distribution of the $i$-th agent, which could be used to sample different training task sequences. After this training, the collaboration strategy can be applied to tasks $\mathcal{T}_i$ following the same distribution.

\noindent \textbf{Network supervision.} The network supervision for one task $\mathcal{T}_{1:N}^{train}$ is the expected system average task loss:
\begin{align*}
    &\ell \left(\mathcal{T}_{1:N}^{train} \big | \gamma \right) \\
    &= \frac{1}{Nt}\sum_{i=1}^N \sum_{k=1}^{t} \mathbb{E}_{\left(\mathcal{X}_i, \mathcal{Y}_i\right) \sim \mathcal{T}_{i}^{train}}\left[\mathcal{L}\left(f_{\boldsymbol{\theta}_i^{(k)}}(\mathcal{X}_i), \mathcal{Y}_i\right)\right],
\end{align*}
where $\mathcal{L}(\cdot)$ is the task-specific loss and $f_{\boldsymbol{\theta}_i^{(k)}}(\cdot)$ is the learned model with parameter $\boldsymbol{\theta}_i^{(k)}$ generated from the network $\mathcal{F}_{\gamma}(\cdot)$ at timestamp $k$. 

\noindent \textbf{Network training.} Training network $\mathcal{F}_{\gamma}(\cdot)$ corresponds to figuring out the best mechanism to collaborate and learn model parameters $\mathbf{\Theta}^{(t)}$ according to the task $\mathcal{T}_{1:N}^{train}$. Mathematically, this means minimizing the expected supervision of the training tasks $\mathcal{T}_{1:N}^{train}$ according to the rule:
\begin{align*}
    \gamma^* = \arg \min_{\gamma} \mathbb{E}_{\mathcal{T}_{1:N}^{train} \sim \mathcal{P}\left(\mathcal{T}_{1:N}\right)} \ell \left(\mathcal{T}_{1:N}^{train} \big | \gamma \right),
\end{align*}

\noindent \textbf{Network evaluation.} Once trained, the optimal $\mathcal{F}_{\gamma^*}(\cdot)$ implicitly carries prior collaboration strategies that can be applied to new task learning configurations following the distribution $\mathcal{P}(\mathcal{T}_{1:N})$. In the evaluation phase, the evaluation metric is the expected network supervision loss on the test tasks $\mathcal{T}_{1:N}^{test}$:
\begin{align*}
    L = \mathbb{E}_{\mathcal{T}_{1:N}^{test} \sim \mathcal{P}\left(\mathcal{T}_{1:N}\right)} \ell \left(\mathcal{T}_{1:N}^{test} \big | \gamma^* \right),
\end{align*}
 where $\gamma^*$ is the optimal parameter of $\mathcal{F}_{\gamma}(\cdot)$.

This training approach can be understood as utilizing plenty of learning tasks $\mathcal{T}_{1:N}^{train}$ as examples to supervise the network how to learn model parameters, and subsequently applying the acquired knowledge to new learning tasks.

\vspace{-3mm}

% \subsubsection{Learning and inference of model parameters}
\subsubsection{Agent's model learning}
Learning the model parameters $\mathbf{\Theta}^{(t)}$ corresponds to the forward process of the unrolled network $\mathcal{F}_{\gamma^*}(\cdot)$ defined in~\eqref{eq:unroll_network_forward}, which takes the training data $\mathcal{D}_{1:N}^{(t)}$ as input and outputs $\mathbf{\Theta}^{(t)}$. During the forward pass of $\mathcal{F}_{\gamma^*}(\cdot)$, the finite learning memory $\mathcal{M}_{1:N}^{(t-1)}$ is recurrently updated, dynamically adapting to future observed tasks. Compared to the optimization procedure $\mathcal{F}(\cdot)$ defined in~\eqref{compact_solution}, the number of communication steps of $\Phi_{\rm graph}(\cdot)$ and $\Phi_{\rm param}(\cdot)$ is reduced, which further reduces the computational cost among multiple agents.
\vspace{-3mm}
\subsubsection{Agent's model inference}
Given the evaluation set $\widetilde{\mathcal{D}}_i^{(t)} = \left\{\widetilde{\mathcal{X}}_i^{(t)}, \widetilde{\mathcal{Y}}_i^{(t)}\right\}$ at timestamp $t$, agents calculate the predictions according to their learned model $\boldsymbol{\theta}_i^{(t)}$ and trained backbone $\phi_\beta(\cdot)$ by $f_{\boldsymbol{\theta}_i^{(t)}} \circ \phi_\beta\left(\widetilde{\mathcal{X}}_i^{(t)}\right)$. Here the model parameters are learned from the process $\mathcal{F}_{\gamma^*}(\cdot)$ given fixed parameter $\gamma$.

% theoretical analysis
\vspace{-3mm}
\section{Theoretical Analysis}
\label{sec:theory}
In this section, we analyze the numerical properties of \texttt{DeLAMA} shown in \textbf{Algorithm}~\ref{alg4}. Specifically, we introduce the theoretical properties of the solutions $\mathbf{\Theta}^{(t)}$, $\mathbf{W}^{(t)}$ corresponding to $\mathbf{\Phi}_{\rm local}(\cdot)$, $\mathbf{\Phi}_{\rm param}(\cdot)$ and $\mathbf{\Phi}_{\rm graph}(\cdot)$, including the best point for Taylor expansion, and the convergence properties of the iterations in solving $\mathbf{\Theta}^{(t)}$ and $\mathbf{W}^{(t)}$.

\noindent \textbf{Best point for Taylor expansion.} To minimize the Taylor approximation error, we aim to find the exact expansion point $\boldsymbol{\alpha}_i^{(t)}$ such that the solution from Taylor expansion $\boldsymbol{\theta}_i^{(t)}$ is not far from the exact optimal solution point $\boldsymbol{\theta}_i^{(t)*}$, which corresponds to minimizing the approximation error $|| \boldsymbol{\theta}_i^{(t)} - \boldsymbol{\theta}_i^{(t)*} ||_2$. Due to the difficulty of analyzing the approximation error belonging to general loss functions, our target is optimizing an upper bound of the original error:
\begin{equation}
    \begin{aligned}
        \min_{\boldsymbol{\alpha}_i^{(t)}}  \mathbb{E}_{\boldsymbol{\theta}_i^{(t)*}}\left( \operatorname{\textbf{tr}}\left(\mathbf{H}\left(\boldsymbol{\alpha}_i^{(t)}\right)\right)^{-1} \left\|\boldsymbol{\alpha}_i^{(t)} - \boldsymbol{\theta}_i^{(t)*} \right\|_2^2\right),
    \end{aligned}
\end{equation}
where $\boldsymbol{\theta}_i^{(t)*} \sim \mathcal{N}(\boldsymbol{0}, \mathbf{\Sigma})$. In \textbf{Theorem}~\ref{best point taylor}, we prove that for supervised learning problems with standard supervised training loss (mean square error or cross-entropy) and the linear model $f_{\boldsymbol{\theta}_i^{(t)}}(\cdot)$, we can analytically decide the best point $\boldsymbol{\alpha}_i^{(t)}$ for Taylor expansion.
\begin{theorem}
    \label{best point taylor}
    Let $\mathcal{L}\left(f_{\boldsymbol{\theta}_i^{(t)}}\left(\mathbf{X}_i^{(t)}\right), \mathbf{Y}_i^{(t)}\right)$ be the standard supervised training loss function described by mean square error or cross-entropy of the linear model $f_{\boldsymbol{\theta}_i^{(t)}}(\cdot)$. Suppose $\mathbf{H}\left(\boldsymbol{\alpha}_i^{(t)}\right)=\nabla_{\boldsymbol{\alpha}_i^{(t)}}^2 \mathcal{L}\left(f_{\boldsymbol{\theta}_i^{(t)}}\left(\mathcal{X}_i^{(t)}\right), \mathcal{Y}_i^{(t)}\right)$ is Lipschitz continuous with non-zero constant and $\exists M$ such that $1 \leq k\left(\mathbf{H}\left(\boldsymbol{\alpha}_i^{(t)}\right)\right) \leq M$ where $k$ is the conditional number of $\mathbf{H}\left(\boldsymbol{\alpha}_i^{(t)}\right)$, then
    $$
    \arg \min_{\boldsymbol{\alpha}_i^{(t)}} \mathbb{E}_{\boldsymbol{\theta}_i^{(t)*}}\left( \operatorname{\textbf{tr}}\left(\mathbf{H}\left(\boldsymbol{\alpha}_i^{(t)}\right)\right)^{-1} \left\|\boldsymbol{\alpha}_i^{(t)} - \boldsymbol{\theta}_i^{(t)*} \right\|_2^2\right) = \boldsymbol{0}.
    $$
\end{theorem}
\vspace{-3mm}
\textbf{Theorem}~\ref{best point taylor} shows that the best Taylor expansion point $\boldsymbol{\alpha}_i^{(t)}$ can be directly predetermined as $\boldsymbol{0}$, simplifying the procedure to find the ideal expansion point; see \textbf{Appendix}~\ref{proof-to-best-point-taylor} for detailed analysis and proof.

\noindent \textbf{Convergence of solving $\mathbf{W}^{(t)}$.} To reduce the communication cost and increase the convergence speed of collaborative relational inference $\mathbf{\Phi}_{\rm graph}(\cdot)$ shown in \textbf{Algorithm}~\ref{alg3}, we leverage a continuously differentiable function $\boldsymbol{h}_b(x) = (\sqrt{x^2+b} + x)/2$ to approximate $\operatorname{ReLU}(x)$. In \textbf{Theorem}~\ref{graph learning approx}, we analyze the approximation error between the standard optimization solution and the approximated solution. 
\begin{theorem}
\label{graph learning approx}
    Let $\mathbf{W}^{(t)*}$ be the approximated solution of \textbf{Problem}~\ref{graph_learning_optimization_problem} containing $N$ agents with dual variable $\boldsymbol{z}^*$ satisfying \eqref{equation-z}. Suppose $\boldsymbol{h}_b(\cdot),b>0$ is an approximation function of $\operatorname{ReLU}(\cdot)$. Let the original solution be $\mathbf{W}^{(t)}$ with dual variable $\boldsymbol{z}$. Then
    \begin{align*}
        \left\|\mathbf{W}^{(t)} - \mathbf{W}^{(t)*}\right\|_{\mathbf{F}} \leq \frac{N \sqrt{b}}{2} \sqrt{1 + \frac{1}{C^2}},
    \end{align*}
    where $C$ is a non-negative lower bound satisfies $\boldsymbol{h}_b^\prime(x)\geq C$ for $x \in [\boldsymbol{z}, \boldsymbol{z}^*]$.
\end{theorem}
\vspace{-2mm}
\textbf{Theorem}~\ref{graph learning approx} shows that the solution of \textbf{Algorithm}~\ref{alg3} can approximate the optimum as $\boldsymbol{b}$ goes to zero, demonstrating the effectiveness of our approach; see \textbf{Appendix}~\ref{proof-to-graph-learning-approx} for the detailed proof.

\noindent \textbf{Convergence properties of solving $\mathbf{\Theta}^{(t)}$.} Recall that in \textbf{Problem}~\ref{colla_param_learning}, optimizing the objective function corresponds to solving a linear system defined by the formula~\ref{foc}. We propose an iterative method based on Jacobi iterations~\cite{matrix-computations} to solve the large-scale sparse linear equations. We have the following theoretical guarantee about the proposed method.
\begin{theorem}
\label{convergence}
    Let $\mathbf{\Theta}^{(t),k} = \left(\boldsymbol{\theta}_1^{(t),k}, \ldots, \boldsymbol{\theta}_N^{(t),k}\right)$ be the model parameters in the $k$-th iteration of \textbf{Algorithm}~\ref{alg3}, $\mathbf{\Theta}^{(t)*} = \left(\boldsymbol{\theta}_1^{*(t)}, \ldots, \boldsymbol{\theta}_N^{*(t)}\right)$ is the target solution satisfying \eqref{foc}. For any $\boldsymbol{\epsilon}$, if $\left\| \mathbf{\Theta}^{(t)} - \mathbf{\Theta}^{(t)*}\right\|_{\mathbf{F}}^2 \leq \boldsymbol{\epsilon}$, then $k \sim \mathcal{O}\left[\log \left(\frac{1}{\boldsymbol{\epsilon}}\right)\right]$. 
\end{theorem}
\vspace{-1mm}
% {\HC \textbf{Theorem}~\ref{convergence} reflects that the iterations of \textbf{Algorithm}~\ref{alg3} do converge and can converge to an almost optimal solution within a few iterations, which is computationally efficient.} 
\textbf{Theorem}~\ref{convergence} shows that model parameters obtained through \textbf{Algorithm}~\ref{alg3} can converge to the optimum within a few iterations, demonstrating the computational efficiency of our approach; see \textbf{Appendix}~\ref{proof-to-convergence} for detailed proof.

In summary, we prove that the solution provided by \texttt{DeLAMA} could converge to the optimum of \textbf{Problem}~\eqref{colla_relation_learning} and \textbf{Problem}~\eqref{colla_param_learning}.

% experiments
\vspace{-3mm}
\section{Experiments}
\vspace{-2mm}
\label{sec:experiments}
We evaluate \texttt{DeLAMA} from four aspects, in \textbf{Section}~\ref{regression}, we first apply our method to a regression problem; in \textbf{Section}~\ref{image}, we delve into image classification tasks with more complex data forms; in \textbf{Section}~\ref{map}, we apply \texttt{DeLAMA} to multi-robot mapping tasks, representing real scenarios' performance; in \textbf{Section}~\ref{human}, we compare the performance of \texttt{DeLAMA} with humans. In each aspect, performance is evaluated both quantitatively and qualitatively. We also conduct ablation studies on \texttt{DeLAMA} to highlight the significance of each module in the collaborative learning mechanism.
\vspace{-3mm}
\subsection{Regression Problem}
\label{regression}
\vspace{-1mm}
\subsubsection{Experimental Setup}
% We conduct a regression analysis task to {\HC model relationships of a bunch of scatter points }, drawing inspiration from cognitive scientific principles detailed in~\cite{almaatouq2020adaptive}. These principles examine adaptive human decision-making processes. Our study extends this study to machine collaboration, employing a similar analytical framework to elucidate patterns of cooperation among agents at a cognitive level, thus shifting the focus from human-human to agent-agent collaboration. 
\noindent \textbf{Task.} Following the cognitive scientific principles detailed in~\cite{almaatouq2020adaptive}, we design a system with 6 agents in this collaborative regression setting, where the collaboration relationship is easy to verify.
% {\HC Here we choose the regression task for each agent's model learning, as it shares similar cognitive scientific principles detailed in~\cite{almaatouq2020adaptive}.} 
We design several non-linear functions to generate data points and each agent accesses a small dataset at each time drawn from one of the non-linear functions. One function can be accessed by multiple agents. Each agent is required to regress the whole corresponding function. Since the data points of one single agent are not informative enough, correctly collaborating with other agents corresponding to the same function is essential. Moreover, to increase the difficulty of the collaborative relationship inferring, agents are categorized into three types based on data quality: i) type one, agents possess numerous samples with minimal noise; ii) type two, agents have a limited sampling range with medium noise; and iii) type three,
agents access very few samples with significant noise. This agent categorization requires the agent to consider both the function correspondence and data quality during collaboration.  

\noindent \textbf{Dataset.} The total function number ranges randomly from 1 to 3 and each function is ensured to correspond to at least two agents. The data points are obtained by sampling points from a function $f(x)$ by $\hat{y} = f(x) + \epsilon$, with a Gaussian noise $\epsilon \sim \mathcal{N}(0, \sigma)$. The functions are random linear combinations of quadratic functions and sinusoidal functions, with domain [-5, 5]. At each time, different types of agents' training sets have different sample ranges: type one agents' training set is generated from the full domain [-5, 5], while type two and type three agents' training sets are generated from a random limited interval with length 1.  We generated 600 task learning sequences with different ground truth functions, with 2/3 used for training, and 1/3 used for testing.

\noindent \textbf{Evaluation metric.}  We adopt two metrics for evaluation: i) the average mean square error between the agent regression points and the ground-truth function's points at each time stamp $t$; ii) the graph mean square error, which is defined as the normalized distance between the learned collaboration graph and the oracle collaboration graph at time $t$: 
\begin{equation}
    \label{gmse}
    \begin{aligned}
        \text{GMSE} = \frac{\left\|\mathbf{W}^{(t)} - \mathbf{W}_{\rm oracle} \right\|_{\mathbf{F}}}{\left\|\mathbf{W}_{\rm oracle} \right\|_{\mathbf{F}}}.
    \end{aligned}
\end{equation}
The oracle graph adjacency matrix $\mathbf{W}_{\rm oracle}$ is defined as $\mathbf{W}_{\rm oracle} = \boldsymbol{m} \frac{\widetilde{\mathbf{W}}}{\|\widetilde{\mathbf{W}}\|_1}$, with $\widetilde{\mathbf{W}}_{ij} = 1$ if agent $i$ and $j$ share a same function and $\widetilde{\mathbf{W}}_{ij} = 0$ otherwise. Note that the oracle graph represents a relatively optimal solution for collaboration since it acquires the agent-function correspondence but may not be the absolute optimal solution since the agent data quality is not considered.

\begin{table}
    \centering
    \caption{\small Performance comparison for the regression problem. We compare our method with representative federated learning approaches and decentralized optimization methods. We also perform an ablation study of our approach, including \texttt{DeLAMA-WC} and \texttt{DeLAMA-WM}. The former stands for \texttt{DeLAMA} without collaboration among agents while the latter stands for \texttt{DeLAMA} without lifelong-adaptive learning capabilities.}
    \vspace{-3mm}
    \begin{tabular}{c|c|cc}
         \toprule
         Setting &Method & $\text{MSE}_{t=1}$  & $\text{MSE}_{t=10}$\\
         \midrule
         \multirow{8}{*}{\makecell[c]{\textbf{Federated}\\ \textbf{learning}}} &FedAvg\cite{fedavg} &12.9147 &10.3229  \\
         &FedProx\cite{fedprox} &11.8117 &10.5989  \\
         &SCAFFOLD\cite{scaffold} &13.6683 &12.8070  \\
         &FedAvg-FT~\cite{fedavg} &9.8845 &6.7536  \\
         &FedProx-FT~\cite{fedprox} &9.4131 &6.6035\\
         &Ditto\cite{ditto} &16.3461 &7.9752  \\
         &FedRep\cite{fedrep} &40.0225 &11.4844  \\
         &pFedGraph\cite{pfedgraph} &8.6481 &5.9771  \\
         \midrule 
         \multirow{3}{*}{\makecell[c]{\textbf{Decentralized}\\ \textbf{optimization}}}
         &DSGD~\cite{DSGD} &23.8660 &12.6872\\
         &DSGT~\cite{DSGT} &17.7407 &17.7413\\
         % &CHOCO-SGD~\cite{CHOCO-SGD} & &\\
         &DFedAvgM~\cite{DFedAvgM} &13.2054 &10.4894\\
         \midrule
         \multirow{3}{*}{\makecell[c]{\textbf{Decentralized and}\\ \textbf{lifelong-adaptive}\\ \textbf{collaborative learning}}} 
         &\texttt{DeLAMA-WC} &5.0794 &0.0785 \\
         &\texttt{DeLAMA-WM}&\textbf{2.3377} &2.6150 \\
         &\texttt{DeLAMA} &\textbf{2.3377}  &\textbf{0.0719}\\
         \bottomrule
    \end{tabular}
    \vspace{-3mm}
    \label{tab:regression-decentralization}
\end{table}

\begin{table*}
    \centering
    \caption{\small The evaluation of the collaboration performance compared with other centralized collaboration mechanisms. \texttt{DeLAMA} effectively learns collaboration relationships under different communication structure constraints.}
    \vspace{-3mm}
    \renewcommand\arraystretch{0.8}
    \begin{tabular}{c|c|c|c|cccccc}
    \toprule
         $\mathcal{G}^{(t)}$ &$c(\mathcal{G}^{(t)})$ &Method &Decentralize &$\text{GMSE}_{t=1}$ &$\text{GMSE}_{t=5}$ &$\text{GMSE}_{t=10}$ & $\text{MSE}_{t=1}$ & $\text{MSE}_{t=5}$ & $\text{MSE}_{t=10}$ \\
         \midrule
         \multirow{7}{*}{\textbf{FC}} &\multirow{6}{*}{1} &NS~\cite{neighborhood-selection} &\usym{2713} &1.4198 &1.4900 &1.4396 & 56.1900 &0.8384 &0.3919 \\
         & &GLasso~\cite{graphlasso} &\usym{2717} &0.9687 &0.9449 &0.9540 &3.9664 &0.1923 &0.0817\\
         & &MTRL~\cite{liu2017distributed} &\usym{2717} &72.26 &3.955 &60.3069 &2.6570 &0.3009 &0.1837\\
         % & &GL-SigRep\cite{graph-learning-smooth} &\usym{2717} &0.4325 &\textbf{0.0726} &\textbf{0.0391} &\textbf{2.2757} &\textbf{0.1111} &\textbf{0.0704} \\
         & &GL-LogDet~\cite{graph-learning-smooth} &\usym{2717} &1.0391 &0.8916 &0.8331 &3.2395 &0.1416 &0.0726\\
         & &L2G-PDS~\cite{L2G} &\usym{2717} &3.6415 &1.2675 &0.6638 &4.4638 &0.1633 &0.0744\\
         % & &L2G-ADMM\cite{L2G} &\usym{2717} &2.9854 &1.1369 &0.6491 &3.6178 &0.1190 &0.0734\\
         & &Oracle &- &- &- &- &1.1611 &0.0847 &0.0695\\         
         & &\texttt{DeLAMA} &\usym{2713} &\textbf{0.2992} &\textbf{0.0930} &\textbf{0.0721} &\textbf{2.3377} &\textbf{0.1104} &\textbf{0.0719}\\
         \midrule
         
         \multirow{2}{*}{\textbf{ER}} &0.3 &\multirow{4}{*}{\texttt{DeLAMA}} &\multirow{4}{*}{\usym{2713}} &1.8291 &1.7683 &1.7712 &3.0061 &0.2176 &0.0944\\
         &0.5 & & &1.1627 &0.9794 &0.9736 &2.8659 &0.1654 &0.0850\\
         \multirow{2}{*}{\textbf{BA}} &0.3 & & &1.8244 &1.7697 &1.7728 &2.8893 &0.2391 &0.0905\\
         &0.5 & & &0.9755 &0.8475 &0.8371 &2.7184 &0.1781 &0.0787\\
    \bottomrule
    \end{tabular}
    \vspace{-3mm}
    \label{tab:regression_graph}
\end{table*}
\begin{figure*}

    \centering
    \captionsetup[subfloat]{labelsep=none,format=plain,labelformat=empty}
    \subfloat[\texttt{GL-LogDet}]{\includegraphics[width=0.14\linewidth]{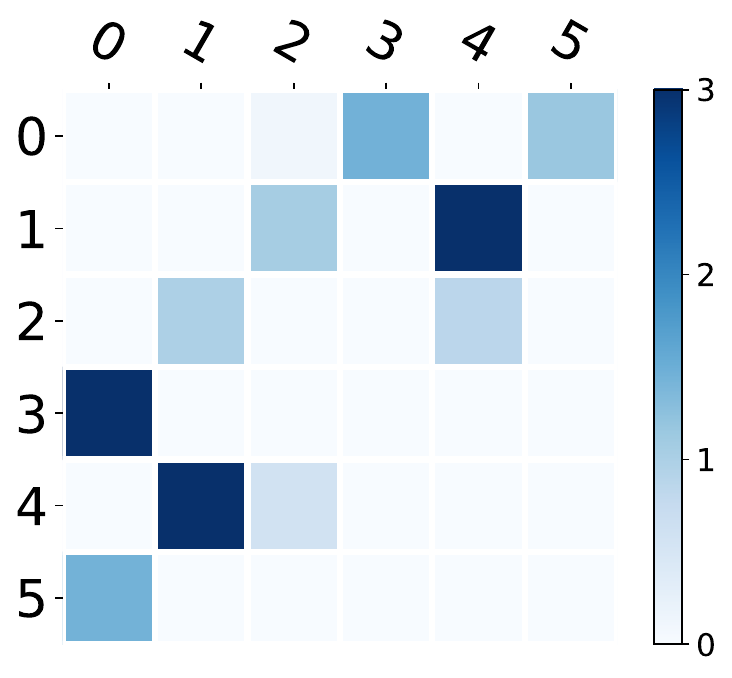}}
    \hfill
    \subfloat[\texttt{GLasso}]{\includegraphics[width=0.14\linewidth]{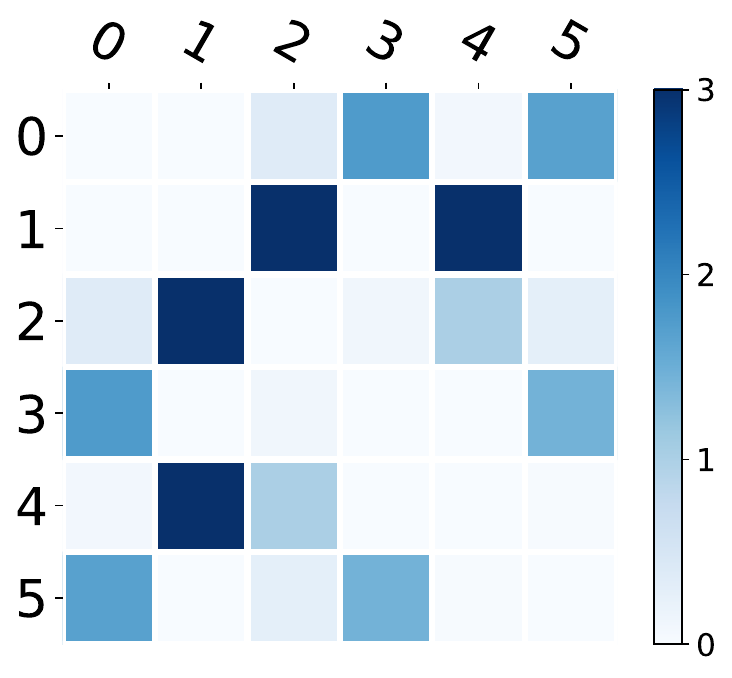}}
    \hfill
    \subfloat[\texttt{MTRL}]{\includegraphics[width=0.14\linewidth]{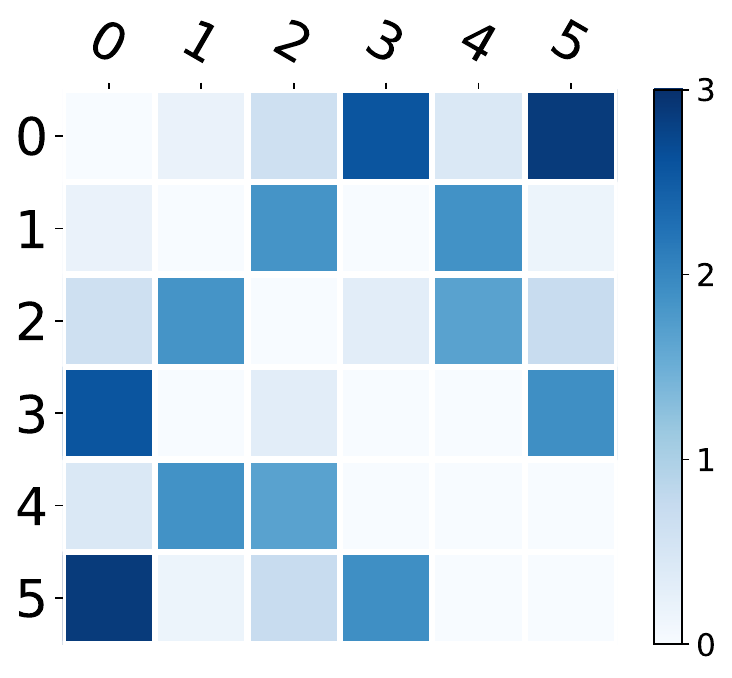}}  
    \hfill
    \subfloat[\texttt{L2G-PDS}]{\includegraphics[width=0.14\linewidth]{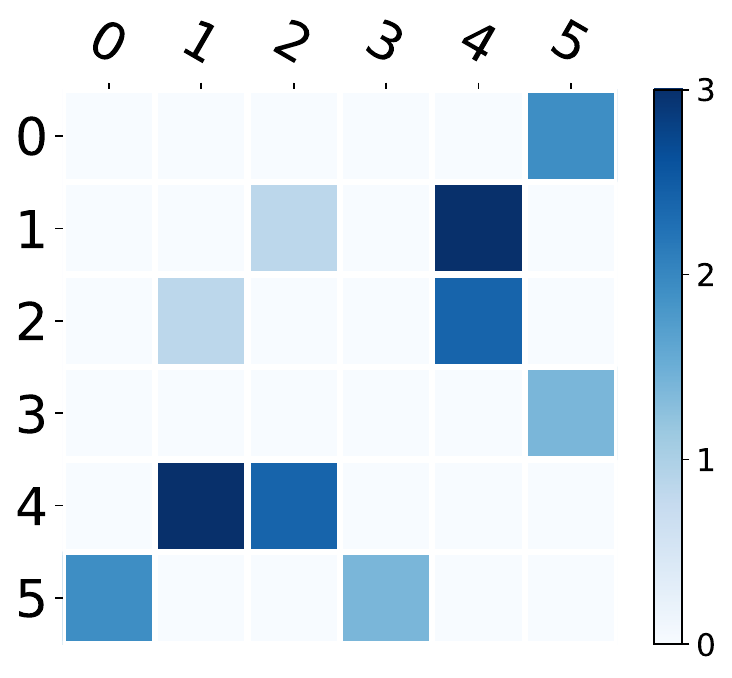}}
    \hfill
    \subfloat[\texttt{NS}]{\includegraphics[width=0.14\linewidth]{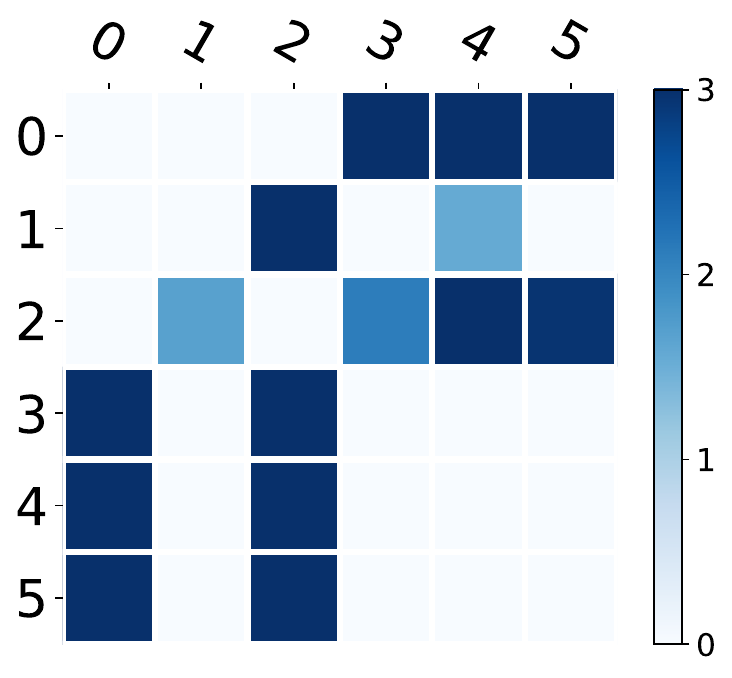}}
    \hfill
    \subfloat[\texttt{DeLAMA} (\textbf{Ours})]{\includegraphics[width=0.14\linewidth]{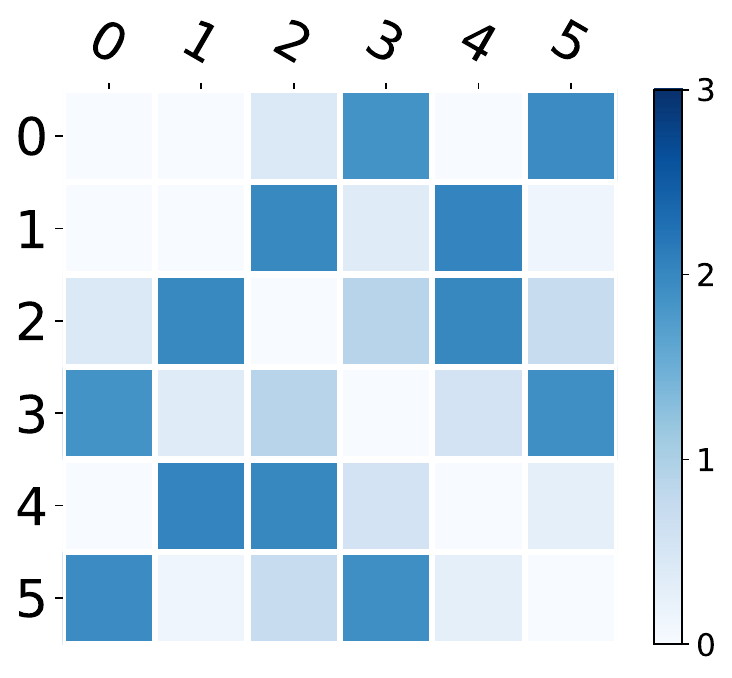}}
    \hfill
    \subfloat[\texttt{Oracle}]{\includegraphics[width=0.14\linewidth]{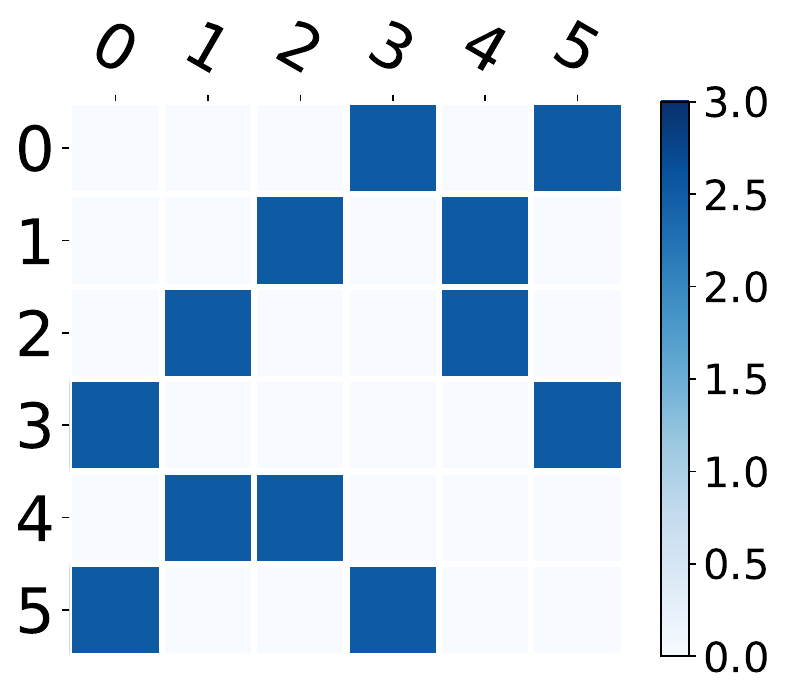}}
    \vspace{-3mm}
    \caption{\small Visualization of the learned graph structure at time $t=3$ compared with the oracle collaboration structure. \texttt{DeLAMA} learns smooth and accurate edge weights compared to baseline methods.}
    \vspace{-3mm}
    \label{fig:regression_graph_compare}
\end{figure*}
\vspace{-4mm}
\subsubsection{Quantitative Analysis}
\noindent Here we aim to: i) verify the effectiveness of the decentralized collaboration mechanism in \texttt{DeLAMA} on regression learning performance; ii) evaluate the collaboration graph learning performance of \texttt{DeLAMA} under different constraints of the communication graph $\mathbf{C}^{(t)}$. Note that in i) the communication structure of \texttt{DeLAMA} is set to be fully connected, while in ii) we will try other kinds of communication structures.

To verify i), \textbf{Table}~\ref{tab:regression-decentralization} presents the system performance comparison with federated learning and decentralized optimization methods, including standard federated learning~\cite{fedavg, fedprox, scaffold}, personalized federated learning~\cite{ditto, pfedgraph} and \cite{DSGD, DSGT, DFedAvgM}. The experimental result shows that i) our decentralized and lifelong collaboration approach significantly outperforms both federated learning and decentralized optimization methods, reducing the MSE by \textbf{98.80}\% at last timestamp $t=10$; ii) adding our agent collaboration remarkably improves the learning performance, especially when timestamp is small that every agent lacks data; iii) as timestamp increases, the learning performance improvement brought by our lifelong-adaptive learning enlarges, reflecting the effectiveness of our lifelong-adaptive learning design.

To verify ii), we perform \texttt{DeLAMA} under different communication structures, including the fully connected(FC) graph, the Erdos-Renyi(ER) graph, and the Barabasi Albert(BA) graph~\cite{network_sci}. The compared approaches contain both classic relation learning methods~\cite{graphlasso, liu2017distributed}, and recent widely used structural learning approaches L2G~\cite{L2G}. We compare these approaches by substituting $\mathbf{\Phi}_{\rm graph}(\cdot)$ of \texttt{DeLAMA} to other relational learning methods while preserving other parts of \texttt{DeLAMA} stay unchanged.
To achieve a fair comparison, we investigate the system performance under the same connectivity level of these two types of random graphs. Here the connectivity level $c\left(\mathcal{G}^{(t)}\right)$ of the communication graph adjacency matrix $\mathbf{C}^{(t)} \in \{0, 1\}^{N \times N}$ is defined as: $c\left(\mathcal{G}^{(t)}\right) = \|\mathbf{C}^{(t)}\|_1/N^2$. \textbf{Table}~\ref{tab:regression_graph} presents the comparison with previous relational learning methods on both the structural correctness and the task performance. The results show that i) \texttt{DeLAMA} outperforms both traditional centralized and decentralized structural learning methods;  ii) \texttt{DeLAMA} matches the performance of recent structural learning performance L2G~\cite{L2G}; and iii) \texttt{DeLAMA} stills works even under low connectivity level.

\vspace{-3mm}
\subsubsection{Qualitative Analysis}
\begin{figure}
    \centering
    \subfloat[Functions learned by each agent]{
            \vspace{-2mm}
\includegraphics[width=0.95\linewidth]{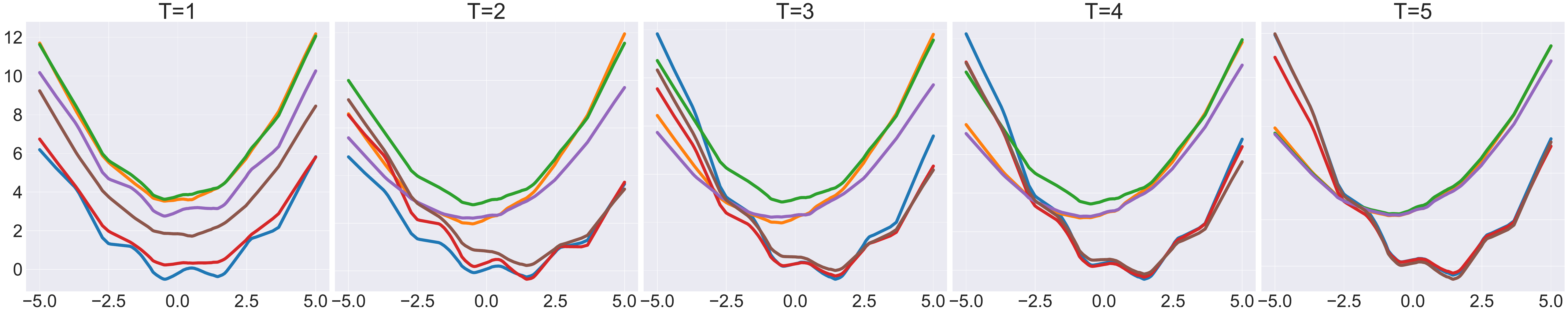}
    }\hfill
    \subfloat[Time evolving collaboration structure]{
            \vspace{-2mm}
\includegraphics[width=0.95\linewidth]{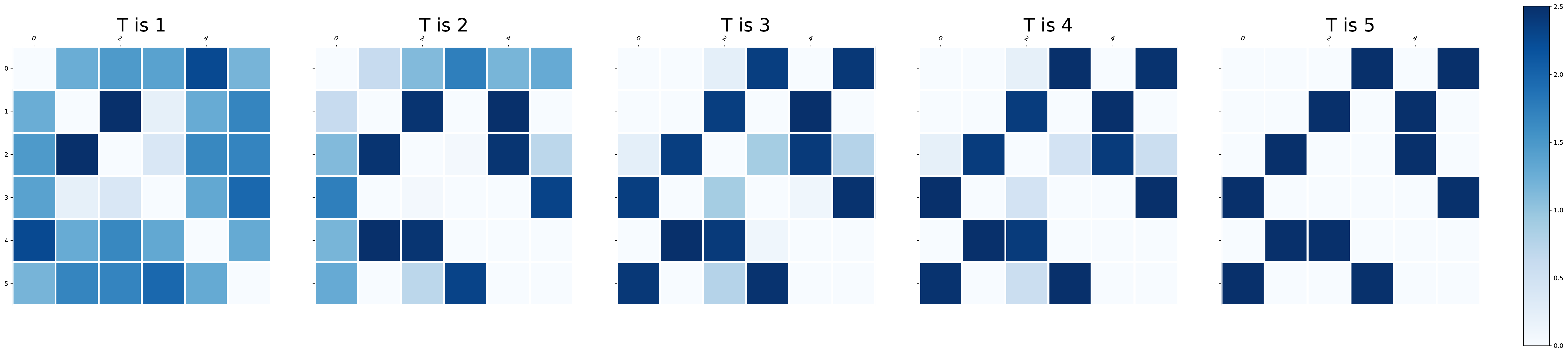}
    }
    \vspace{-3mm}
    \caption{\small A collaboration result of two groups of agents that are performing two different regression tasks. (a) The learned functions after collaboration for each agent. (b) The time-evolving collaboration graph of this collaboration system. \texttt{DeLAMA} enables the collaboration system to evolve with dynamic collaboration relationships.}
    \vspace{-3mm}
    \label{fig:regression_res}
\end{figure}

Here we aim to i) verify that the system of \texttt{DeLAMA} could dynamically evolve as time increases, and ii) compare the performances of our collaboration relationship inference with existing graph learning approaches, including NS~\cite{neighborhood-selection}, GLasso~\cite{graphlasso}, MTRL~\cite{liu2017distributed}, GL-LogDet~\cite{graph-learning-smooth} and L2G~\cite{L2G}.

To show i), we visualize the collaborative learning results of each agent from time stamp 1 to 5, and the corresponding collaborative relationships. \textbf{Figure}~\ref{fig:regression_res} illustrates the collaboration learning results of two groups among six agents. The visualizations show that agents can find correct collaborators where the connection forms two groups of agents, and refine the time-evolving collaboration structures as time increases. The visualization of their learned functions shows that they can learn from past encountered data samples.

To show ii), we visualize the system collaboration relationship learned by different approaches. To achieve a fair comparison, we calculate the collaboration relationship for the group of agents at the same time $t=3$. The visualization of collaboration relationships shown in \textbf{Figure}~\ref{fig:regression_graph_compare} demonstrates that \texttt{DeLAMA} could discover more suitable collaboration relations compared to other approaches.

\begin{table*}
    \caption{\small The system performance analysis on classification tasks of four different settings: Lifelong learning, federated learning, decentralized optimization, decentralized and lifelong-adaptive collaborative learning. \texttt{DeLAMA-WC} stands for \texttt{DeLAMA} without collaboration among agents and \texttt{DeLAMA-WM} stands for \texttt{DeLAMA} without lifelong-adaptive learning capabilities.}
    \vspace{-3mm}
    \label{tab:classification-decentralization}
    \centering
    \renewcommand\arraystretch{0.8}
    \begin{tabular}{c|c|cccccc}
    \toprule
         \multirowcell{2}{\textbf{Setting}} &\multirowcell{2}{\textbf{Method}}  &\multicolumn{3}{c}{\textbf{MNIST}} &\multicolumn{3}{c}{\textbf{CIFAR-10}} \\
         & & $\text{Acc}_{t=1}$ & $\text{Acc}_{t=5}$ & $\text{Acc}_{t=10}$ & $\text{Acc}_{t=1}$ & $\text{Acc}_{t=5}$ & $\text{Acc}_{t=10}$\\
    
         \midrule
         
         \multirow{5}{*}{\makecell[c]{\textbf{Lifelong} \\ \textbf{learning}}} 
         & LWF\cite{lwf}  &20.17 &26.49 &32.34 &20.00 &20.00 &20.00\\ 
         & EWC\cite{ewc}  &20.10 &26.64 &31.47 &20.00 &20.00 &20.01\\ 
         & MAS\cite{MAS} &20.10 &26.64 &31.47 &20.00 &20.00 &20.00\\
         & GEM\cite{gem} &20.00 &25.41 &49.84 &20.00 &25.94 &32.49\\
         & A-GEM\cite{A-gem} &20.00 &20.94 &22.47 &20.00 &20.38 &20.07\\
         \midrule
         \multirow{8}{*}{\makecell[c]{\textbf{Federated}\\ \textbf{learning}}} & FedAvg~\cite{fedavg} & 43.09 &51.76 &59.50 &19.94 &20.84 &23.88\\
         &FedProx~\cite{fedprox} &49.12 &56.17 &60.28 &21.70 &23.74 &26.32\\
         &SCAFFOLD~\cite{scaffold} &48.04 &56.14 &64.44 &21.09 &23.78 &26.18\\
         &FedAvg-FT~\cite{fedavg} &27.36 &50.27 &59.10 &20.00 &20.34 &23.02\\
         &FedProx-FT~\cite{fedprox} &36.07 &53.46 &58.87 &20.00 &20.33 &22.51\\
         &Ditto~\cite{ditto} &37.80 &51.63 &59.41 &20.04 &20.29 &22.32\\
         &FedRep~\cite{fedrep} &20.00 &20.57 &23.23 &20.00 &20.00 &20.00\\
         &pFedGraph\cite{pfedgraph} &20.00 &26.09 &34.79 &20.00 &20.93 &22.23\\
         \midrule
         \multirow{3}{*}{\makecell[c]{\textbf{Decentralized}\\ \textbf{optimization}}}
         &DSGD~\cite{DSGD} &20.10 &20.52 &22.61 &20.02 &20.02 &20.00\\
         &DSGT~\cite{DSGT} &22.42 &21.34 &22.81 &19.94 &20.02 &20.05\\
         % &CHOCO-SGD~\cite{CHOCO-SGD} & & & & & &\\
         &DFedAvgM~\cite{DFedAvgM} &36.91 &51.37 &54.91 &19.97 &21.30 &22.38\\
         \midrule
         \multirow{3}{*}{\makecell[c]{\textbf{Decentralized and lifelong-}\\ \textbf{adaptive collaborative learning}}} &\texttt{DeLAMA-WC} &37.15 &78.92 &95.54 &30.96 &57.81 &71.46\\
         &\texttt{DeLAMA-WM} &\textbf{67.67} &66.67 &67.10 &46.87 &46.07 &45.32\\
         &\texttt{DeLAMA}  &\textbf{67.67} &\textbf{98.03} &\textbf{99.51} &\textbf{46.87} &\textbf{72.57} &\textbf{76.03}\\ 
    \bottomrule
    \end{tabular}
    \vspace{-2mm}
\end{table*}

\begin{table*}
    \centering
    \caption{\small The system collaborative learning performance of different collaboration relationship learning methods. Here {Oracle} means we use graph structure where agents with the same observation configuration collaborate together. The evaluation metric contains both the graph-level GMSE metric against the Oracle structure and the system performance on the evaluation set. }
    \vspace{-3mm}
        \renewcommand\arraystretch{0.8}
    \begin{tabular}{c|c|c|cccccc}
         \toprule
         \textbf{Data} &\textbf{Method} &\textbf{Decentralize} & $\text{GMSE}_{t=1}$ & $\text{GMSE}_{t=3}$ & $\text{GMSE}_{t=5}$  & $\text{Acc}_{t=1}$ & $\text{Acc}_{t=3}$ & $\text{Acc}_{t=5}$ \\
         \midrule
         \multirow{7}{*}{\textbf{MNIST}} & NS\cite{neighborhood-selection} &\usym{2713} &0.9999 &2.3562 &1.3772 &52.83 &83.13 &95.81\\
         
         &GLasso \cite{graphlasso} &\usym{2717} &1.7585 &0.2728 &0.1365 &38.52 &81.89 &95.88\\

         % &GLasso $\alpha=1e-4$\cite{graphlasso} &\usym{2717} &9.696 &2.592 &1.780 &35.42 &77.96 &95.46\\

         &MTRL\cite{liu2017distributed} &\usym{2717} &2.1181 &0.3179 &0.1444 &39.36 &79.61 &97.28\\

         % &GL-SigRep\cite{graph-learning-smooth} &\usym{2717} &0.4897 &0.1966 &0.1009 &68.49 &92.52 &98.12\\

         &GL-LogDet\cite{graph-learning-smooth} &\usym{2717} &7.0086 &0.9999 &1.0000 &44.93 &65.43 &79.93\\

         &L2G-PDS\cite{L2G} &\usym{2717} &0.5211 &0.2970 &0.1391 &68.32 &91.96 &97.30\\

         % &L2G-ADMM\cite{L2G} &\usym{2717} &0.5093 &0.2200 &0.0811 &68.57 &92.03 &97.77\\

         &Oracle &- & - & - & - & 69.28 & 92.83 & 98.22\\

         &\texttt{DeLAMA} &\usym{2713} &\textbf{0.4998} &\textbf{0.1968} &\textbf{0.0855} &\textbf{67.67} &\textbf{92.02} &\textbf{98.03} \\

         \midrule
         
         \multirow{7}{*}{\textbf{CIFAR-10}} & NS\cite{neighborhood-selection} &\usym{2713} &1.9756 &4.0367 &2.321 &38.36 &54.28 &66.19\\
         
         &GLasso \cite{graphlasso} &\usym{2717} &14.3883 &1.3897 &0.3823 &24.41 &55.39 &69.74\\

         % &GLasso $\alpha=1e-4$\cite{graphlasso} &\usym{2717} &29.11 &1.705 &1.085 &25.81 &67.33 &76.02\\

         &MTRL\cite{liu2017distributed} &\usym{2717} &375.2 &2.0157 &0.4352 &25.36 &53.23 &64.88\\

         % &GL-SigRep\cite{graph-learning-smooth} &\usym{2717} & & & &46.72 &65.91 &73.24\\

         &GL-LogDet\cite{graph-learning-smooth} &\usym{2717} &1.0000 &1.0033 &1.0109 &35.52 &48.51 &58.23\\

         &L2G-PDS\cite{L2G} &\usym{2717} &0.5067 &0.3226 &0.2065 &45.63 &65.03 &72.06\\

         % &L2G-ADMM\cite{L2G} &\usym{2717} &0.7353 &0.5434 &0.5072 &46.50 &65.44 &73.01\\

         &Oracle & - &- &- &- &48.21 &68.51 &74.16\\

         &\texttt{DeLAMA} &\usym{2713} &\textbf{0.5948} &\textbf{0.3089} &\textbf{0.1475} &\textbf{46.87} &\textbf{65.19} &\textbf{72.57}\\

         \bottomrule
    \end{tabular}
    \vspace{-3mm}
    \label{tab:classification_graph}
\end{table*}

\vspace{-3mm}
\subsection{Image Classification}
\label{image}
\vspace{-1mm}
\subsubsection{Experimental Setup}
\textbf{Task.} We consider a collaboration system with 6 collaborators. To encourage the collaboration behaviors between the agents, we divide the system of agents into two groups where each group's agents are all doing one 5-class classification task. We adopt the class incremental learning paradigm into our experiment, where each time agents randomly access purely one single class of data drawn from the five classes.

\noindent \textbf{Dataset.} Our experiments are performed on both the \textbf{MNIST} dataset\cite{mnist} and the \textbf{CIFAR-10} dataset\cite{cifar10}. We create the 5-class classification tasks by sampling images from the full \textbf{MNIST} or \textbf{CIFAR-10} dataset. The 5 classes are randomly chosen from 0 to 9, each time agents will access one single class data sample. Each time the sampling size of the train set is 10 and the corresponding test set size with all five classes is 50. To enable algorithm unrolling, we created several lifelong collaboration learning tasks with 70\% for training and 30\% for testing. Here $N=6$ with maximal learning time $T=10$.

\noindent \textbf{Metric.} We adopt two metrics for evaluation: 1) the average prediction accuracy for each agent measured on the test set. 2) the graph mean square error shown in ~\eqref{gmse}.

\vspace{-3mm}
\subsubsection{Quantitative Analysis}
\noindent \textbf{Decentralized collaboration} Here we aim to: i) verify the effectiveness of the decentralized and lifelong collaboration mechanism in \texttt{DeLAMA} on classification task learning performance; ii) evaluate the collaboration graph learning performance of \texttt{DeLAMA} on classification tasks.

% Here we aim to: i) verify that the decentralized collaboration of \texttt{DeLAMA} outperforms pure single agent learning in classification tasks; ii) verify that the performance of the decentralized collaboration mechanism could match even outperform the centralized collaborative learning approaches on classification tasks. \Note{CX:graph learning ability?}

To verify i), \textbf{Table}~\ref{tab:classification-decentralization} presents the system comparison with lifelong learning~\cite{lwf, ewc, A-gem}, federated learning~\cite{fedavg, fedprox} and decentralized learning methods~\cite{DSGD, DSGT, fedavgm}. The experimental result shows that i) our decentralized and lifelong collaboration approach could effectively remember previously learned training data, significantly outperforms classic lifelong learning approaches by \textbf{99.66\%} and \textbf{134.01\%} on MNIST and CIFAR-10 at $t=10$, showing the effectiveness of our lifelong learning mechanism; ii) compared with other multi-client and decentralized training approaches~\cite{fedavg, pfedgraph, fedavgm} that uses a static collaboration graph, our decentralized collaboration significantly improves the classification accuracy by \textbf{54.42\%} and \textbf{188.87\%} on MNIST and CIFAR-10 at $t=10$, showing the effectiveness of our dynamic learning of collaboration graph; iii) adding our agent collaboration and 
remarkably improves the learning performance, especially
when all the agents lack data at the beginning.

To verify ii), \textbf{Table}~\ref{tab:classification_graph} presents the comparison with other classic relationship learning approaches~\cite{neighborhood-selection, graphlasso, L2G}, including centralized methods~\cite{graph-learning-smooth, liu2017distributed} and decentralized approaches~\cite{neighborhood-selection}. We compare these approaches by substituting $\mathbf{\Phi}_{\rm graph}(\cdot)$ of \texttt{DeLAMA} to these relational learning methods while preserving other parts of \texttt{DeLAMA} stay unchanged. The results show that i) \texttt{DeLAMA} outperforms classic relationship learning algorithms in classification tasks with a significant reduction in the graph structure inference error of \textbf{38.53\%} and \textbf{28.57\%}, reflecting that our method accurately learns the collaboration graph; ii) the accurate learning of collaboration graph further leads to the improvement of task performance (classification). Our method outperforms classic collaboration relation learning methods and is very close to the oracle method that utilizes the ground-truth collaboration graph; iii) the collaboration relationships learned by \texttt{DeLAMA} could dynamically improve with time increases from the system performance perspective.

\begin{figure}
    \centering
    \subfloat[Collaboration relationships of MNIST]{
        \vspace{-4mm}
        \includegraphics[width=0.95\linewidth]{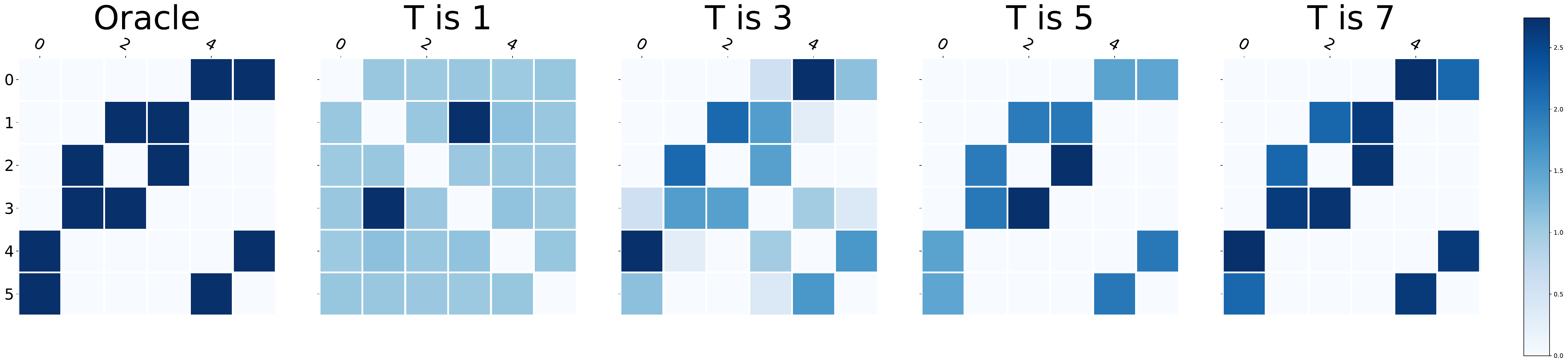}
    }\hfill
    \subfloat[Collaboration relationships of CIFAR-10]{
                \vspace{-4mm}
\includegraphics[width=0.95\linewidth]{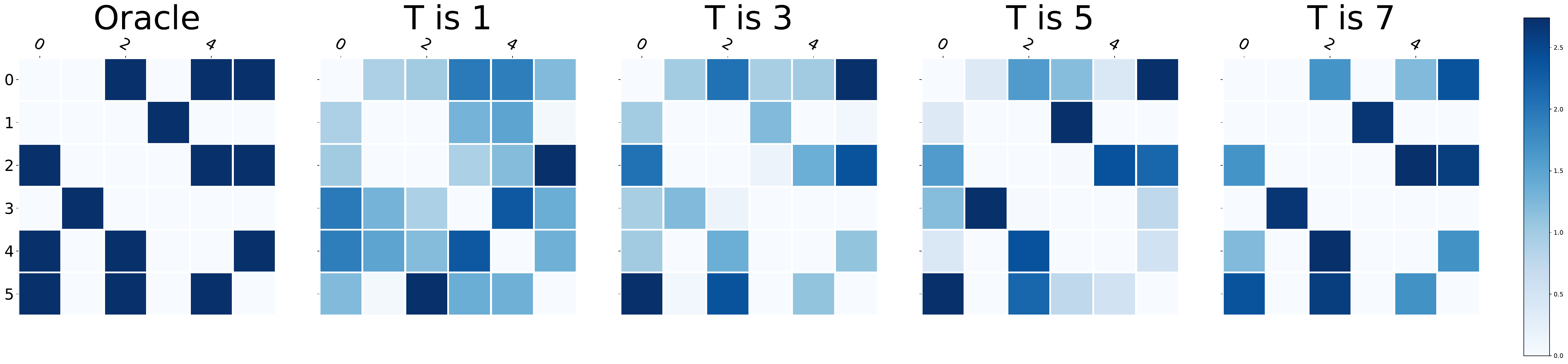}
    }
    \vspace{-3mm}
    \caption{\small Example of the evolving collaboration graphs. The system tends to assign a relatively weak collaboration relationship between all the agents, then cut off the undesired ones, and finally, agents can correctly find useful collaborators.}
    \vspace{-3mm}
    \label{fig:classification graphs}
\end{figure}

% To verify i), we compare \texttt{DeLAMA} with single agent learning methods, each powered by lifelong learning approaches. The results of \textbf{Table}~\ref{tab:classification-decentralization} show that \texttt{DeLAMA} outperforms all these single-agent learning approaches without collaboration, which reveals the effectiveness of the decentralized collaboration.

% To verify ii), we compare \texttt{DeLAMA} with other centralized task learning approaches from two perspectives, including model update and collaboration relational inference. For model update, we compare \texttt{DeLAMA} with multi-task learning \Note{CX:multi-task learning?} approaches, where both methods are powered by a central server. For collaboration relational inference, we substitute $\Phi_{\rm graph}(\cdot)$ of \texttt{DeLAMA} into other structure learning approaches~\cite{L2G, graphlasso}, where many of these methods are based on centralized optimization. The performance of \textbf{Table}~\ref{tab:classification_graph} and \textbf{Table}~\ref{tab:classification-decentralization} shows that the decentralized collaboration mechanism could outperform the centralized collaboration mechanisms from the model update perspective compared to federated learning, and match the performance of centralized ones compared with other centralized structure learning approaches as ii) claimed.

\noindent \textbf{Lifelong adaptation} Here we aim to: i) verify the lifelong learning capability from the single-agent learning perspective; ii) verify that \texttt{DeLAMA} is lifelong-adaptive compared to the static collaboration approaches. 

To verify i), we compare \texttt{DeLAMA-WC} with several lifelong learning approaches with the concern of fair comparison. \textbf{Table}~\ref{tab:classification-decentralization} shows that even without collaboration, \texttt{DeLAMA-WC} outperforms standard lifelong learning approaches.

To verify ii), we compare the system performance of \texttt{DeLAMA} with other static collaboration approaches such as federated learning~\cite{fedavg, fedprox} and collaboration relationship learning~\cite{neighborhood-selection, graphlasso, L2G}. The performance compared with federated learning and multitask learning shown in \textbf{Table}~\ref{tab:classification-decentralization} and \textbf{Table}~\ref{tab:classification_graph} reveals that i) \texttt{DeLAMA} has lifelong-adaptive capabilities rather than static collaboration approaches; ii) as
timestamp increases, the learning performance improvement
brought by our lifelong-adaptive learning enlarges, reflecting
the effectiveness of our lifelong adaptive learning design.

\vspace{-3mm}
\subsubsection{Qualitative Analysis}
Here we aim to verify that the learning system of \texttt{DeLAMA} could dynamically evolve as the time stamp increases when faced with class incremental learning tasks. \textbf{Figure}~\ref{fig:classification graphs} visualizes the time-evolving collaboration structures produced by \texttt{DeLAMA}. Compared with the oracle structure, the results on \textbf{MNIST} and \textbf{CIFAR-10} dataset show that i) the collaboration relationships are evolving according to agents' observations; ii) agents tend to collaborate with more agents at the beginning, and delete less useful relationships as their observation grows, which is a reasonable collaboration strategy.

% \Note{CX:more compact, add text description of each row, one representative baseline for comparison?} \Note{sc: one column, half page? }
\begin{figure}
    \centering
    \includegraphics[width=\linewidth]{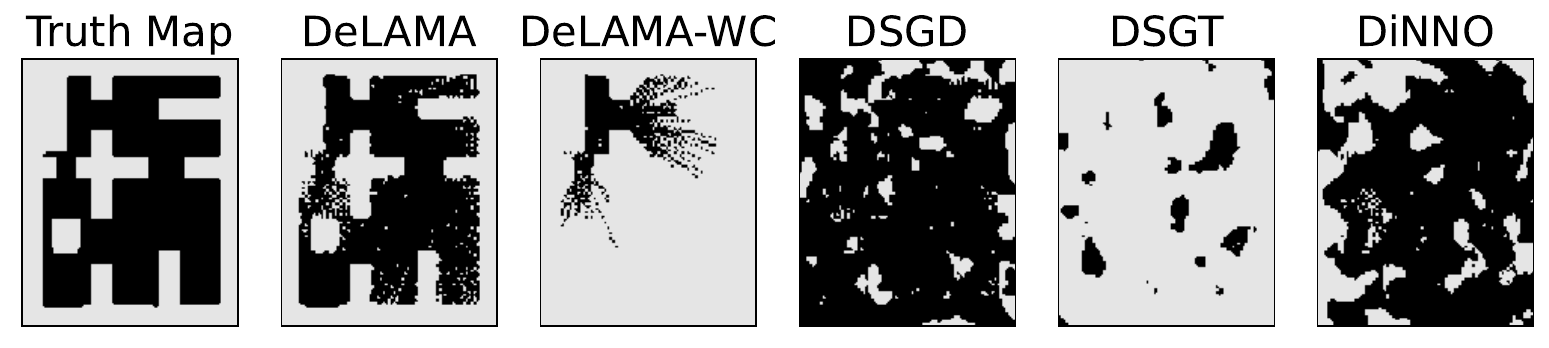}
    \vspace{-6mm}
    \caption{\small The simulation result of robot mapping at $t=12$ based on the CubiCasa5k room Dataset. \texttt{DeLAMA} outperforms other decentralized collaborative model training approaches.  \textbf{Black:} empty space. \textbf{White:} occupied space.}
    \vspace{-6mm}
    \label{fig:room_visual}
\end{figure}
\vspace{-3mm}
\vspace{-1mm}
\subsection{Multi-Robot Mapping}
\label{map}
\vspace{-1mm}
\subsubsection{Experimental Setup}

\noindent \textbf{Task.} We consider a collaboratively robot 2D mapping task with 5 robots. The room structures are generated from the CubiCasa5k dataset\cite{cubicasa5k}. In this environment exploration task, robots are placed randomly in the room and set a manually predefined exploration trajectory. Each robot can have a local view of the environment based on LiDAR samples along the trajectory and cannot visit every part of the room. The goal of this task is to realize the efficient and accurate mapping for each agent through collaboration.

\noindent \textbf{Dataset.} The room structure is generated from the CubiCasa5k dataset\cite{cubicasa5k}. To improve the learning task complexity, we increased the wall thickness, removed some small substructures such as doors, and balconies, and added some substructures into the rooms. The LiDAR scan is simulated by simple ray tracking methods and the angular resolution is one degree. Each time the LiDAR scan can create a local observation from each agent, serving as a dataset containing features of each scanned point. The features are representations of different types of sensor signals with environmental noise. In this task, the sampling duration contains 12 time stamps and 5 collaborative agents. We use 70\%  of the total room structures for training and the rest 30\% rooms for testing. 

\noindent \textbf{Implementation.} We treat this task as a binary classification task. Each point is either occupied or empty, which can be regarded as a positive label or a negative label. The detected map of each agent is indeed binary and can be evaluated by various classification metrics.

\noindent \textbf{Metric.} For regular rooms, occupied parts are much less than empty parts. Hence we utilize the F1 score of the binary classification task as the metric to represent the unbalanced classification task's classification power. 

\vspace{-3mm}
\subsubsection{Quantitative and Qualitative  Analysis}
In this section, we show that \texttt{DeLAMA} could be applied to multi-robot mapping scenarios. Specifically, we show that our decentralized collaborative mapping mechanism outperforms previous decentralized optimization and single-agent mappings. 

\textbf{Table}~\ref{tab:room_f1} shows the comparison with previous decentralized optimization methods~\cite{DSGD, DSGT, DiNNO}. We see that i) our collaboration learning strategy outperforms previous decentralized optimization methods and significantly improves the efficiency and performance of environmental exploration; ii) compared to \texttt{DeLAMA-WC} that is an individual agent mapping without collaboration, adding agent collaboration
remarkably improves the exploration performance. \textbf{Figure}~\ref{fig:room_visual} illustrates that i) \texttt{DeLAMA} outperforms other decentralized model training approaches, ii) compared to a single agent's limit perception range, multi-agent collaboration could discover a much more complete room structure. According to the visualization, the detection range and classification accuracy of the agent are improved with the help of multi-agent collaboration.

% \texttt{DeLAMA} & 71.37 & 79.85 & 89.23 & 92.17 & 94.84 \\
% \texttt{DeLAMA-WC} & 67.12 & 69.14 & 71.94 & 73.26 & 74.14\\
\begin{table}
    \vspace{-3mm}
    \caption{\small The agents' average F1-score of the multi-robot mapping task. \texttt{DeLAMA} outperforms other decentralized collaborative model training approaches.}
    \vspace{-3mm}
    \label{tab:room_f1}
    \centering
    \begin{tabular}{c|ccc}
    \toprule
         Method & ${\rm F1}_{t=1}$ &${\rm F1}_{t=6}$ &${\rm F1}_{t=12}$ \\
         \midrule
         DSGD~\cite{DSGD} &13.44 &43.90 &45.75\\
         DSGT~\cite{DSGT} &38.32 &38.31 &38.31\\
         DiNNO~\cite{DiNNO} &14.73 &20.34 &25.57\\
         % CHOCO-SGD~\cite{CHOCO-SGD} & & &\\
         % BrainTorrent~\cite{BrainTorrent} & & &\\
         \texttt{DeLAMA-WC} & 68.50  & 72.66  & 74.92\\
         \texttt{DeLAMA} & \textbf{72.11}  & \textbf{88.52}  & \textbf{95.31}\\
    \bottomrule
    \end{tabular}
    \vspace{-3mm}
\end{table}

\vspace{-3mm}
\subsection{Human Involved Experiment}
\label{human}
\vspace{-1mm}
\subsubsection{Experimental Setup}
\noindent \textbf{Purposes.} Here we designed a human experiment for multi-agent collaboration. It aims to verify whether human tends to collaborate with other people with similar knowledge, which corresponds to the graph smoothness modeling shown in~\eqref{prior}. We created a cognitive learning task formed by a collaborative learning game to investigate the following two problems:

$\bullet$ Whether the proposed algorithm obtains a compatible task performance compared to humans in cognitive learning scenarios.

$\bullet$ Does human collaboration share the same mechanism to find collaborators like the proposed graph learning algorithm?

% To demonstrate that DeLAMA's collaboration mechanism mirrors human cognitive collaboration, we adopt a linear regression task from Section 6.1 by substituting one system agent with a human participant.

\noindent \textbf{Task.} To demonstrate that \texttt{DeLAMA}'s collaboration mechanism mirrors human cognitive collaboration, we adopt a linear regression task as \textbf{Section}~\ref{regression} where each human participant represents one system agent. Here, to promote efficient human-machine interaction, we create a web application with a user-friendly GUI. From the human's perspective, i) the linear regression procedure is performed based on their visual observation and memory of previously encountered data points; ii) the interaction between the human participant and other agents occurs on the visual level, where the human can view the agents' regression outcomes and refine their regression results by drawing on the GUI based on the evaluation of other agents' results; see details in \textbf{Appendix}~\ref{appendix_humaneval}.

\vspace{-3mm}
\subsubsection{Quantitative Analysis}
To answer the first question, we aim to analyze the average performance of humans compared with \texttt{DeLAMA}. We set the number of agents $n=4$ with 3 timestamps in total. We set the two different target regression lines between the agents, which corresponds to two possible collaboration groups. We conducted three different levels of data quality. The qualitative results shown in \textbf{Table}~\ref{tab: human_exp_mse} reveal that \texttt{DeLAMA} outperforms human collaboration with a lower regression error at various data qualities.
\begin{table}
% \vspace{-3mm}
\caption{\small The average performance of human collaboration compared with \texttt{DeLAMA}. We analyzed the system performance at $t=1,2,3$. } 
\vspace{-3mm}
\label{tab: human_exp_mse}
    \centering
    \begin{tabular}{c|c|ccc}
    \toprule
         Method& Quality & $\text{MSE}_{t=1}$ &$\text{MSE}_{t=2}$ &$\text{MSE}_{t=3}$\\
         \midrule         
         \multirow{3}{*}{Human}  &Low &268.9874 &8.4663 &0.7053\\
         &Medium &\textbf{20.8548} &1.2173 &0.0859\\
         &High &0.0503 &0.0183 &0.0270\\
        
         \midrule
         \multirow{3}{*}{\texttt{DeLAMA}} &Low &\textbf{246.9219} &\textbf{4.1179} &\textbf{0.6533}\\
         &Medium &22.8994 &\textbf{0.6001} &\textbf{0.0600}\\
         &High &\textbf{0.0320} &\textbf{0.0120} &\textbf{0.0112}\\
    \bottomrule
    \end{tabular}
    \vspace{-3mm}
\end{table}
% [0.05027754 0.01830868 0.02698012] [0.03204667 0.01202998 0.01119648]
% [20.85479351  1.21733497  0.08590563] [22.89943785  0.60010934  0.06003689]
% [268.98735251   8.46627079   0.70534351] [246.92187255   4.11787191   0.65333043]
\begin{figure}
    \centering
    \vspace{-3mm}
    \includegraphics[width=\linewidth]{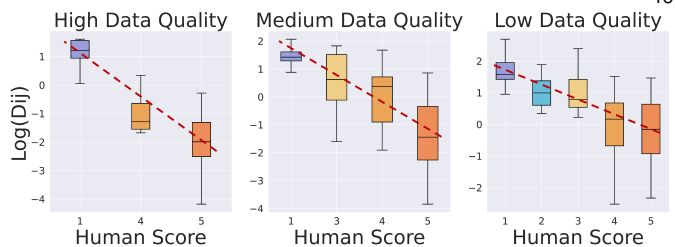}
    \vspace{-7mm}
    \caption{\small The human evaluation of collaboration relationships with each other. \textbf{x-axis:} human rating scores, ranging from 0 (lowest collaboration) to 5 (highest collaboration). \textbf{y-axis:} logarithm of model parameter distance. \textbf{From the first to last column:} different data quality human evaluation score distribution (highest to lowest). There exist negative correlations between agents' model parameters and human collaboration strength.}
    \label{fig:agentlevel_structure}
    \vspace{-3mm}
\end{figure}

\vspace{-3mm}
\subsubsection{Qualitative Analysis}
\vspace{-1mm}
To answer the second question, we aim to compare the correlations between the collaboration relationships and the corresponding agents' differences. From the definition of collaboration relation defined in \texttt{DeLAMA}~\eqref{w_graph}, the collaboration strength is negatively correlated with the distance of model parameters. Here we aim to verify that human collaboration still satisfies this property. Specifically, we substitute one agent in \texttt{DeLAMA} into a human and compare the relationship between learned edge weights and parameter distances. To quantify the collaboration strength between humans and collaborated machines, we ask humans to provide their rating scores (0 to 5, a higher score means a higher collaboration relationship) against other agents. \textbf{Figure}~\ref{fig:agentlevel_structure} demonstrates the relation of human rating scores and learned regression model parameter distances under different task data quality, where the y-axis is the logarithm of the distance defined as $\log\left(\mathbf{D}_{ij}^{(t)}\right) = \log \left(\left\|\boldsymbol{\theta}_i^{(t)} - \boldsymbol{\theta}_j^{(t)}\right\|_2\right)$. From the human score distribution, we see that 1) the collaboration relation strength is negatively correlated with the distance between the agents' model parameters, which verifies our assumption that humans tend to collaborate with similar ones; and 2) for harder collaboration problems with lower data quality, humans have difficulty in distinguishing useful collaborators. This suggests that the mechanisms of human collaboration are not merely based on the similarity between individuals. There may exist some more complex logic and structures, such as complex gaming relationships and higher-order interactions.

% related works
\vspace{-3mm}
\section{Related Work}
\label{sec:related_work}

\subsection{Multi-Agent Communication}
\label{multi-agent-communicate}
Multi-agent communication enables agents to send messages to each other to realize inter-agent collaboration. The key part of multi-agent communication lies in the designing of inter-agent communication strategy~\cite{learningwhen2communicate}. Traditional methods apply predefined features or heuristics \cite{camera_net, camera_control} to design communication protocols. More recently, learning-based methods have been proposed~\cite{multiagentcom_backprob}, which train the communication protocols end-to-end under task performance supervision. These approaches are often applied in perception or decision-making, such as collaborative perception or multi-agent reinforcement learning (MARL). In collaborative perception, current methods mainly focus on designing communication strategies~\cite{where2com} or finding the exact time to share~\cite{when2com}. In multi-agent reinforcement learning~\cite{marl_survey_1}, studies have explored various communication approaches, from simply sharing sensations, and policies to sharing other abstract data embeddings,like SARNet~\cite{marl-communication}, TarMAC~\cite{tarmac}, and IMAC~\cite{informativeMAC},  

Our method also focuses on communication between collaborative agents, but the differences are: i) Our communication strategy is derived from model-based numerical optimization problems with theoretical guarantees, rather than end-to-end training. ii) Our research mainly focuses on cognitive-level collaboration, rather than perception-level or decision-making-level collaborations.

\vspace{-3mm}
\subsection{Lifelong Learning}
\label{lifelong-learning}
Lifelong learning~\cite{continuallearningsurvey2} aims to train a continuously learning agent who can memorize previously learned tasks and quickly adapt to new training tasks. Optimizing model parameters directly each time may suffer from catastrophic forgetting~\cite{catastrophic-forget}. To overcome this issue, current approaches can be categorized into three aspects: regularization-based, rehearsal-based, and architecture-increasing-based methods. Regularization-based methods like EWC~\cite{ewc}, PathInt~\cite{PathInt}, and MAS~\cite{MAS} aim to estimate the important parts of model parameters and keep these parts changing slowly in model training, other methods like LWF~\cite{lwf}, target to keep previously learned data representations unchanged while learning new tasks. Rehearsal-based methods such as iCaRL~\cite{icarl}, EEIL~\cite{EEIL}, and GEM~\cite{gem} use previous tasks' data samples in real-time task learning, others like~\cite{generative-replay} utilize generative models to generate task data samples. TAMiL~\cite{TAMiL} combines rehearsal methods with regularization approaches together. Architecture-increasing methods like~\cite{progress-nn} try to utilize new task model parameters to prevent catastrophic forgetting, requiring relatively large memory capacities. 

Our method also aims to find an approach to overcome catastrophic forgetting and allow agents to memorize learning experiences, but the differences are: i) The application scenario is a group of collaborative agents, rather than a single agent, allowing collective phenomenon to contribute to lifelong learning. ii) Our methodology is based on algorithm unrolling, which is different from those three classic lifelong learning methodologies.

\vspace{-3mm}
\subsection{Distributed Multi-Task Learning}
\label{distributed-mtl}
Distributed multi-task learning~\cite{wang2016distributed} aims to train several models in a distributed manner. Classic methods such as DSML~\cite{wang2016distributed}, DMTL~\cite{zhang2012convex, liu2017distributed}, RMTL~\cite{evgeniou2004regularized}, AMTL~\cite{baytas2016asynchronous} proposed distributed model training algorithms based on numerical optimization, where most approaches are limited to linear models. Early works such as RMTL~\cite{evgeniou2004regularized} and analysis on the kernel for multi-task learning~\cite{evgeniou2005learning} put efforts into designing efficient kernels for multi-task learning. 

Our method also aims to distributively train several models, but the differences are: i) The model training algorithm is fully decentralized without the central server. ii) We consider time-evolving collaboration relations with adaptive capabilities, rather than static task distribution and collaboration relationships.

\vspace{-3mm}
\subsection{Federated Learning}
\label{federated-learning}
Federated learning (FL)~\cite{fedavg} aims to enable multiple agents collaboratively train models under the coordination of a central server. Traditional FL can be divided into two aspects: generalized FL~\cite{fedavg,fedprox} and personalized FL~\cite{ditto,fedrep}.

Generalized FL aims to train a global model that generalizes to datasets of all agents~\cite{fedavg}. To handle data heterogeneity, FedProx~\cite{fedprox} and SCAFFOLD~\cite{scaffold} propose to align agents' models in parameter space; while FedFM~\cite{fedfm} aligns agents' feature space. 
FedAvgM~\cite{fedavgm} and FedOPT~\cite{reddi2021adaptive} introduce momentum for updating global model.
FedNova~\cite{fednova} and FedDisco~\cite{feddisco} adjust the aggregation manner of traditional FL. 
Personalized FL aims to train multiple personalized models for agents' individual interests~\cite{smith2017federated}. 
Addressing data heterogeneity, Ditto~\cite{ditto} and pFedMe~\cite{pfedme} propose proximal regularization on parameter space.
FedPer~\cite{fedper} and FedRep~\cite{fedrep} split the whole model into two shallow and deep parts and keep the deep part localized. 
CFL~\cite{cfl} applies clustering on the participating clients and pFedGraph~\cite{pfedgraph} infers a collaboration graph among clients to promote fine-grained collaboration.

Unlike these methods that are coordinated by a central server, our proposed method finds the model parameters in a decentralized manner under the learned collaboration graph structure.
The model parameters are passed via the link on the graph and the aggregation rule for each agent is formulated by theoretical deduction.
This message-passing learning framework is much more robust and each agent could learn the model according to their own tasks without the constraint of the global model. 

% conclusion
\vspace{-4mm}
\section{Conclusion}
\label{sec:conclusion}

In this paper, we propose a novel decentralized lifelong-adaptive collaborative learning framework based on numerical optimization and algorithm unrolling, named~\texttt{DeLAMA}. It enables multiple agents to efficiently detect collaboration relationships and adapt to ever-changing observations during individual model training. We validate the effectiveness of \texttt{DeLAMA} through extensive experiments on various real-world and simulated datasets. Experimental results show that \texttt{DeLAMA} achieves superior performances compared with other collaborative learning approaches.

% acknowledgement
\appendices

\vspace{-4mm}
\ifCLASSOPTIONcompsoc
  \section*{Acknowledgments}
\else
  \section*{Acknowledgment}
\fi

This research is supported by the National Key R\&D Program of China under Grant 2021ZD0112801, NSFC under Grant 62171276 and the Science and Technology Commission of Shanghai Municipal under Grant 21511100900 and 22DZ2229005.

\ifCLASSOPTIONcaptionsoff
  \newpage
\fi

\bibliographystyle{IEEEtran}
\bibliography{main}

\vspace{-12mm}

\begin{IEEEbiography}
[{\includegraphics[width=1in,height=1.25in, clip,keepaspectratio]{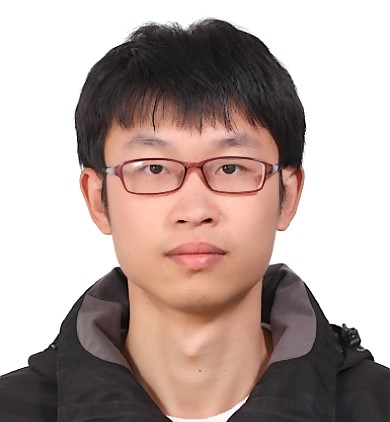}}]{Shuo Tang} is currently working toward a Ph.D. degree at the Cooperative Medianet Innovation Center in Shanghai Jiao Tong University since 2022. Before that, he received the B.E. degree in Computer Science from Shanghai Jiao Tong University, Shanghai, China, in 2022. His research interests include multi-agent collaboration and communication efficiency.
\end{IEEEbiography}
\vspace{-6mm}

\vspace{-6mm}
\begin{IEEEbiography}
[{\includegraphics[width=1in,height=1.25in, clip,keepaspectratio]{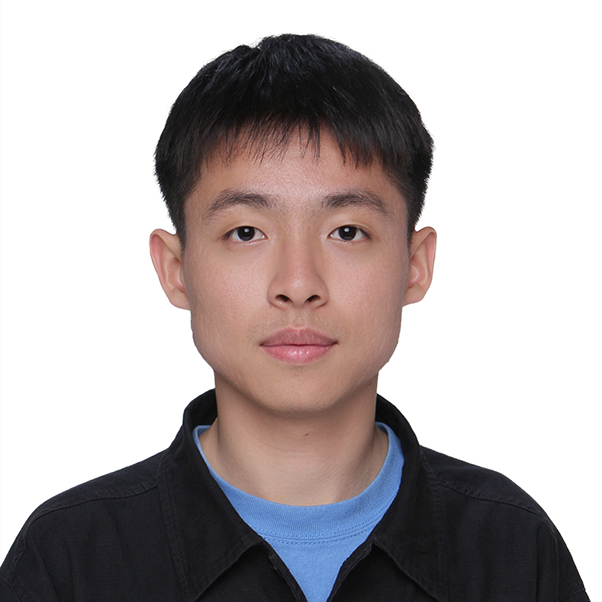}}]{Rui Ye} is currently working toward a Ph.D. degree at the Cooperative Medianet Innovation Center in Shanghai Jiao Tong University since 2022. Before that, he received the B.E. degree in Information Engineering from Shanghai Jiao Tong University, Shanghai, China, in 2022. He was a Research Intern with Microsoft Research Asia in 2022 and 2023. His research interests include federated learning, trustworthy large language models, and multi-agent collaboration. 
\end{IEEEbiography}
\vspace{-6mm}

\vspace{-6mm}
\begin{IEEEbiography}
[{\includegraphics[width=1in,height=1.25in, clip,keepaspectratio]{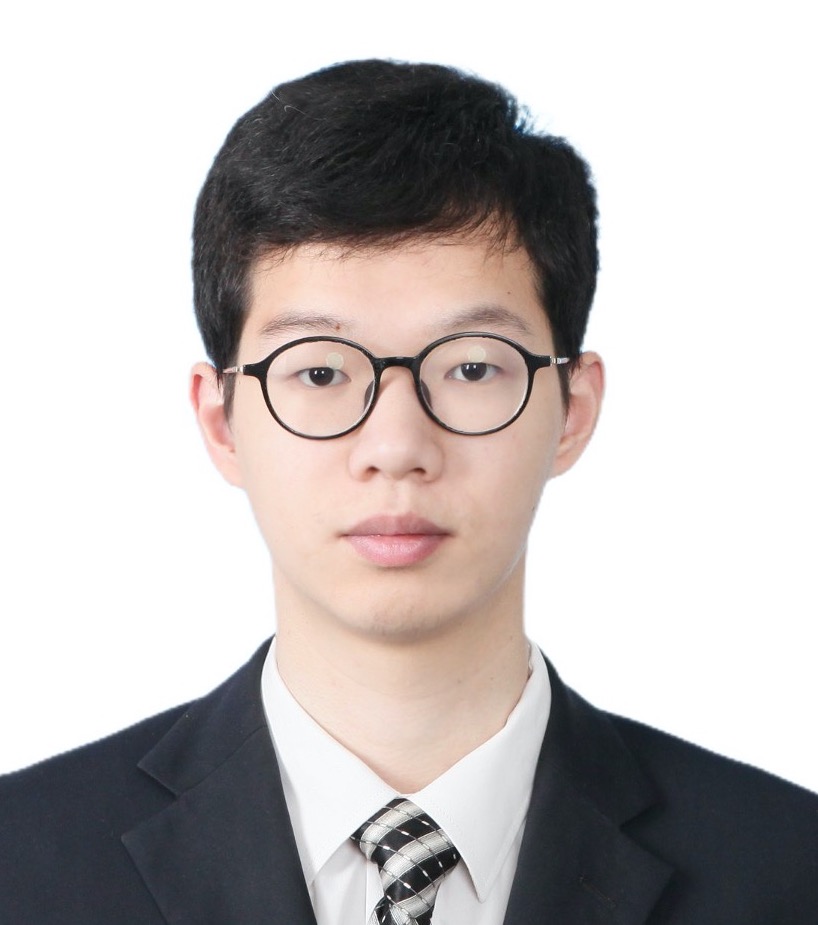}}]{Chenxin Xu} is working toward the joint Ph.D. degree at Cooperative Medianet Innovation Center in Shanghai Jiao Tong University and at Electrical and Computer Engineering in National University of Singapore since 2019. He received the B.E. degree in information engineering from Shanghai Jiao Tong University in 2019. His research interests include trajectory prediction and multi-agent system.  He is the reviewer of some prestigious international journals
and conferences, including IEEE-TPAMI, CVPR, ICCV, ICML and NeurIPS.
\end{IEEEbiography}
\vspace{-6mm}

\vspace{-6mm}
\begin{IEEEbiography}[{\includegraphics[width=1in,height=1.25in, clip,keepaspectratio]{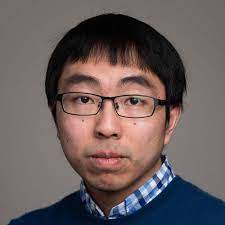}}]{Xiaowen Dong} is an associate professor in the Department of Engineering Science at the University of Oxford, where he is also a member of both the Machine Learning Research Group and the Oxford-Man Institute. His main research interests concern signal processing and machine learning techniques for analysing network data, and their applications in social and economic sciences.
\end{IEEEbiography}
\vspace{-6mm}

\vspace{-6mm}
\begin{IEEEbiography}[{\includegraphics[width=1in,height=1.25in, clip,keepaspectratio]{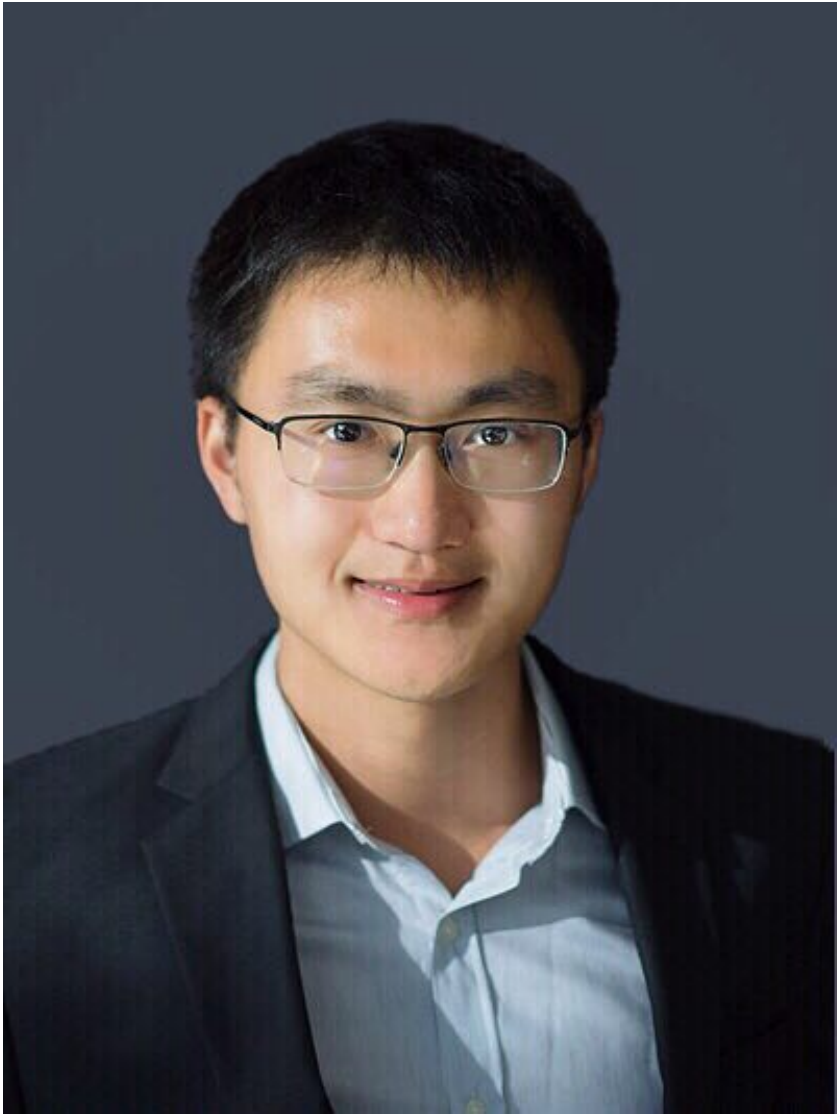}}]{Siheng Chen} is a tenure-track associate professor of Shanghai Jiao Tong University. He was a research scientist at Mitsubishi Electric Research Laboratories (MERL), and an autonomy engineer at Uber Advanced Technologies Group (ATG), working on self-driving cars. Dr. Chen received his doctorate from Carnegie Mellon University in 2016. Dr. Chen's work on sampling theory of graph data received the 2018 IEEE Signal Processing Society Young Author Best Paper Award. He contributed to the project of scene-aware interaction, winning MERL President's Award. His research interests include graph machine learning and collective intelligence.
\end{IEEEbiography}
\vspace{-6mm}

\vspace{-6mm}
\begin{IEEEbiography}[{\includegraphics[width=1in,height=1.25in,clip,keepaspectratio]{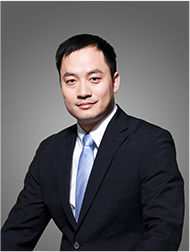}}]{Yanfeng Wang} received the B.E. degree in information engineering from the University of PLA, Beijing, China, and the M.S. and Ph.D. degrees in business management from the Antai College of Economics and Management, Shanghai Jiao Tong University, Shanghai, China. He is currently the Vice Director of the Cooperative Medianet Innovation Center and also the Vice Dean of the School of Electrical and Information Engineering, Shanghai Jiao Tong University. His research interests mainly include media big data and emerging commercial applications of information technology.
\end{IEEEbiography}
\vspace{-6mm}

% appendix
\newpage
\appendices
\section{Proofs and analysis}
\label{sec:proof}
\subsection{Analysis to theorem~\ref{best point taylor}}
\label{proof-to-best-point-taylor}
Our analysis starts from the strong convexity of the loss function $\mathcal{L}\left(f_{\boldsymbol{\theta}_i^{(t)}}\left(\mathbf{X}_i^{(t)}\right), \mathbf{Y}_i^{(t)}\right)$. Since the model $f_{\boldsymbol{\theta}_i^{(t)}}(\cdot)$ is linear and the loss function is standard supervised training loss, the loss function is convex. Also the condition number of the Hessian $\mathbf{H}\left(\boldsymbol{\alpha}_i^{(t)}\right)=\nabla_{\boldsymbol{\alpha}_i^{(t)}}^2 \mathcal{L}\left(f_{\boldsymbol{\theta}_i^{(t)}}\left(\mathcal{X}_i^{(t)}\right), \mathcal{Y}_i^{(t)}\right)$ is bounded, which means the eigenvalues are all strictly large than zero. Based on this observation, we provide an upper-bound of $\left\| \boldsymbol{\theta}_i^{(t)} - \boldsymbol{\theta}_i^{(t)*} \right\|_2$ shown in \textbf{Lemma}~\ref{bound-for-error}.

\begin{mylem}
\label{bound-for-error}
     Suppose for the $i$th agent at time stamp $t$, given training data $\left(\mathbf{X}_i^{(t)}, \mathbf{Y}_i^{(t)}\right)$, let the linear model is $f_{\boldsymbol{\theta}_i^{(t)}}(\cdot)$ with loss function $\mathcal{L}\left(f_{\boldsymbol{\theta}_i^{(t)}}\left(\mathbf{X}_i^{(t)}\right), \mathbf{Y}_i^{(t)}\right)$. Let $\boldsymbol{\alpha}_i^{(t)}$ be the point doing Taylor expansion and the optimal parameter of the loss function is $\boldsymbol{\theta}_i^{(t)*}$. Then:
    \begin{align*}
        \left\| \boldsymbol{\theta}_i^{(t)} - \boldsymbol{\theta}_i^{(t)*} \right\|_2 \leq  \frac{L}{2}  \left\| \mathbf{H}\left(\boldsymbol{\alpha}_i^{(t)}\right)^{-1} \right\| \left\| \boldsymbol{\alpha}_i^{(t)} - \boldsymbol{\theta}_i^{(t)*} \right\|_2^2,
    \end{align*}
    where $L$ is the Lipschitz constant of the Hessian $\mathbf{H}\left(\boldsymbol{\alpha}_i^{(t)}\right)$.
\end{mylem}

\begin{proof}
    
    For the sake of brevity, we write $\boldsymbol{\theta}_i^{(t)}$, $\boldsymbol{\theta}_i^{(t),*}$ and $\boldsymbol{\alpha}_i^{(t)}$ as $\boldsymbol{\theta}$, $\boldsymbol{\theta}^*$ and $\boldsymbol{\alpha}$. Consider the error between $\boldsymbol{\theta}$ and $\boldsymbol{\theta}^*$,  we have
    \begin{align*}
        \boldsymbol{\theta} - \boldsymbol{\theta}^* &= \boldsymbol{\alpha} - \boldsymbol{\theta}^* - \mathbf{H}(\boldsymbol{\alpha})^{-1} \nabla_{\boldsymbol{\alpha}} \mathcal{L}\\
        &= \boldsymbol{\alpha} - \boldsymbol{\theta}^* - \mathbf{H}(\boldsymbol{\alpha})^{-1}\left(\nabla_{\boldsymbol{\alpha}}\mathcal{L} - \nabla_{\boldsymbol{\theta}^*} \mathcal{L}\right) 
    \end{align*}
    define $h(t) = \nabla_{\boldsymbol{\theta}^* + t(\boldsymbol{\alpha} - \boldsymbol{\theta}^*)}\mathcal{L}$, thus
    \begin{align*}
        \hat{\boldsymbol{\theta}} - \boldsymbol{\theta}^* &= \boldsymbol{\alpha} - \boldsymbol{\theta}^* - \mathbf{H}(\boldsymbol{\alpha})^{-1} (h(1) - h(0))\\
        &= \boldsymbol{\alpha} - \boldsymbol{\theta}^* - \mathbf{H}(\boldsymbol{\alpha})^{-1} \int_{0}^{1} h^\prime(s) ds \\
        &= \boldsymbol{\alpha} - \boldsymbol{\theta}^* -\\
        & \mathbf{H}(\boldsymbol{\alpha})^{-1}  \int_{0}^{1} \mathbf{H}(\boldsymbol{\theta}^* + s(\boldsymbol{\alpha} - \boldsymbol{\theta}^*)) \left(\boldsymbol{\alpha} - \boldsymbol{\theta}^* \right) ds \\
        &= \mathbf{H}(\boldsymbol{\alpha})^{-1} \\
        & \times \int_0^1 \left[ \mathbf{H}(\boldsymbol{\alpha}) - \mathbf{H}\left(\boldsymbol{\theta}^* + s(\boldsymbol{\alpha} - \boldsymbol{\theta}^*)\right) \right](\boldsymbol{\alpha} - \boldsymbol{\theta}^*)ds
    \end{align*}
    taking norm on both sides, we obtain
    \begin{align*}
        &\left \| \hat{\boldsymbol{\theta}} - \boldsymbol{\theta}^* \right\|_2 = \\
        & \left \| \mathbf{H}(\boldsymbol{\alpha})^{-1} \int_0^1 \left[ \mathbf{H}(\boldsymbol{\alpha}) - \mathbf{H}\left(\boldsymbol{\theta}^* + s(\boldsymbol{\alpha} - \boldsymbol{\theta}^*)\right) \right](\boldsymbol{\alpha} - \boldsymbol{\theta}^*)ds \right \|_2 \\
        & \leq \left \| \mathbf{H}(\boldsymbol{\alpha})^{-1} \right \|\\
        & \times \left \|\int_0^1 \left[ \mathbf{H}(\boldsymbol{\alpha}) - \mathbf{H}\left(\boldsymbol{\theta}^* + s(\boldsymbol{\alpha} - \boldsymbol{\theta}^*)\right) \right](\boldsymbol{\alpha} - \boldsymbol{\theta}^*)ds \right\|_2\\
        & \leq \left \| \mathbf{H}(\boldsymbol{\alpha})^{-1} \right \| \\
        & \times \int_0^1 \left\| \mathbf{H}(\boldsymbol{\alpha}) - \mathbf{H}\left(\boldsymbol{\theta}^* + s(\boldsymbol{\alpha} - \boldsymbol{\theta}^*)\right) \right\| \left\|\boldsymbol{\alpha} - \boldsymbol{\theta}^*\right \|_2 ds
    \end{align*}
    For linear regression task, the hessian matrix $\mathbf{H}(\boldsymbol{\theta})$ remains constant, then $\left \| \mathbf{H}(\boldsymbol{\alpha}) - \mathbf{H}\left(\boldsymbol{\theta}^* + s(\boldsymbol{\alpha} - \boldsymbol{\theta}^*)\right) \right \| = 0$, which means $\left \| \hat{\boldsymbol{\theta}} - \boldsymbol{\theta}^*  \right\|_2 = 0$.
    
    For the linear classification tasks, the hessian matrix is Lipschitz continuous, thus there exists a constant $L$ such that
    \begin{align*}
        \left \| \mathbf{H}(\boldsymbol{\theta}_1) - \mathbf{H}\left(\boldsymbol{\theta}_2\right) \right \| \leq L \|\boldsymbol{\theta}_1 -\boldsymbol{\theta}_2 \|
    \end{align*}
    taking this property back into the former inequality, we obtain
    \begin{align*}
        \left\| \hat{\boldsymbol{\theta}} - \boldsymbol{\alpha} \right\|_2 &\leq L \left \| \mathbf{H}(\boldsymbol{\alpha})^{-1} \right \| \int_0^1 \left \| \boldsymbol{\alpha} - \boldsymbol{\theta}^* \right \|_2^2 (1-s) ds\\
        &= \frac{L}{2} \left \| \mathbf{H}(\boldsymbol{\alpha})^{-1} \right \| \left \| \boldsymbol{\alpha} - \boldsymbol{\theta}^* \right \|_2^2
    \end{align*}

\end{proof}
Our target is minimizing the error upper bound shown in \textbf{Lemma}~\ref{bound-for-error}. For regression loss functions, the Hessian is constant, making the Lipschitz constant $L$ equal to zero, which means doing expansion at any point $\boldsymbol{\alpha}_i^{(t)}$ is zero error. For other tasks with non-zero Lipschitz constant $L$ such as classification, this goal is equivalent to solving the following constraint optimization problem:
\begin{equation}
    \label{taylor problem original}
    \begin{aligned}
    \min_{\boldsymbol{\alpha}_i^{(t)}} & \mathbb{E}_{\boldsymbol{\theta}_i^{(t)*}}\left(\left\|\mathbf{H}\left(\boldsymbol{\alpha}_i^{(t)}\right)^{-1}\right\| \left\|\boldsymbol{\alpha}_i^{(t)} - \boldsymbol{\theta}_i^{(t)*} \right\|_2^2\right)\\
    & s.t. \  1 \leq k\left(\mathbf{H}\left(\boldsymbol{\alpha}_i^{(t)}\right)\right) \leq M,
    \end{aligned}
\end{equation}
where the prior distribution of the target optimal parameter $\boldsymbol{\theta}_i^{(t)*}$ is $p\left(\boldsymbol{\theta}_i^{(t)*}\right) \sim \mathcal{N}\left(0, \mathbf{\Sigma}\right)$. Due to the difficulty of capturing the spectral radius of the inverse Hessian, considering the eigenvalues $0 <\lambda_1 \leq \lambda_2,\ldots, \leq \lambda_K$ of the objective hessian $\mathbf{H}\left(\boldsymbol{\alpha}_i^{(t)}\right)$, we obtain 
\begin{align*}
    \frac{K}{\sum_{i=1}^K \lambda_i} \leq \left\|\mathbf{H}\left(\boldsymbol{\alpha}_i^{(t)}\right)^{-1}\right\| = \frac{1}{\lambda_1} \leq \frac{KM}{\sum_{i=1}^K \lambda_i},
\end{align*}
which means $\left\|\mathbf{H}\left(\boldsymbol{\alpha}_i^{(t)}\right)^{-1}\right\|$ and $\operatorname{\textbf{tr}}\left(\mathbf{H}\left(\boldsymbol{\alpha}_i^{(t)}\right)\right)^{-1}$ are the same order. Hence we obtain the following final optimization problem:
\begin{equation}
    \label{taylor problem}
    \begin{aligned}
        \min_{\boldsymbol{\alpha}_i^{(t)}} & \mathbb{E}_{\boldsymbol{\theta}_i^{(t)*}}\left( \operatorname{\textbf{tr}}\left(\mathbf{H}\left(\boldsymbol{\alpha}_i^{(t)}\right)\right)^{-1} \left\|\boldsymbol{\theta}_i^{(t)} - \boldsymbol{\theta}_i^{(t)*} \right\|_2^2\right)\\
        & s.t. \  1 \leq k\left(\mathbf{H}\left(\boldsymbol{\alpha}_i^{(t)}\right)\right) \leq M.
    \end{aligned}
\end{equation}
Then we give our proof to \textbf{Theorem}~\ref{best point taylor} in terms of a classification model with $C$ types of output classes.
\begin{proof}
    For the sake of brevity, we write $\boldsymbol{\theta}_i^{(t)}$, $\boldsymbol{\alpha}_i^{(t)}$ as $\boldsymbol{\theta}$ and $\boldsymbol{\alpha}$. The linear model $f_{\boldsymbol{\theta}}(\cdot)$ corresponds to the $i$th class of classification task is $\boldsymbol{\theta}^i$ with training data $\left(\mathbf{X} \in \mathbb{R}^{n\times p}, \mathbf{Y} \in \mathbb{R}^{n \times C} \right)$ where $C$ equals to the number of classes, $p$ is the dimension of input data. The model estimates the probability $\hat{p}^j \in \mathbb{R}^{C}$ of the $j$th element $\boldsymbol{x}_j$ as
    \begin{align*}
        \hat{p}_i^j = \operatorname{exp}\left(\boldsymbol{\theta}^{i\top}\boldsymbol{x}_j\right) / \sum_{i=1}^{C} \operatorname{exp} \left(\boldsymbol{\theta}^{i \top}\boldsymbol{x}_j\right)
    \end{align*}
    Hence the Hessian equals to
    \begin{align*}
        \mathbf{H}\left(\boldsymbol{\alpha}\right)  =\sum_{j=1}^{n} \mathbf{M}_j \otimes \boldsymbol{x}_j \boldsymbol{x}_j^\top,
    \end{align*}
    where $\mathbf{M}_j$ is defined as
    \begin{align*}
        \mathbf{M}_j = \left(\begin{array}{cccc}
        \hat{p}_1^j\left(1-\hat{p}_1^j\right)  & -\hat{p}_1^j \hat{p}_2^j  & \cdots & -\hat{p}_1^j \hat{p}_C^j  \\
        \vdots & \hat{p}_2^j\left(1-\hat{p}_2^j\right)  & & \vdots \\
        -\hat{p}_C^j \hat{p}_1^j  & \cdots & \cdots & \hat{p}_C^j\left(1-\hat{p}_C^j\right) 
        \end{array}\right).
    \end{align*}
 
Based on these definitions and notations, the bound shown in \textbf{Theorem}~\ref{best point taylor} can be reformulated as
    \begin{align*}
        &\operatorname{\textbf{tr}}(\mathbf{H}(\boldsymbol{\alpha}))^{-1}\mathbb{E}_{\boldsymbol{\theta}^*}\left(\|\boldsymbol{\alpha}\|_2^2 - 2 \boldsymbol{\alpha}^\top \boldsymbol{\theta}^* + \|\boldsymbol{\theta}^*\|_2^2\right)\\
        &= \operatorname{\textbf{tr}}(\mathbf{H}(\boldsymbol{\alpha}))^{-1} \left(\|\boldsymbol{\alpha}\|_2^2 - 2\boldsymbol{\alpha}^\top \mathbb{E}_{\boldsymbol{\theta}^*}(\boldsymbol{\theta}^*) + \mathbb{E}_{\boldsymbol{\theta}^*}\left(\|\boldsymbol{\theta}^*\|_2^2\right)\right)\\
        &= \operatorname{\textbf{tr}}(\mathbf{H}(\boldsymbol{\alpha}))^{-1} \left(\|\boldsymbol{\alpha}\|_2^2 + \operatorname{\textbf{tr}}(\Sigma^{-1})\right)\\
        &\geq \operatorname{\textbf{tr}}(\mathbf{H}(\boldsymbol{\alpha}))^{-1} \operatorname{\textbf{tr}}(\Sigma^{-1})
    \end{align*}
    However, for $\operatorname{\textbf{tr}}\left(\mathbf{H}(\boldsymbol{\alpha})\right)$ we have
    \begin{align*}
        &\operatorname{\textbf{tr}}(\mathbf{H}(\boldsymbol{\alpha})) = \operatorname{\textbf{tr}}\left(\sum_{j=1}^n\mathbf{M}_j \otimes \boldsymbol{x}_j\boldsymbol{x}_j^\top\right)\\
        &=\sum_{j=1}^n \operatorname{\textbf{tr}}\left(\mathbf{M}_j \otimes \boldsymbol{x}_j\boldsymbol{x}_j^\top\right)\\
        &=\sum_{j=1}^n \|\boldsymbol{x}_j\|_2^2 \sum_{i=1}^C \hat{p}_i^j\left(1 - \hat{p}_i^j\right)\\
        &=\sum_{j=1}^n \|\boldsymbol{x}_j \|_2^2 \left(1 - \sum_{i=1}^C \hat{p}_i^{j2}\right) \leq \left(1 - \frac{1}{C}\right) \sum_{j=1}^n \|\boldsymbol{x}_j\|_2^2
    \end{align*}
    Taking this inequality back into the original bound, we have
    \begin{align*}
         &\operatorname{\textbf{tr}}(\mathbf{H}(\boldsymbol{\alpha}))^{-1}\mathbb{E}_{\boldsymbol{\theta}^*}\left(\|\boldsymbol{\alpha}\|_2^2 - 2 \boldsymbol{\alpha}^\top \boldsymbol{\theta}^* + \|\boldsymbol{\theta}^*\|_2^2\right)\\
         &\geq \operatorname{\textbf{tr}}(\mathbf{\Sigma}^{-1})\left[\left(1 - \frac{1}{C}\right)\sum_{i=1}^n \|\boldsymbol{x}_i\|_2^2 \right]^{-1}.
    \end{align*}
    The equality is satisfied when $\hat{p}_1=\hat{p}_2 = \ldots = \hat{p}_c$ and $\|\boldsymbol{\alpha}\|_2^2=0$, which is equivalent to $\boldsymbol{\alpha}=0$.

\end{proof}

\subsection{Analysis to theorem~\ref{graph learning approx}}
\label{proof-to-graph-learning-approx}
Before we provide our final proof, we first state a simple observation shown in the following lemma:
\begin{mylem}
\label{simple_inequality}
    Let $\boldsymbol{h}_b(x) = \left(\sqrt{x^2 + b} + x\right)/2$, then
    $$\forall x \in \mathbb{R}, \left|\boldsymbol{h}_b(x) - \operatorname{ReLU}(x)\right| \leq \frac{\sqrt{b}}{2}$$
\end{mylem}
\begin{proof}
    If $x\geq 0$, then 
    \begin{align*}
         &\left|\boldsymbol{h}_b(x) - \operatorname{ReLU}(x)\right| = \left|\frac{\sqrt{x^2 + b} + x}{2} - x\right|\\
         &= \left|\frac{\sqrt{x^2 + b} - x}{2} \right| = \left|\frac{b}{2(\sqrt{x^2 + b} + x)} \right| \leq \frac{\sqrt{b}}{2}.
    \end{align*}
    If $x < 0$, then $\left|\boldsymbol{h}_b(x) - \operatorname{ReLU}(x)\right| = \left| \left(\sqrt{x^2 + b} + x\right)/2 \right|$. Thus it also satisfies $\left|\boldsymbol{h}_b(x) - \operatorname{ReLU}(x)\right|<\sqrt{b}/2$.
   
\end{proof}
To begin with, rewrite the KKT condition of the solution as 
\begin{equation}
    \begin{aligned}
        & \boldsymbol{w}^{(t)} = \operatorname{ReLU}\left(\alpha \boldsymbol{y} + \beta z \boldsymbol{1}\right)\\
        & \boldsymbol{1}^\top \boldsymbol{w}^{(t)} = \boldsymbol{m}
    \end{aligned}
\end{equation}
where $\boldsymbol{w}^{(t)}\in \mathbb{R}^{N^2}$ is the vectorized version of $\mathbf{W}^{(t)}$, $\boldsymbol{y} \in \mathbb{R}^{N^2}$ is the vectorized version of all the agents' model parameter distance. $\boldsymbol{1}\in \mathbb{R}^{N^2}$ is all one vector, $\alpha = -\lambda_2 / 2\lambda_3$, $\beta = -1/2 \lambda_3$. Thus when we change the $\operatorname{ReLU}$ into function $h$, the KKT condition of the new solution $\boldsymbol{w}^{(t)*}$ has changed to 
\begin{equation}
    \begin{aligned}
         & \boldsymbol{w}^{(t)*} = \boldsymbol{h}_b\left(\alpha \boldsymbol{y} + \beta z^* \boldsymbol{1}\right)\\
        & \boldsymbol{1}^\top \boldsymbol{w}^{(t)*} = \boldsymbol{m}
    \end{aligned}
\end{equation}
Based on these notations and definitions we start to prove \textbf{Theorem}~\ref{graph learning approx}.
\begin{proof} 
    Define $\boldsymbol{p} = \alpha \boldsymbol{y} + \beta z \boldsymbol{1}$, $\boldsymbol{p}^* = \alpha \boldsymbol{y} + \beta \boldsymbol{z}^* \boldsymbol{1}$. Consider $\|\boldsymbol{w} - \boldsymbol{w}^*\|_2^2$, we have
    \begin{align*}
        &\|\boldsymbol{w} - \boldsymbol{w}^*\|_2^2 = \|\operatorname{ReLU}(\boldsymbol{p}) - \boldsymbol{h}_b(\boldsymbol{p}^*)\|_2^2\\
        &= \| \operatorname{ReLU}(\boldsymbol{p}) - \boldsymbol{h}_b(\boldsymbol{p}) + \boldsymbol{h}_b(\boldsymbol{p}) - \boldsymbol{h}_b(\boldsymbol{p}^*)\|_2^2\\
        &\leq \| \operatorname{ReLU}(\boldsymbol{p}) - \boldsymbol{h}_b(\boldsymbol{p})\|_2^2 + \|\boldsymbol{h}_b(\boldsymbol{p}) - \boldsymbol{h}_b(\boldsymbol{p}^*)\|_2^2\\
        &\leq \frac{N^2b}{4} + \sum_{i=1}^{N^2} \left[\boldsymbol{h}_b^\prime(\xi_i)\beta(\boldsymbol{z} - \boldsymbol{z}^*)\right]^2\\
        &\leq \frac{N^2b}{4} + N^2\beta^2 (\boldsymbol{z}-\boldsymbol{z}^*)^2
    \end{align*}
    On the other hand, define $\phi(x) = \boldsymbol{1}^\top \boldsymbol{h}_b(\alpha \boldsymbol{y} + \beta x \boldsymbol{1})$, we have
    \begin{align*}
        & \phi(\boldsymbol{z}) - m = \phi(\boldsymbol{z}) - \phi(\boldsymbol{z}^*) = \phi^\prime (\xi) (\boldsymbol{z} - \boldsymbol{z}^*)
    \end{align*}
    where $\xi \in [\boldsymbol{z}, \boldsymbol{z}^*]$. Thus
    \begin{align*}
        &|\boldsymbol{z} - \boldsymbol{z}^*| = \left|\frac{\phi(\boldsymbol{z}) - m}{\phi^\prime(\xi)}\right|\\
        & = \frac{|\phi(\boldsymbol{z}) - \boldsymbol{1}^\top \operatorname{ReLU}(\alpha \boldsymbol{y} + \beta \boldsymbol{z} \boldsymbol{1})|}{\left|\sum_{i=1}^{N^2} \boldsymbol{h}_b^\prime(\alpha \boldsymbol{y}_i + \beta \xi ) \beta\right|}\\
        & \leq \frac{|\phi(\boldsymbol{z}) - \boldsymbol{1}^\top \operatorname{ReLU}(\alpha \boldsymbol{y} + \beta \boldsymbol{z} \boldsymbol{1})|}{N^2C |\beta|}\\
        & \leq \frac{N^2 \sqrt{b}}{2 N^2 C |\beta|} = \frac{\sqrt{b}}{2 C |\beta|},
    \end{align*}
    where the second inequality comes from the \textbf{Lemma}~\ref{simple_inequality}. Taking the inequality back into $\|\boldsymbol{w} - \boldsymbol{w}^*\|_2^2$, we obtain
    \begin{align*}
        &\|\boldsymbol{w} - \boldsymbol{w}^*\|_2^2 \leq \frac{N^2 b}{4} + \frac{N^2b}{4C^2} = \frac{N^2 b}{4} \left[1 + \frac{1}{C^2}\right]
    \end{align*}
    which means the error $\|\boldsymbol{w} - \boldsymbol{w}^*\|_2$ is bounded by $\frac{N \sqrt{b}}{2} \sqrt{1 + \frac{1}{C^2}}$.
\end{proof}

\subsection{Analysis to theorem~\ref{convergence}}
\label{proof-to-convergence}
Recall that the optimal solution corresponds to the first-order condition shown in \eqref{foc}. However, this condition requires each agent to know $\mathbf{W}_{ij}^{(t)}$ and $\mathbf{W}_{ji}^{(t)}$ both. In \texttt{DeLAMA}, our real message passing mechanism~\eqref{eq:message_passing} only requires the agent to know $\mathbf{W}_{ij}^{(t)}$ for the $i$th agent. In the following lemma, we first claim that this mechanism does converge to the exact optimal solution by demonstrating that $\mathbf{W}^{(t)}$ is symmetric.
\begin{mylem}
\label{symmetric_graph_lemma}
    The collaboration graph structure $\mathbf{W}^{(t)} \in \mathbb{R}^{N \times N}$ learned from \textbf{Algorithm}~\ref{alg3} is symmetric.
\end{mylem}
The proof comes from a simple observation. First rewrite the optimization problem~\eqref{graph_learning_optimization_problem}:
\begin{align*}
    &\min_{\mathbf{W}^{(t)}} \  \mathcal{L}\left(\mathbf{W}^{(t)}\right)=\lambda_2 \|\mathbf{W}^{(t)} \odot \mathbf{D}^{(t)}\|_1 + \lambda_3 \| \mathbf{W}^{(t)}\|_{\mathbf{F}}^2 \\
    & \text{s.t.} \  \|\mathbf{W}^{(t)} \|_1 = \boldsymbol{m}, \ \mathbf{W}^{(t)} \geq 0, \ \operatorname{diag}\left(\mathbf{W}^{(t)}\right) = \boldsymbol{0}
\end{align*}
where $\mathbf{D}_{ij}^{(t)}$ stands for $\left\|\boldsymbol{\theta}_i^{(t)} - \boldsymbol{\theta}_j^{(t)}\right\|_2^2$ and $\odot$ is element wise production. Then the observation is, for any feasible solution $\mathbf{W}_0^{(t)}$ of this optimization problem, $\mathbf{W}_0^{(t)\top}$ is also feasible and $\mathcal{L}\left(\mathbf{W}_0^{(t)}\right) = \mathcal{L}\left(\mathbf{W}_0^{(t)\top}\right)$. In fact, we have the relationship  $\lambda_2 \left\|\mathbf{W}_0^{(t)\top} \odot \mathbf{D}^{(t)}\right\|_1 + \lambda_3 \left\|\mathbf{W}_0^{(t)\top}\right\|_{\mathbf{F}}^2 = \lambda_2 \left\|\mathbf{W}_0^{(t)\top} \odot \mathbf{D}^{(t)\top}\right\|_1 + \lambda_3 \left\|\mathbf{W}_0^{(t)\top}\right\|_{\mathbf{F}}^2 = \lambda_2 \left\|\mathbf{W}_0^{(t)}\odot \mathbf{D}^{(t)}\right\|_1 + \lambda_3 \left\|\mathbf{W}_0^{(t)}\right\|_{\mathbf{F}}^2$, which means the value of objective function remains unchanged. Thus the proof of \textbf{Lemma}~\ref{symmetric_graph_lemma} is as follows:
\begin{proof}
    \label{symmetric graph proof}
    $\|\mathbf{W}^{(t)} \odot \mathbf{D}^{(t)}\|_1$ is convex, thus $\mathcal{L}\left(\mathbf{W}^{(t)}\right)$ is strongly convex, making the optimization problem with a unique solution $\mathbf{W}^{(t)*}$.
    However, from our observation, the transpose of $\mathbf{W}^{(t)*}$ is also a feasible solution, and $\mathcal{L}\left(\mathbf{W}^{(t)*}\right) = \mathcal{L}\left(\mathbf{W}^{(t)*\top}\right)$, thus $\mathbf{W}^{(t)*} = \mathbf{W}^{(t)*\top}$, which means the collaboration graph structure learned from the optimization is symmetric.
\end{proof}
Based on \textbf{Lemma}~\ref{symmetric_graph_lemma}, the message-passing mechanism does converge to the optimal solution, hence we start our proof to \textbf{Theorem}~\ref{convergence}.
\begin{proof}
\label{proof to convergence}
For the sake of brevity, we write $\mathbf{B}_i^{(t)}$, $\boldsymbol{\theta}_i^{(t, k)}$, $\mathbf{W}^{(t)}$ as $\mathbf{B}_i$, $\boldsymbol{\theta}_i$, $\mathbf{W}$. Define
\begin{align*}
    \mathbf{\boldsymbol{\theta}}^{k} = \begin{bmatrix}
        \boldsymbol{\theta}_1^{k} \\
        \boldsymbol{\theta}_2^k \\
        \vdots\\
        \boldsymbol{\theta}_N^k
    \end{bmatrix} \ ,
    \mathbf{b} = \begin{bmatrix}
        b_1^{k} \\
        b_2^k \\
        \vdots\\
        b_N^k
    \end{bmatrix} \ ,
    \mathbf{M} = \begin{bmatrix}
        \mathbf{B}_1^{-1}  &0  &\ldots &0 \\
        0  &\mathbf{B}_2^{-1}  &\ldots &0 \\
        \vdots &\vdots &\ddots    &\vdots \\
        0  &0  &\ldots &\mathbf{B}_N^{-1}
    \end{bmatrix},
\end{align*}

then the iterative algorithm simplifies to:
\begin{align*}
    \mathbf{\boldsymbol{\theta}}^k = \mathbf{M} \left(4\lambda_2 \mathbf{W} \otimes \mathbf{I} \mathbf{\boldsymbol{\theta}}^{k-1} + \mathbf{b}\right).
\end{align*}
Suppose the optimal parameter is $\mathbf{\boldsymbol{\theta}}^{*}$, then we have
\begin{align*}
  \left  \| \mathbf{\boldsymbol{\theta}}^k - \mathbf{\boldsymbol{\theta}}^{*} \right \|_2^2 &= \left \| 4\lambda_2 \mathbf{M} \mathbf{W} \otimes \mathbf{I}\left[\mathbf{\boldsymbol{\theta}}^{k-1} - \mathbf{\boldsymbol{\theta}}^{*} \right] \right\|_2^2
\end{align*}
We write $\left\| \mathbf{M} \mathbf{W} \otimes \mathbf{I} \right\|$ as the operator norm of $\mathbf{M} \mathbf{W} \otimes \mathbf{I}$. Then we have
\begin{align*}
   \|4\lambda_2\mathbf{M} \mathbf{W} \otimes \mathbf{I} \mathbf{x}\|_2^2  &= \sum_{i=1}^{N} \left\|\sum_{j=1}^{N} 4\lambda_2\mathbf{W}_{ij}\mathbf{B}_i^{-1} x_j \right \|_2^2\\
   &\leq \sum_{i=1}^{N} \sum_{j=1}^{N}\left\|4\lambda_2\mathbf{W}_{ij}\mathbf{B}_i^{-1} x_j \right \|_2^2\\
   &\leq \sum_{i=1}^{N} \sum_{j=1}^{N}\left\|4\lambda_2\mathbf{W}_{ij} \mathbf{B}_i^{-1}\right\|^2 \|x_j  \|_2^2\\
   & \leq \sum_{i=1}^{N} \left[\sum_{j=1}^{N}\left\|4\lambda_2\mathbf{W}_{ij} \mathbf{B}_i^{-1}\right\| \|x_j  \|_2 \right]^2\\
\end{align*}
where $\mathbf{x} \in \mathbb{R}^{N d} = [x_1^\top; \ldots; x_N^\top]^\top$. Define $\tilde{x} = \left[\|x_1\|; \ldots; \|x_N\|\right]^\top$, $\mathbf{\tilde{M}}_{ij} = \left\|4\lambda_2\mathbf{W}_{ij} \mathbf{B}_i^{-1}\right\|$. Then we can conclude that
\begin{align*}
    &\|4\lambda_2\mathbf{M} \mathbf{W} \otimes \mathbf{I} \mathbf{x}\|_2^2 \leq \left \| \mathbf{\tilde{M}} \tilde{x}\right\|_2^2  \leq \left\|\mathbf{\tilde{M}}\right\|^2 \|\tilde{x}\|_2^2 = \left\|\mathbf{\tilde{M}}\right\|^2 \| \mathbf{x}\|_2^2
\end{align*}
Recall that the loss function $\mathcal{L}(\boldsymbol{\theta}_i)$ is convex, means we have an eigendecomposition $\mathbf{A}_i = \mathbf{Q}_i^\top \mathbf{\Lambda}_i \mathbf{Q}_i$.
Taking into the definition of $\mathbf{B}_i$, we have
\begin{align*}
4\lambda_2 \mathbf{W}_{ij}\mathbf{B}_i^{-1} &=4\lambda_2 \mathbf{W}_{ij} \left[\mathbf{Q}_i^\top \mathbf{\Lambda}_i \mathbf{Q}_i + \left(2\lambda_1 + 4\lambda_2 d_i\right) \mathbf{I} \right]^{-1} \\
&= \mathbf{Q}_i^\top \left[\frac{\mathbf{\Lambda}_i + \left(2\lambda_1 + 4\lambda_2 d_i\right) \mathbf{I}}{4\lambda_2 \mathbf{W}_{ij}}\right]^{-1} \mathbf{Q}_i
\end{align*}
Thus the operator norm of $4\lambda_2 \mathbf{W}_{ij}\mathbf{B}_i^{-1}$ equals to
\begin{align*}
    \left \| 4\lambda_2 \mathbf{W}_{ij}\mathbf{B}_i^{-1} \right\| = \frac{4\lambda_2 \mathbf{W}_{ij}}{\lambda_i^* + \left(2\lambda_1 + 4\lambda_2 d_i\right) },
\end{align*}
where $\lambda_i^*>0$ is the minimum eigenvalue of $\mathbf{\Lambda}_i$.

By Gershgorin circle theorem, the eigenvalues of $\tilde{\mathbf{M}}$ is bounded by
\begin{align*}
    &\lambda < \sum_{j=1}^{N} \tilde{\mathbf{M}}_{ij} \\
    &= \sum_{j=1}^{N} \frac{4\lambda_2 \mathbf{W}_{ij}}{\lambda_i^* + \left(2\lambda_1 + 4\lambda_2 d_i\right) } = \frac{4\lambda_2 d_i}{\lambda_i^* + \left(2\lambda_1 + 4\lambda_2 d_i\right) },
\end{align*}
for some $1 \leq i \leq N$. Thus
\begin{align*}
    \left\| \tilde{\mathbf{M}} \right \| < \max _{1\leq i \leq N} \left[\frac{4\lambda_2 d_i}{2\lambda_1 + 4\lambda_2d_i}\right] = \rho < 1,
\end{align*}
apply this property to $\mathbf{\boldsymbol{\theta}}^{k-1}$ we have
\begin{align*}
    \left\| \mathbf{\boldsymbol{\theta}}^k - \mathbf{\boldsymbol{\theta}}^* \right\|_2^2 < \rho  \left\| \mathbf{\boldsymbol{\theta}}^{k-1} - \mathbf{\boldsymbol{\theta}}^* \right\|_2^2.
\end{align*}
Taking $k$ goes to infinity we have $ \left\| \mathbf{\boldsymbol{\theta}}^k - \mathbf{\boldsymbol{\theta}}^* \right\|_2^2$ goes to zero with a linear convergence rate. 
\end{proof}

\section{Experimental details}

\subsection{Regression Problem}
\begin{figure*}
    \centering
    \subfloat{
        \includegraphics[width=0.95\linewidth]{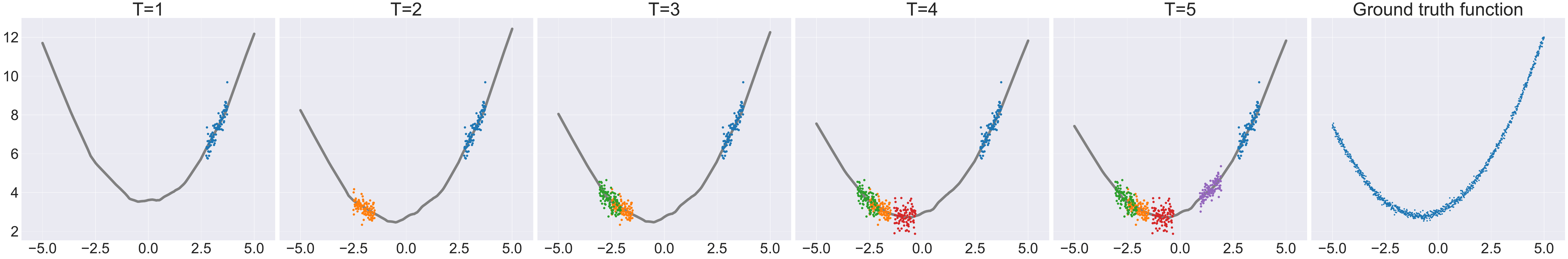}
    }\hfill
    \subfloat{
        \includegraphics[width=0.95\linewidth]{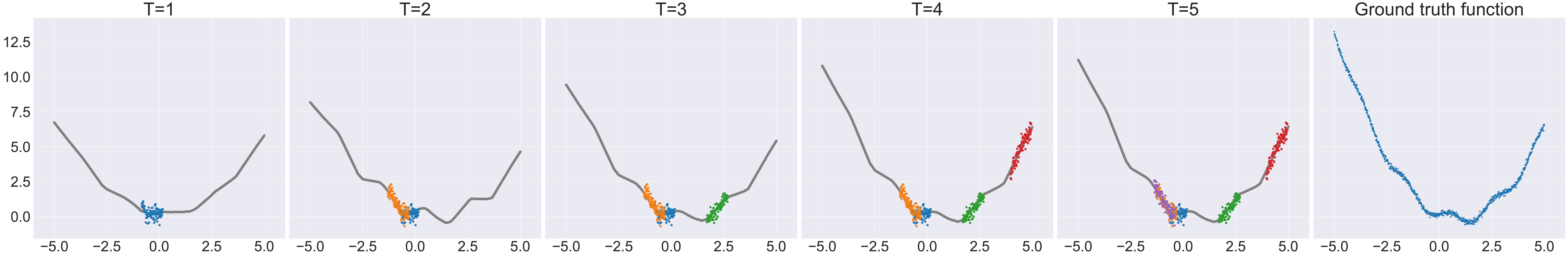}
    }
    \vspace{-3mm}
    \caption{\small A time-evolving learning task example of two agents sampled from the regression task dataset. For each agent, the training samples are generated around the agent's ground truth function with a limited sampling interval. The agent's knowledge of the target function will improve as the number of sampling intervals increases.}
    \vspace{-3mm}
    \label{fig:regression_dataset}
\end{figure*}
\noindent \textbf{Implementation.} We use a fully connected network as the backbone to represent the nonlinear transform $\phi$ as part of the training parameters of \texttt{DeLAMA}. The number of hidden layers is 2 with the output layer set in dimension 50. The batch size used for training is 2 with a learning rate of 1e-3. The iteration numbers for collaborative relational inference $\mathbf{\Phi}_{graph}$ and lifelong model update $\mathbf{\Phi}_{param}$ are set to $M_1 = 10$ and $M_2=10$, respectively. 

\noindent \textbf{Task configuration.} We visualize one example of the generated task sequences shown in \textbf{Figure}~\ref{fig:regression_dataset}. In this example, we can see two different agents with distinct learning goals (the black line). From $T=1$ to $T=5$, the agents encountered different batches of data samples, forming different viewpoints of the target regression function. In the beginning, the agents could not have a full understanding of their learning task. As time increases, agents could be able to guess the learning target function with the help of their lifelong learning ability.

\subsection{Image classification}

\noindent \textbf{Implementation.} We used self-designed convolution network in \textbf{MNIST} tasks and ResNet-18 in \textbf{CIFAR-10} tasks to represent the backbone for nonlinear transform $\phi_{\beta}(\cdot)$. For the self-designed convolution network, the convolution layer number is set to 2 with the output feature dimension 50. For the \textbf{CIFAR-10} tasks we removed the batch normalization layers of ResNet-18 to increase the training stability, and the output feature dimension is also set to 50. The batch size is 2 with a learning rate of 1e-3. The iteration number for lifelong model update $\mathbf{\Phi}_{param}(\cdot)$ is set to 10 and the number of iteration steps of collaboration relational inference $\mathbf{\Phi}_{graph}(\cdot)$ is set to 5. For federated learning methods, we run $100$ rounds for each timestamp. During each round, each client conducts $5$ iterations of model training for \textbf{MNIST} and \textbf{CIFAR-10}.

\subsection{Multi-robot mapping}

\noindent \textbf{Implementation.} The backbone used for \texttt{DeLAMA} is defined as a simple multi-layer perceptron with 2 hidden layers. The output feature dimension of the MLP is 10 and the local models' output dimension is 1 followed by a sigmoid activation function to represent the occupancy probability.  For the training of the unrolled network part, the learning rate is 1e-3 with batch size 2. The iteration number of message passing is set to 10 and the graph learning iteration number is set to 5.

\subsection{Human involved experiment}
\label{appendix_humaneval}
\noindent \textbf{System overview.} The collaboration system employs a hierarchical structure. At the user interface layer, it comprises real human users and 'virtual agents' used for human-machine interaction. At the back end are the agents powered by  \texttt{DeLAMA}. Note that to enable interaction between humans and machines in the back end, we utilize the 'proxy machines' to mimic human behaviors. The full system framework is shown in \textbf{Figure}~\ref{fig:human_exp}.
\begin{figure}
    \centering
    \includegraphics[width=0.9\linewidth]{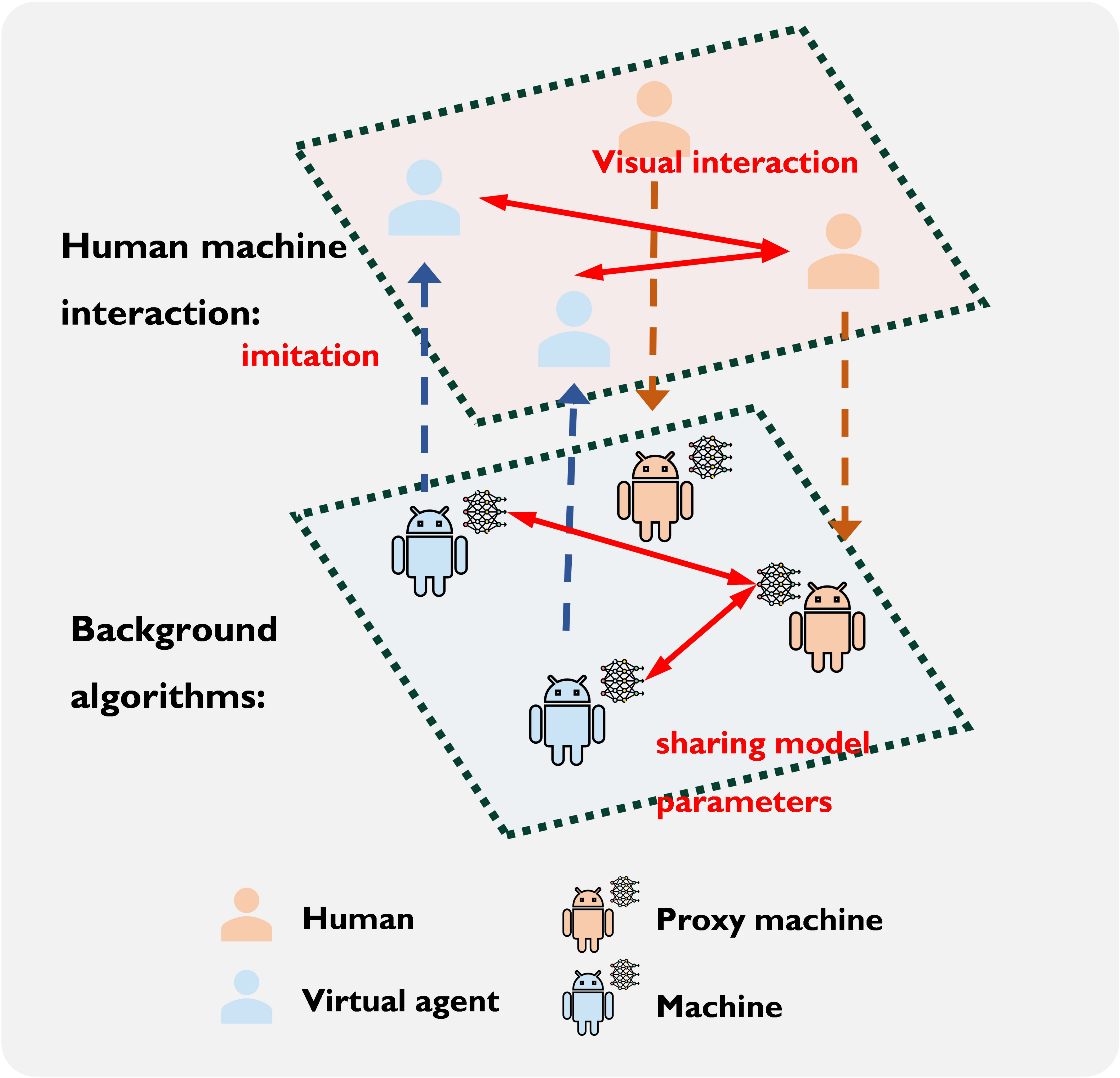}
    \caption{\small The human-machine interaction framework used for the human-involved experiment. In the interaction phase, agents can observe others' behaviors and submit their understandings. To enable the interaction between humans and agents running in \texttt{DeLAMA}, we add virtual agents to perform visual interaction with each other. In the back end is the collaboration algorithm powered by \texttt{DeLAMA}, with proxy agents imitating human behaviors. This top-down framework enables human-machine communication and provides a platform for real human evaluation.}
    \label{fig:human_exp}
\end{figure}

\noindent \textbf{Web GUI.} Here we provide a web GUI example to describe how our human-involved experiment works. There are 5 main pages of the web GUI, including the login page (\textbf{Figure}~\ref{fig:web_page1}), the local learning page (\textbf{Figure}~\ref{fig:web_page2}), the result modification page (\textbf{Figure}~\ref{fig:web_page4}), and the output page (\textbf{Figure}~\ref{fig:web_page5}). The human-involved experiment's routine is as follows:
\begin{enumerate}
    \item \textbf{Login.} Each participant will read the user guide at the login page (shown in \textbf{Figure}~\ref{fig:web_page1}). Then input their user name to represent their ID and enter the game.
    \item \textbf{Local learning.} The participant will first encounter a plot of scatter points (as shown in \textbf{Figure}~\ref{fig:web_page2}), then draw a line to represent their guess of the regression line. This line represents their learning knowledge purely based on their visual information.
    \item \textbf{Aggregate and collaborate.} The agent will have a view of other participants' learning lines, and guess whether to modify their initial learned regression results according to others' results. Then the submitted final learning result will be evaluated and return the regression loss to the participant as shown in \textbf{Figure}~\ref{fig:web_page4}.
\end{enumerate}
On the final page, the participant will see all their historical learning results and the exact ground truth regression target function as shown in \textbf{Figure}~\ref{fig:web_page5}.

\begin{figure}
    \centering
    \includegraphics[width=\linewidth]{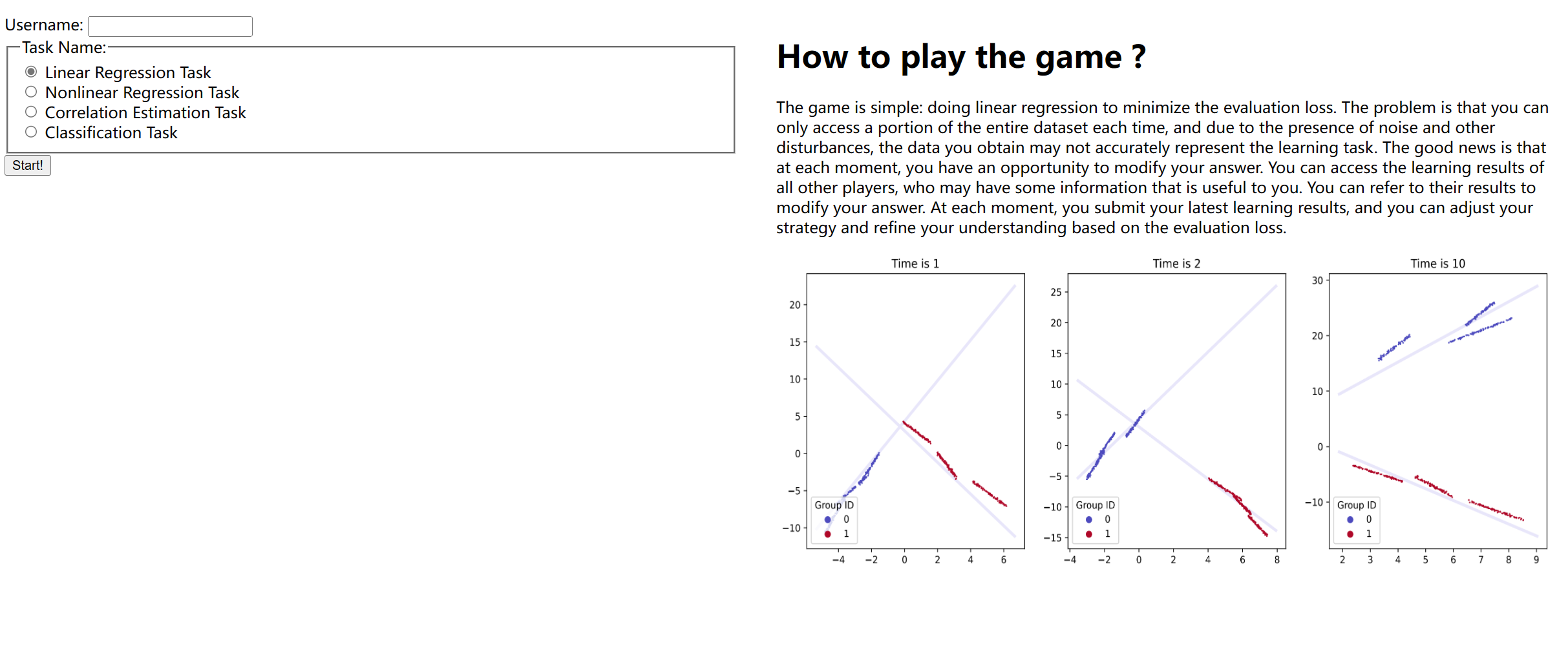}
    \caption{The login page of the human-involved experiment, contains a user guide for the participants to learn how to play the game.}
    \label{fig:web_page1}
\end{figure}

\begin{figure}
    \centering
    \includegraphics[width=\linewidth]{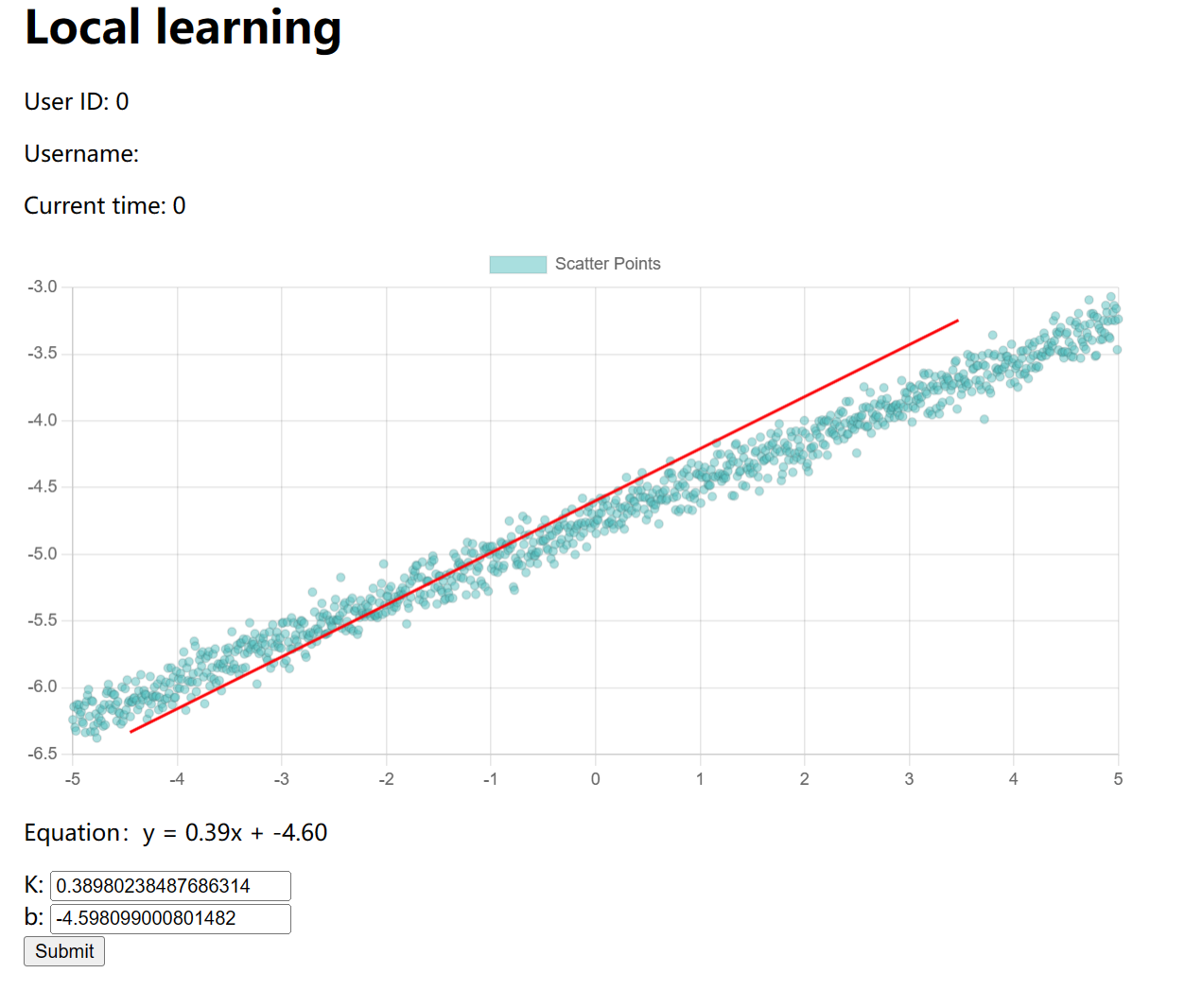}
    \caption{The local learning page of the human-involved experiment. Here participant's drawn line will be transformed into a linear equation with two parameters.}
    \label{fig:web_page2}
\end{figure}

\begin{figure}
    \centering
    \includegraphics[width=\linewidth]{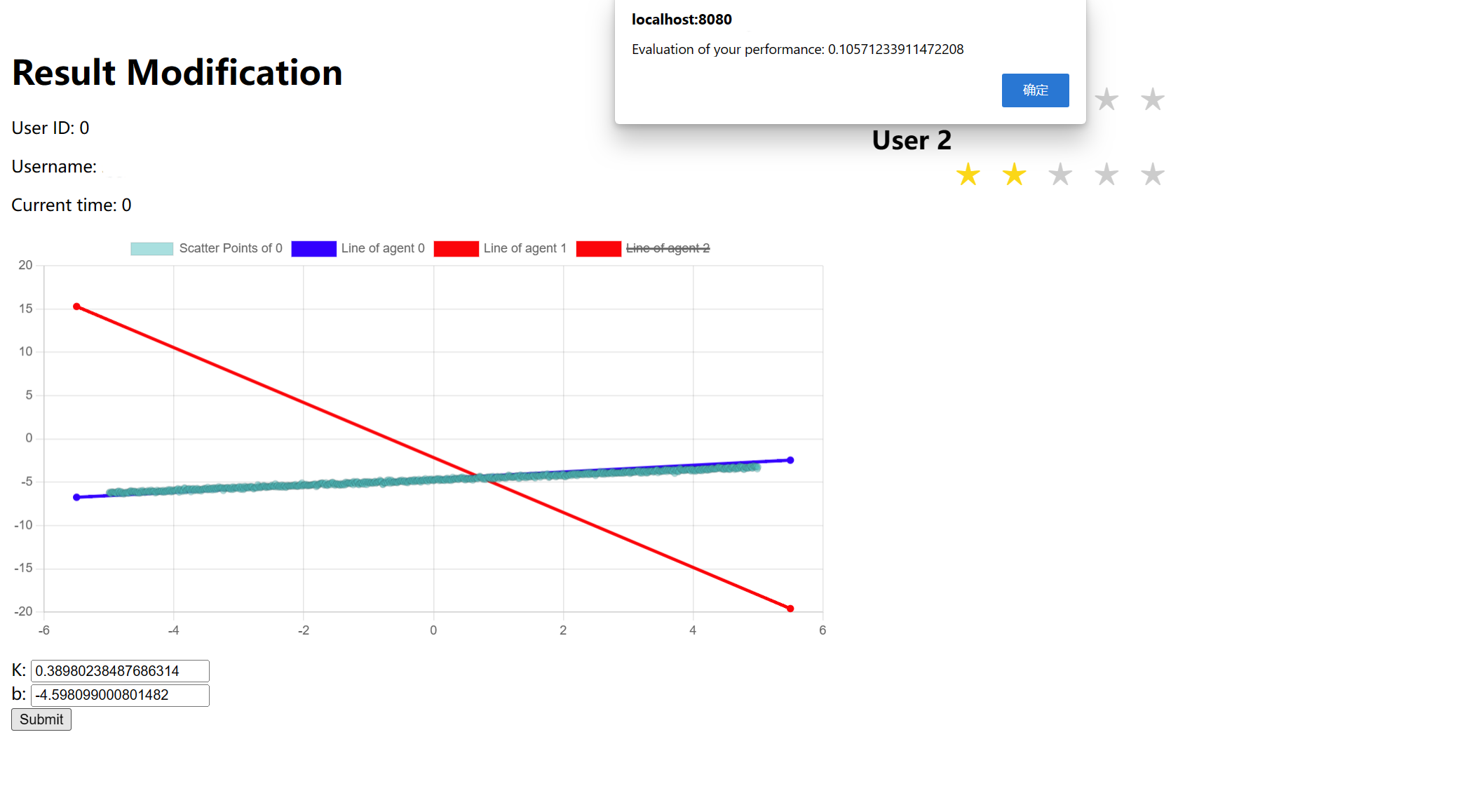}
    \caption{The aggregation page of the human-involved experiment, which contains other participants' results and the score rated by each agent.}
    \label{fig:web_page4}
\end{figure}

\begin{figure}
    \centering
    \includegraphics[width=\linewidth]{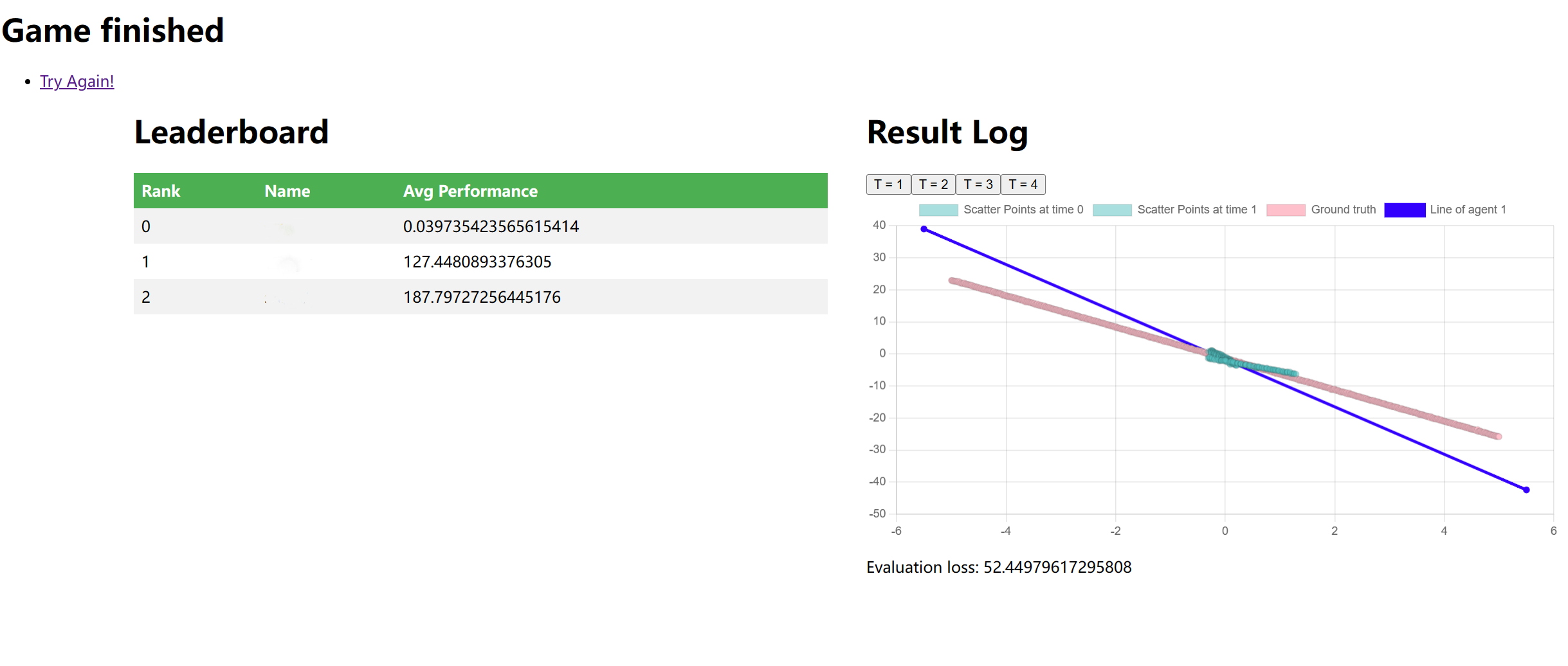}
    \caption{The final page of the human-involved experiment with a leaderboard, each participant can look back to historical learning results.}
    \label{fig:web_page5}
\end{figure}

\end{document}